\documentclass[10pt,journal,compsoc]{IEEEtran}

\ifCLASSOPTIONcompsoc
\usepackage[nocompress]{cite}
\else
\usepackage{cite}
\fi

\usepackage{adjustbox}
\usepackage{capt-of}

\usepackage{bbding}
\usepackage{booktabs} 
\usepackage{xspace}
\usepackage{xcolor}
\usepackage{comment}
\usepackage{multirow}
\usepackage{amsmath}
\usepackage{amsfonts}
\usepackage{mathtools}
\usepackage{makecell}
\usepackage{footnote}
\usepackage{graphicx}
\usepackage{hyperref}
\usepackage{bm}
\makesavenoteenv{tabular}
\makesavenoteenv{table}
\usepackage{comment}

\usepackage{framed}
\usepackage{wrapfig}

\usepackage{color,colortbl}
\definecolor{gray90}{gray}{0.9}

\usepackage{algorithm, enumitem, algorithmic, amsthm, comment, amssymb, xspace}
\usepackage[caption=false]{subfig}

\newtheorem{proposition}{Proposition}
\newtheorem{theorem}{Theorem}
\newtheorem{lemma}{Lemma}
\theoremstyle{definition}
\newtheorem{definition}{Definition}

\newcommand{\nystrom}{Nystr$\mathrm{\ddot{o}}$m}
\newcommand{\mframe}{EasyDGL}
\newcommand{\mname}{EasyDGL}
\newcommand{\mreg}{TPPLE}
\newcommand{\mssl}{CaM}

\newcommand{\reviewerone}[1]{{\color{black} #1}}
\newcommand{\reviewertwo}[1]{{\color{black} #1}}

\begin{document}
	
\title{EasyDGL: Encode, Train and Interpret for Continuous-time Dynamic Graph Learning}
\author{
Chao~Chen,~\IEEEmembership{Member,~IEEE,}
Haoyu~Geng,~\IEEEmembership{Student Member,~IEEE,}
Nianzu~Yang,~\IEEEmembership{Student Member,~IEEE,}
Xiaokang~Yang,~\IEEEmembership{Fellow,~IEEE},
and Junchi~Yan,~\IEEEmembership{Senior Member,~IEEE}
		
\IEEEcompsocitemizethanks{
\IEEEcompsocthanksitem
C. Chen, H. Geng, N. Yang, X, Yang and J. Yan are with  Department of Computer Science and Engineering and MoE Key Lab of Artificial Intelligence, Shanghai Jiao Tong University, Shanghai, 200240, China. J. Yan is also with School of Artificial Intelligence, Shanghai Jiao Tong University, Shanghai, 200030, China.
E-mail: \{chao.chen, genghaoyu98, yangnianzu, xkyang, yanjunchi\}@sjtu.edu.cn.
\IEEEcompsocthanksitem   Correspondence author: Junchi Yan.
}
}

\IEEEtitleabstractindextext{
\begin{abstract}
Dynamic graphs arise in various real-world applications, and it is often welcomed to model the dynamics in continuous time domain for its flexibility. This paper aims to design an easy-to-use pipeline (EasyDGL which is also due to its implementation by DGL toolkit) composed of three modules with both strong fitting ability and interpretability, namely encoding, training and interpreting: i) a temporal point process (TPP) modulated attention architecture to endow the continuous-time resolution with the coupled spatiotemporal dynamics of the graph with edge-addition events; ii) a principled loss composed of task-agnostic TPP posterior maximization based on observed events, and a task-aware loss with a masking strategy over dynamic graph, where the tasks include dynamic link prediction, dynamic node classification and node traffic forecasting; iii) interpretation of the  outputs (e.g., representations and predictions) with scalable perturbation-based quantitative analysis in the graph Fourier domain, which could comprehensively reflect the behavior of the learned model. Empirical results on public benchmarks show our superior performance for time-conditioned predictive tasks, and in particular EasyDGL can effectively quantify the predictive power of frequency content that a model learns from evolving graph data.
\end{abstract}
}
	
\maketitle	
\IEEEdisplaynontitleabstractindextext
\IEEEpeerreviewmaketitle

\IEEEraisesectionheading{
\section{Introduction}\label{sec:intro}}
\subsection{Background and Challenges}
\IEEEPARstart{D}ynamic graph representation learning (DGRL) in general aims to obtain node vector embeddings over time $t$ (ideally in continuous time domain, i.e., $t\!\in\!\mathbb{R}$) on a graph of which the topological structures and the node attributes are evolving as shown in Fig.~\ref{fig:dyg}. Under such a general paradigm, tailored models are often devised and trained according to specific (time-conditioned) predictive tasks, ranging from (dynamic) node classification, link prediction, to traffic flow forecasting. These tasks have found widespread application in social networks, bioinformatics, transport, etc.

In practise, a general DGRL pipeline should consist of the encoding framework whose inputs include the graph information and the observed events, the learning scheme with the derived loss according to specific tasks at hand. In addition, the interpretation of model output is often of high interest in real-world applications. Our following discussion further elaborates on these three aspects.
\begin{figure}[tb!]
\centerline{\includegraphics[width=.42\textwidth]{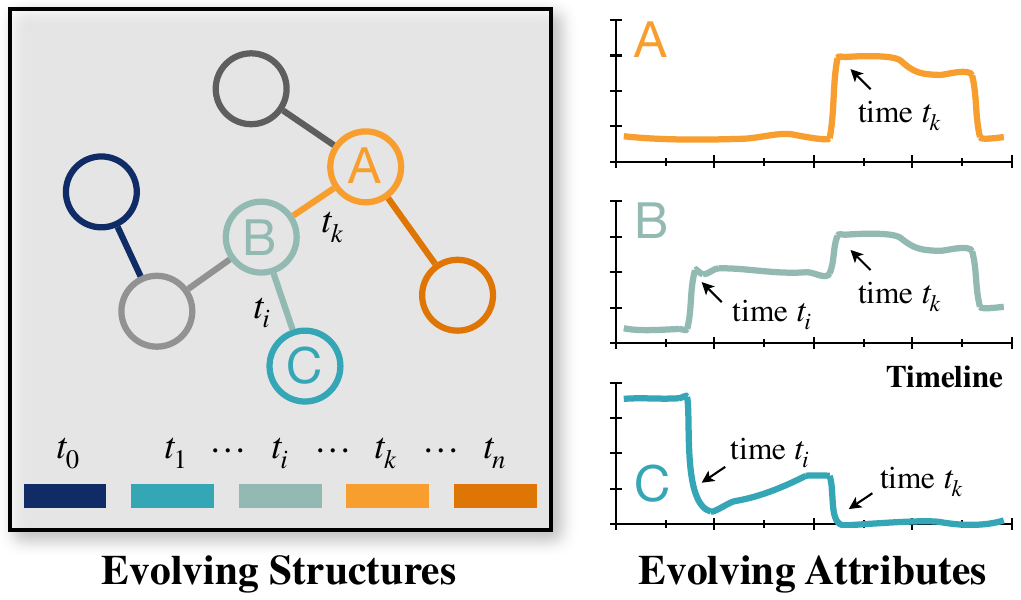}}
\vspace{-5pt}
\caption{\label{fig:dyg}Example of dynamic graph where time $t_i$ corresponds to an event of adding edge $(B,C)$. It shows this event changes the graph structures and the node attributes: the connection between node $B$ and node $C$ is changed and the attribute value of node $C$ (in blue) drops significantly. Also, the attribute of node $C$ shows stable trend before the addition of edge $(B,C)$ at time $t_i$, then exhibits upward trend until the event of edge $(A,B)$ emerging at time $t_k$. This suggests that the structural and temporal dynamics of the observed graph can be entangled.}
\end{figure}
	
\begin{figure*}[tb!]
\centerline{\includegraphics[width=.99\textwidth]{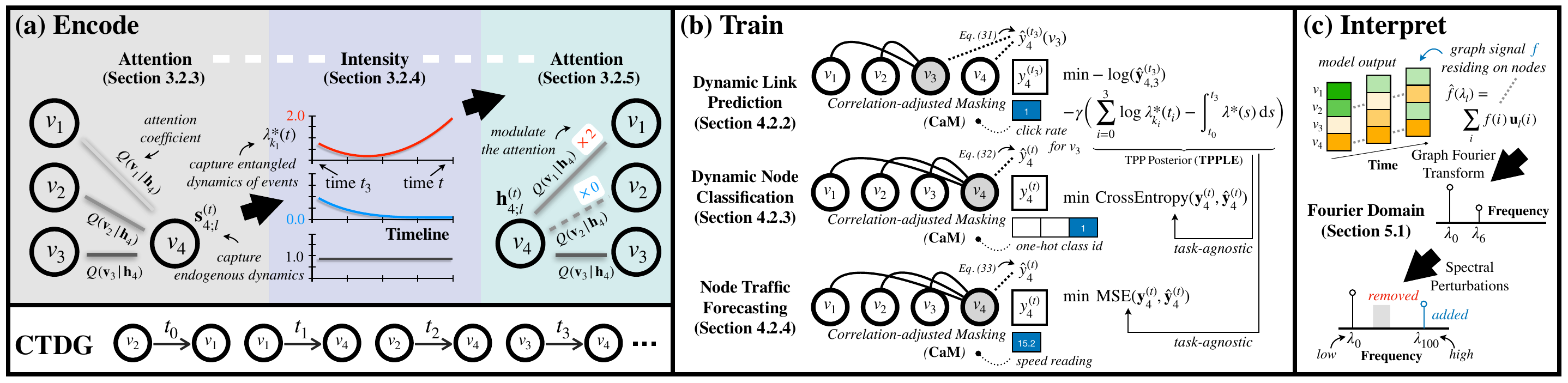}}
\vspace{-5pt}
\caption{\label{fig:overview}{\mname} (encode-train-interpret) pipeline for continuous-time graph. 
(\textbf{a}) We present the attention-intensity-attention architecture to encode the graph with its event history into node embeddings, for example $\mathbf{h}^{(t)}_{4;l}$ of node $v_4$ which is outputted from the $l^\mathrm{th}$ layer and will be used as input to predict what will happen at a given time $t$ in future. In this example, there are four events of edge addition occurred at time $t_0, t_1, t_2, t_3$ respectively. The TPP intensity function is adopted to model the dynamics of events, which is then used to modulate attention networks on graph (see Sec.~\ref{sec:encode}). Specifically, the attention (gray colored) encodes endogenous dynamics within neighborhood into $\mathbf{s}^{(t)}_{4;l}$, then it will be coupled with exogenous temporal dynamics by the intensities $\lambda^\ast_{k_1}(t),\lambda^\ast_{k_2}(t),\lambda^\ast_{k_3}(t)$ (blue colored) that specify the occurrence of events on edges $(v_1,v_4),(v_2,v_4),(v_3,v_4)$ at time $t$. This quantity that captures evolution of the graph can be used to modulate message passing of the attention (green colored);
(\textbf{b}) We consider three popular tasks: dynamic link prediction (see Sec.~\ref{sec:lpredc}), dynamic node classification (see Sec.~\ref{sec:nclassf}) and node traffic forecasting (see Sec.~\ref{sec:tforecast}) under a unified masking-based training scheme;
(\textbf{c}) Finally, perturbation-based analysis is performed to quantify the importance of frequency content that the model learns from the data (see Sec.~\ref{sec:interpret}) in a scalable way. In particular, it enables model-level interpretation instead of instance-level interpretation as is done in most existing explainable graph learning works~\cite{fan2021gcn,pope2019explainability,yuan2021explainability,yuan2022explainability}.}
\end{figure*}

$\blacksquare$ \textit{Q1: how to efficiently and effectively encode the entangled structural and temporal dynamics of the observed graph?} 
	
Earlier works, e.g., DCRNN\!\cite{li2018diffusion}, DySAT\!\cite{sankar2020dysat} and STGCN\!\cite{yu2018spatio}, often view the dynamic temporal graph as a sequence of snapshots with discrete time slots. To improve the model expressiveness and flexibility, continuous-time approaches are recently developed promptly whereby the evolution of graph structures and node attributes, as shown in Fig.~\ref{fig:dyg}, are directly modeled.
Beyond the works\cite{kumar2019predicting,nguyen2018continuous,xu2020inductive,li2020time} that encode time into embeddings independent of spatial information, seminal works\cite{zuo2018embedding,trivedi2019dyrep,zuo2020transformer} explore the use of temporal point process (TPP) techniques for DGRL and introduce the MessagePassing-Intensity framework to capture the entangled spatiotemporal dynamics. Even though these approaches have shown promising results, they encounter common limitations:
\reviewertwo{1) the message-passing module takes node features as input, whereas the intensity module offers edge intensities as output. This mismatch between the input and the output limits the MessagePassing-Intensity network to be used as a basic building block that could be stacked to learn high-level features in multi-layered architectures;
2) existing works mainly focus on link prediction tasks, as the intensity naturally specifies the occurrence of edge-addition events at any given time. When it comes to node regression tasks (e.g., node traffic forecasting), the edge-wise intensity cannot easily be used to design the loss function.}
	
With the immense success of attention networks, a recent work~\cite{cho2021learning} designs an Attention-Attention framework where the former attention models periodic patterns and the latter attention encodes the structural and temporal dependencies. However, we argue that the lack of the ability to model the evolving patterns in continuous time domain will hurt the performance. CADN~\cite{chien2022learning} develops an ODERNN-Attention encoding framework where ODERNN, an extension of RNN with ordinary differential equation (ODE) solver, is adopted to learn continuous-time dynamics. However, RNN requires sequential computation and thus suffers from slow training. In this paper, we address the aforementioned issues by formalizing an attention-intensity-attention encoding framework, suitable for link-level and node-level applications.

$\blacksquare$ \textit{Q2: for typical prediction tasks, how to design a principled family of loss functions and training scheme for dynamic graphs?} 

\reviewertwo{During the training, we did not observe the significant accuracy improvement (see Table~\ref{tab:motiv}) when using conventional training strategies for (dynamic) graphs: masked autoencoding\cite{devlin2018bert,he2022masked} and temporal smoothness\cite{zhu2016scalable,zhou2018dynamic}.}

\reviewertwo{Recent works~\cite{devlin2018bert,sun2019bert4rec,hou2022graphmae}, which randomly mask a portion of graph data using special tokens and then learn to predict the removed content, have shown promising results in link prediction, but many studies~\cite{liu2022graph} have found that masking would negatively affect model performance on benchmarks for node-level tasks.
We conjecture the reason is that different queries at different timestamps which are replaced by special tokens would have the same query embedding, making the ranking of attention weights of the keys shared across these queries.
This may induce a global ranking of influential nodes, which makes sense in domains such as recommender systems and social networks where popular products and top influencers dominate the market. However, a node predictive task might emphasize recently formed edges or certain historical timestamps.} 
In this paper, we propose a new correlation-adjusted masking (CaM) for task-aware learning, which can help attention identify and exploit the above patterns in predicting node labels, while enjoying the benefits of existing masking in link prediction.

\reviewertwo{Despite masked autoencoding, we found by experiments that node embeddings might not effectively learn the dynamics of the observed graph (see supplemental material for more discussions).
In the literature, temporal smoothness is often considered~\cite{zhu2016scalable,zhou2018dynamic} which suppresses sharp changes over time. However, we argue that this prior can sometimes be too restrictive in real-world scenarios, for example road networks where graph data evolves quickly.} In this paper, we address the issue by introducing the TPP likelihood of events (TPPLE) as a principled way of explaining the graph changes in continuous time domain.

\begin{table}[t!]
\centering
\caption{\label{tab:motiv}RMSE on META-LA for 15, 30 and 45 minutes ahead traffic forecasting, where {\mname} w/ masking randomly replace the nodes with special tokens~\cite{hou2022graphmae}, and {\mname} w/ temporal smoothness penalizes the $\ell_2$-norm distance of node embeddings at consecutive snapshots~\cite{zhou2018dynamic}.}
\vspace{-5pt}
\begin{tabular}{lc c c}\toprule\toprule
    & \textbf{15 min} & \textbf{30 min} 
    & \textbf{45 min} 
    \\\midrule

   \reviewertwo{\textbf{{\mname} w/o CaM and TPPLE}} &
    \reviewertwo{5.367} & 
    \reviewertwo{6.368} & 
    \reviewertwo{6.969}
    \\\midrule
    
    \reviewertwo{\textbf{{\mname} w/ masking}} &
    \reviewertwo{5.340} & 
    \reviewertwo{6.479} & 
    \reviewertwo{7.047} \\
    
    \reviewertwo{\textbf{{\mname} w/ CaM}} &
    \reviewertwo{5.291} & 
    \reviewertwo{6.230} & 
    \reviewertwo{6.849}
    \\\midrule
    
    \reviewertwo{\textbf{{\mname} w/ temporal smoothness}} &
    \reviewertwo{5.429}  & 
    \reviewertwo{6.427} & 
    \reviewertwo{7.161} \\
    
    \reviewertwo{\textbf{{\mname} w/ TPPLE}} &
    \reviewertwo{5.250}  & 
    \reviewertwo{6.237} & 
    \reviewertwo{6.853}
    \\\bottomrule\bottomrule
\end{tabular}
\end{table}
	
$\blacksquare$ \textit{Q3: how to analyze and interpret dynamic node representations or model predictions on large-scale graphs over time?}

\reviewertwo{When dealing with a complex model like GNN, it is of critical importance to understand and interpret its behaviors which can help the experts refine model architectures and training strategies.}
Existing work on interpretable GRL often analyzes the importance of neighboring nodes' features to a query node (mostly indicated by the learned attention weights~\cite{fan2021gcn} and the gradient-based weights~\cite{pope2019explainability}) as well as the importance of the (sub)graph topologies based on the Shapley value~\cite{yuan2021explainability}. These works are mainly example-specific and explain why a model predicts what it predicts in the graph vertex domain. It often remains unclear which frequency information the model is exploiting in making predictions. 
We note that the quantitative analysis in the graph Fourier domain has a distinct advantage since it can provide the model-level explanations with respect to the intrinsic structures of the underlying graph. In general, model-level explanations are high-level and can summarize the main characteristics of the model (see Table~\ref{tbl:perturb} for example). 

To achieve this, there exist two main challenges: 
1) the graph Fourier transform (GFT)~\cite{ortega2018graph} provides a means to map the data (e.g., model predictions and node embeddings) residing on the nodes into the spectrum of frequency components, but it cannot quantify the predictive power of the learned frequency content. For example in Fig.~\ref{fig:vf_analysis}, applying GFT shows that dynamic graph model BERT4REC~\cite{sun2019bert4rec} can learn high-frequency information that is not provided by static graph model GSIMC~\cite{chen2023graph}, but it is still not clear if BERT4REC outperforms GSIMC because of this high-frequency information.
\reviewertwo{2) GFT requires the performance of orthogonalized Laplacian decomposition, whose cost is up to $\mathcal{O}(N^3)$, cubic to the number of nodes $N$ in the graph.  Although {\nystrom} methods~\cite{fowlkes2004spectral} can obtain the linear cost, their results do not meet the orthogonality condition (see more discussion in Sec.~\ref{sec:interpret_decomp}). 
In this paper, we address the challenges by proposing appropriate perturbations in the graph Fourier domain, in company with a new scalable and orthogonalized graph Laplacian decomposition algorithm.}

\subsection{Approach Overview and Contributions}
In light of such above discussions, we propose \textbf{\mframe}, a new encoding, training and interpreting pipeline for dynamic graph representation learning, of which the goal is:
1) effective and efficient encoding of the input dynamic graph with its event history; 
2) principled loss design accounting for both prediction accuracy and history fitting; 
3) scalable (spectral domain) post-analyzing on large dynamic graphs. The holistic overview of {\mframe} is shown in Fig.~\ref{fig:overview}.

This paper is a significant extension (almost new)\footnote{Compared with the conference version, extensions can be summarized in three folds (see also Table~\ref{tab:algsurvey}: 
i) extend to a general representation learning framework on dynamic graph from the task-specific method for link prediction; 
ii) propose a principled training scheme that consists of task-agnostic TPP posterior maximization and task-aware loss with a masking strategy designed for dynamic graph; 
iii) devise a spectral quantitative post-analyzing technique based on perturbations, where a scalable and orthogonalized graph Laplacian decomposition is proposed to support the analysis on very large graph data.} of the preliminary works~\cite{chen2021learning}. The main contributions of this work are as follows:

\textbf{1) Attention-Intensity-Attention Encoding (Sec.~\ref{sec:encode}):} 
the former attention parameterizes the TPP intensity function to capture the dynamics of (edge-addition) events occurred at irregular continuous timestamps, while the latter attention is modulated by the TPP intensity function to compute the (continuous) time-conditioned node embedding that can be easily used for both link-level and node-level tasks. 
We note that our proposed network can be stacked to achieve highly nonlinear spatiotemporal dynamics of the observed graph, different from existing TPP based works~\cite{zuo2018embedding,trivedi2019dyrep,zuo2020transformer,chien2022learning}. 
We also address \textbf{a new challenge}: existing attention models interpret the importance of neighboring node's features to query node by the learned attention weights~\cite{kang2018self}.
However, they do not provide any insights into the importance of the observed timestamps.
By contrast, the TPP intensity has the ability to explain the excitation and inhibition effects of past timestamps and events in the future (see Fig.~\ref{fig:tf_trends}).

\textbf{2) Principled Learning (Sec.~\ref{sec:learning}):} 
we introduce a novel correlation-adjusted masking ({\mssl}) strategy which can derive a family of loss functions for different predictive tasks (node classification, link prediction and traffic forecasting). Specifically, we mask a portion of graph nodes at random: different from prior works that replace masked node with special tokens\cite{sun2019bert4rec,yuan2020future,you2020does,hou2022graphmae}, in {\mname} the masked key nodes are removed and not used for training, which provides diversified topologies of the observed graph in multiple epochs and allows us to process a portion of graph nodes. This can reduce the overall computation and memory cost; masked query nodes are replaced with their labels, which encourages the attention to focus on the keys whose embeddings are highly correlated with the label information of the query. We also add time encodings to every node such that the masked nodes could have information about their ``locations'' in graph.
The task-aware loss with masking on dynamic graph is further combined with a task-agnostic loss based on the TPP posterior maximization that accounts for the evolution of a dynamic graph. 
This term (called {\mreg}) replaces the temporal smoothing prior as widely used in existing works~\cite{zhu2016scalable,zhou2018dynamic} that assume no sharp changes.

\textbf{3) Scalable and Perturbation-based Spectral Interpretation (Sec.~\ref{sec:interpret}):} to our knowledge, for the first time we achieve $\mathcal{O}(Nsr\!+\!r^3)$ ($N,s,r$ is the number of graph nodes, sampled columns and Laplacian eigenvectors respectively) complexity to transform model outputs into frequency domain by introducing a scalable and orthogonalized graph Laplacian decomposition algorithm, while vanilla techniques require $\mathcal{O}(N^3)$ cost.
On top of it, we design intra-perturbations and inter-perturbations to perturb model predictions (or node embeddings) in the graph Fourier domain and examine the resulting change in model accuracy to precisely quantify the predictive power of the learned frequency content. Note that the frequency in fact describes the smoothness of outputted predictions with respect to the underlying graph structure, which is more informative than plain features.

\textbf{4) Empirical Effectiveness (Sec.~\ref{sec:expr}).} We evaluate {\mname} on real-world datasets for traffic forecasting, dynamic link prediction and node classification. Experimental results verify the effectiveness of our devised components, and in particular {\mname} achieves $6.7\!-\!14.5\%$ relative Hit-Rate gains and $11.4\!-\!14.7\%$ relative NDCG gains over state-of-the-art baselines on the challenging Netflix prize data (with $497,959$ nodes and $100,444,166$ edges). The source code is available at \textcolor{blue}{\url{https://github.com/cchao0116/EasyDGL}}.

\section{Background and Problem Statement}\label{sec:notat}
We first briefly introduce the basics of attention neural networks and temporal point process techniques for dynamic graphs, then formally describe our focused problem setting and notations, and in the end provide further discussion on related works.

\subsection{Preliminaries}\label{sec:gatintro}
\subsubsection{Graph Attention Networks}
Suppose that $v_u$ is the query node and a set $\mathcal{N}_u$ represents its neighbors of size $|\mathcal{N}_u|=n$, then in the general formulation, attention network computes the embedding $\mathbf{h}_u$ for query $v_u$ as a weighted average of the features $\{\mathbf{v}_i\}$ over the neighbors $v_i\in\mathcal{N}_u$, using the learned attention weights $e_{u,i}$:
\begin{equation}\label{eq:att}
\begin{split}
\mathrm{Att}(\mathbf{h}_u|\mathcal{N}_u, \mathbf{v})
    &=\sum_{i\in\mathcal{N}_u} 
    \left( 
        \frac{\exp(e_{u,i})}
        {\sum_{i'\in\mathcal{N}_u}\exp(e_{u,i'})}
        \right)\mathbf{v}_i \\
    &=\sum_{i\in\mathcal{N}_u} 
        Q(\mathbf{v}_i|\mathbf{h}_u)\mathbf{v}_i
    =\mathbf{E}_{\mathbf{v}_i\sim{Q}}\mathbf{v}_i, 
\end{split}
\end{equation}where $Q(\mathbf{v}_i|\mathbf{h}_u)$ is the normalized attention weight (score) that indicates the importance of node $v_i$'s features to node $v_u$. It is worth pointing out that the scoring function $e_{u,i}$ can be defined in various ways, and in the literature three popular implementations for representation learning on graphs are self-attention networks (SA) \cite{vaswani2017attention}, graph attention networks (GAT) \cite{velivckovic2017graph} and its recent variant GATv2 \cite{brody2021attentive}:
\begin{flalign}
\mathrm{SA}&:
    \quad  e_{u,i} = \mathbf{q}_u^\top\mathbf{k}_i/\sqrt{d} \label{eq:sa}\\
    \mathrm{GAT}&:
    \quad e_{u,i} = \mathrm{LeakyReLU}(
    \mathbf{a}^\top
    \big[\,\mathbf{k}_u\parallel\mathbf{k}_i\,\big]) \label{eq:gat}\\
    \mathrm{GATv2}&:
    \quad e_{u,i} = \mathbf{a}^\top
    \mathrm{LeakyReLU}(
\big[\,\mathbf{k}_u\parallel\mathbf{k}_i\,\big]) \label{eq:gatv2}
\end{flalign}where $\mathbf{q}_u\in\mathbb{R}^d$ and $\mathbf{k}_i\in\mathbb{R}^d$ are the transformed features of query node $v_u$ and key node $v_i\in\mathcal{N}_u$ respectively, $\mathbf{a}\in\mathbb{R}^{2d}$ is the model parameter and $\parallel$ is the concatenation operation.

\subsubsection{Temporal Point Process and Intensity Function}\label{sec:bk_tpp}
Temporal pint process (TPP) is a stochastic model for the time of the next event given all the times of previous events $\mathcal{H}_t\!\coloneqq\!\{(k_i,t_i)|t_i<t\}$, where the tuple $(k_i,t_i)$ corresponds to an event of type $k_{i}\in\{1,\dots,K\}$ occurred at time $t_i$ and $K$ is the total number of event types. The dynamics of the arrival times can equivalently be represented as a counting process $\mathrm{N}(t)$ that counts the number of events occurred at time $t$. In TPPs, we specify the above dynamics using the conditional intensity function $\lambda_k^\ast(t)$ that encapsulates the expected number of type-$k$ events occurred in the infinitesimal time interval $[t, t+\mathrm{d}t]$\cite{daley2003introduction} given the history $\mathcal{H}_t$:
\begin{gather}\label{eq:intensity}
\lambda_k^\ast(t)\,\mathrm{d}t 
    = \mathbf{E} \Big[\mathrm{N}_k(\mathrm{d}t\big| \mathcal{H}_{t})\Big],
\end{gather}where $*$ reminds that the intensity $\lambda^\ast$ is history dependent, $\mathrm{N}_k(\mathrm{d}t|\mathcal{H}_{t})$ denotes the number of type-$k$ events falling in a time interval $[t, t\!+\!\mathrm{d}t)$. Note that the conditional intensity function has a relation with the conditional density function of the interevent times, according to survival analysis~\cite{aalen2008survival}:
\begin{flalign}
f_k^\ast(t)\!=\!\lambda_k^\ast(t)\exp
    \Bigg(
	-\int_{t_n}^t \lambda^\ast(s) \,\mathrm{d}s
	\Bigg),\quad
\lambda^\ast(s)\!=\!\sum_{k=1}^K \lambda_k^\ast(s) \nonumber
\end{flalign}where the exponential term is called the survival function that characterizes the conditional probability that no event happens during $[t_n, t)$, and $t_n$ is the last event occurred before time $t$. Using the chain rule, the likelihood for a point pattern $\{(t_i,k_i)|t_i<T)\}$ is given by
\begin{flalign}\label{eq:tpp_mle}
L=\Bigg(
    \prod_{i=1}^n \lambda_{k_i}^\ast(t_i)
    \Bigg) \exp
    \Bigg(
        -\int_{0}^T \lambda^\ast(s) \,\mathrm{d}s
\Bigg),
\end{flalign}The integral for all possible non-events is intractable to compute. We show two approximate algorithms in Sec.~\ref{ssec:treg}.

\begin{table*}[tb!]
\caption{Summary of existing literature on representation learning for dynamic graphs based on different types of dynamic graph definitions (discrete and continuous-time), encoding architectures, masked training schemes, model interpreting techniques and embedding evolving strategies.}
\label{tab:algsurvey}
\vspace{-10pt}
\begin{center}
\begin{tabular}{@{}r cccccc@{}}\toprule
    \textbf{Method} & \textbf{Dynamic Graph} &
    \textbf{Encoding Architecture} &
    \textbf{Masked Training} & 
    \textbf{Model Interpreting} & 
    \textbf{Embedding Evolving} \\\midrule

    DCRNN~\cite{li2018diffusion} & Discrete &
    GCN-RNN & 
    N/A & N/A & 
    Snapshot-based \\
    STGCN~\cite{yu2018spatio} & Discrete &
    GCN-CNN & 
    N/A & N/A & 
    Snapshot-based \\
    DSTAGNN~\cite{lan2022dstagnn} & Discrete &
    Attention-GCN-CNN &
    N/A & N/A &
    Snapshot-based \\
    EvolveGCN~\cite{pareja2020evolvegcn} & Discrete &
    GCN-RNN &
    N/A & N/A &
    Snapshot-based  \\
    GREC~\cite{yuan2020future} & Discrete &
    Convolution & 
    Special tokens & N/A & 
    Snapshot-based  \\
    DySAT~\cite{sankar2020dysat} & Discrete &
    Attention & 
    N/A & Attention coefficient & 
    Snapshot-based \\
    BERT4REC~\cite{sun2019bert4rec} & Discrete &
    Attention & 
    Special tokens & Attention coefficient & 
    Snapshot-based \\\midrule
			
    JODIE~\cite{kumar2019predicting} & Continuous &
    GraphSAGE-RNN &  
    N/A & N/A & 
    Event-triggered  \\
    TGN~\cite{rossi2020temporal} & Continuous &
    MessagePassing-RNN & 
    N/A & N/A & 
    Event-triggered  \\
    DyREP~\cite{trivedi2019dyrep} & Continuous &
    MessagePassing-Intensity & 
    N/A & TPP intensity & 
    Spontaneous  \\
    TGAT\cite{xu2020inductive} & Continuous &
    TimeEncoding-Attention & 
    N/A & Attention coefficient & 
    Spontaneous  \\
    TiSASREC\cite{li2020time} & Continuous &
    TimeEncoding-Attention & 
    N/A & Attention coefficient & 
    Spontaneous  \\
    CADN~\cite{chien2022learning} & Continuous & 
    ODERNN-Attention & 
    N/A & Attention coefficient & 
    Spontaneous \\
    TimelyREC~\cite{cho2021learning} & Continuous & 
    Attention-Attention & 
    Special tokens & Attention coefficient & 
    Spontaneous  \\\midrule
			
    CTSMA$^*$\cite{chen2021learning} & Continuous &
    Attention-Intensity-Attention & 
    N/A & Both of above & 
    Spontaneous  \\
    {\mname} (ours) & Continuous &
    Attention-Intensity-Attention & 
    CaM \& TPPLE & + Spectral perturbation & 
    Spontaneous 
    \\\bottomrule
\end{tabular}
\end{center}
$^*$Conference version of this work.
\end{table*}
	
\subsection{Problem Formulation and Notations}
	
\begin{definition}[\textbf{Graph}]
A directed graph $\mathcal{G}=\{\mathbb{V}, \mathbb{E}\}$ contains nodes $\mathbb{V}\!=\!\{v_1,\dots,v_N\}$ and edges $\mathbb{E}\subseteq\mathbb{V}\times\mathbb{V}$, where $(u,v)\in\mathbb{E}$ denotes an edge from a node $u$ to node $v$. An undirected graph can be represented with bidirectional edges.
\end{definition}
	
\begin{definition}[\textbf{Dynamic Graph}]
Suppose that $\mathcal{G}$ is a graph representing an initial state of a dynamic graph at time $t_0$, let $\mathcal{O}$ denote the set of observations where tuple $(u,v,\mathrm{op}_t)\!\in\!\mathcal{O}$ denotes an edge operation $\mathrm{op}_t$ of either addition or deletion conducting on edge $(u,v)\in\mathbb{E}$ at time $t$, then we represent a \textit{continuous-time (dynamic) graph} as a pair $(\mathcal{G},\mathcal{O})$, while a \textit{discrete-time (dynamic) graph} is a set of graph snapshots sampled from the dynamic graph at regularly-spaced times, represented as $\{\mathcal{G}_0,\mathcal{G}_{1},\dots\}$.
\end{definition}
	
 \begin{definition}[\textbf{Continuous-time Representation}]
Given a continuous-time graph $(\mathcal{G}, \mathcal{O})$, if node $v_u$ involves $n$ events at time $\{t_1,\dots,t_n\}$ and $\mathcal{X}\!\!=\!\!\{\mathbf{X}^{(1)},\dots,\mathbf{X}^{(n)}\}$ denotes dynamic node attributes, then we define its continuous-time representation at any future timestamp $t>t_n$ by $\mathbf{h}^{(t)}_{u}$:
\begin{flalign}
\mathbf{h}^{(t)}_{u} = 
    \text{GRL}(u,t;
    \mathcal{G},\mathcal{O},\mathcal{X})
\end{flalign}where GRL represents graph representation learning model.
\end{definition}
	
\begin{definition}[\textbf{Dynamic Link Prediction}]
\label{def:dyLink}
Given $(\mathcal{G},\mathcal{O})$ and let $v_u$ denote the query node, then the goal of \textit{dynamic link prediction} is to predict which is the most likely vertex that will connect with node $v_u$ at a given future time $t$:
\begin{flalign}
    (\mathcal{G}, \mathcal{O}, \mathcal{X};u,t)
	\stackrel{\mathbf{h}_u^{(t)}}{\longrightarrow}
    \textit{prob. of $v_i$ connecting with $v_u$}.
\end{flalign}
\end{definition}
	
\begin{definition}[\textbf{Dynamic Node Classification}]
\label{def:dyNode}
Given $(\mathcal{G},\mathcal{O})$ and $v_u$, then \textit{dynamic node classification} aims to predict the label $\mathbf{y}^{(t)}_u$ of node $v_u$ at a given future time $t$:
\begin{flalign}
(\mathcal{G}, \mathcal{O}, \mathcal{X};u,t)
    \stackrel{\mathbf{h}_u^{(t)}}{\longrightarrow}
	\mathbf{y}_u^{(t)}.
\end{flalign}
\end{definition}

\begin{definition}[\textbf{Traffic Forecasting}]
\label{def:traff}
Given $\mathcal{X}\!=\!\{\mathbf{X}^{(t)}\}$ and $\mathbf{X}^{(t)}\!\in\!\mathbb{R}^{N\times{P}}$ denotes the speed readings at time $t$ where $P$ is the number of features, then existing approaches for traffic forecasting maintain a sliding window of $T'$ historical data to predict future $T$ data, on a given road network (graph) $\mathcal{G}$:
    \begin{flalign}
        [\mathbf{X}^{(t-T'+1)},\dots,\mathbf{X}^{(t)};\mathcal{G},\mathcal{O}] 
        \longrightarrow
        [\mathbf{X}^{(t+1)},\dots,\mathbf{X}^{(t+T)}].
    \end{flalign}
 \end{definition}
	
To capture the dynamics on graph, in this paper we focus on the events of edge addition and define a temporal point process on dynamic graph:
	
\begin{definition}[\textbf{TPP on Dynamic Graph and Events of Edge Addition}]
Given any query node $v_u$ and its neighborhood $\mathcal{N}_u$, we introduce the modeling of a TPP on $v_u$ where the event is defined as the addition of an edge connecting node $v_u$. In such case, the intensity $\lambda_i^\ast(t)$	essentially models the occurrence of edges between node $v_i$ and node $v_u$ at time $t$. 
However, it is expensive to compute $\lambda_i^\ast(t)$ for all $v_i\in\mathcal{N}_u$. 
To address the issue, we propose to first cluster nodes into $K$ groups and make the nodes in the same (for example, $k$-th) group share the intensity $\lambda_{k}^\ast(t)$ for reduced computation cost, i.e., from $\mathcal{O}(|\mathcal{N}_u|)$ to $\mathcal{O}(K)$.
\end{definition}
	
In Sec.~\ref{sec:encode}, we will show how to adapt the conditional intensity function\footnote{We term it as `intensity' for brevity in this paper.} as defined by Eq.~\eqref{eq:intensity} to the fine-grained modeling of a continuous-time graph $(\mathcal{G},\mathcal{O})$. While in Sec.~\ref{sec:learning} we present our training scheme that uses Eq.~\eqref{eq:tpp_mle} to account for the observed events $\mathcal{O}$ of edge addition on graph $\mathcal{G}$.
	
We note that the dynamic experimental protocol requires the model to \textit{predict what will happen at a given future time}, and it has already been shown more reasonable for evaluating continuous-time/discrete-time approaches in\cite{kazemi2020representation}.

\subsection{Related Works}\label{sec:related}
In literature, discrete-time dynamic graph approaches have been extensively studied due to ease of implementation. It can be mainly categorized into tensor decomposition~\cite{mao2018spatio,mcneil2021temporal,fang2022bayesian} and sequential modeling~\cite{li2018diffusion,zhang2018gaan,pareja2020evolvegcn,lan2022dstagnn,yu2018spatio,sankar2020dysat} two families. By contrast, continuous-time graph approaches are relatively less studied. In the following, we review the related works in two aspects: 
1) spatio-temporal decoupled modeling, where continuous-valued timestamps are directly encoded into vectorized representations; 
and 2) spatio-temporal entangled modeling, in which the structural information is considered in temporal modeling.
	
\subsubsection{Spatio-temporal Decoupled Modeling Methods}
Early works assume that temporal dynamics are independent of structural information in the graph. A line of works~\cite{kumar2019predicting,qu2020continuous,nguyen2018continuous,garcia2018learning} adopt temporally decayed function to make recently formed nodes/edges more important than previous ones. Differently, in \cite{li2020time} the elapsed time data is divided into bins, such that continuous-valued timestamps can be embedded into vector representations, analogous to positional encoding. This idea of time encoding is extended by~\cite{xu2020inductive,goel2020diachronic} using nonlinear neural networks and sinusoidal functions, respectively.  While the assumption of independence is desired for many cases, in some applications the structural and temporal dynamics are entangled. 

\subsubsection{Spatio-temporal Entangled Modeling Methods}
Recent effort has been made to capture the entangled structural and temporal dynamics over graph. In~\cite{mei2017neural,trivedi2019dyrep,zuo2020transformer}, the intensity function of TPP is parameterized with deep (mostly recurrent) neural networks to model spatiotemporal dynamics. These works focus on link prediction tasks and are orthogonal to us as they build task specific methods and do not focus on representation learning. 
In contrast,~\cite{chien2022learning,jin2022neural} include the ordinary differential equations (ODE) to capture continuous evolution of structural properties coupled with temporal dynamics. Typically, the ODE function is parameterized with an RNN cell and updates hidden states in a sequential manner, not optimal on industrial scale data.
\cite{cho2021learning} proposes a novel Attention-Attention architecture where the former attention captures periodic patterns for each node while the latter attention models the entangled structural and temporal information. Note that this method does not consider continuous-time evolving patterns.
	
As summarized in Table \ref{tab:algsurvey},	most of existing continuous-time approaches either perform time encoding or deep RNN to model the fine-grained temporal dynamics. By contrast, we propose a TPP-modulated attention network which can efficiently model the entangled spatiotemporal dynamics on industry-scale graph. 
As mentioned above, one shortcoming	of existing masking strategies is that they might not provide satisfactory results for node-level tasks. 
Our method tackles this problem by designing {\mreg} and {\mssl}.
Additionally, it is unclear for existing models which frequency content they are exploiting in making their predictions. We propose a perturbation-based spectral graph analyzing technique to provide insight in the graph Fourier domain.
Another observation is that in {\mframe}, the embedding for each node is essentially a function of time (spontaneous) which characterizes continuous evolution of node embedding, when no observations have been made on the node and its relations. 
This addresses the problem of memory staleness~\cite{kazemi2020representation} that discrete-time and event-triggered approaches suffer.

\section{Encoding: Attention-Intensity-Attention\\Scheme\,for\,Temporal-structural\,Embedding}\label{sec:encode}
In Sec.~\ref{sec:gatt}, we generalize the attention to continuous time space and discuss the motivation behind our idea. We then present in Sec.~\ref{sec:catnn} the detailed architecture that encodes the entangled structural and temporal dynamics on graph.

\subsection{Parameterizing Attention in Continuous Time}\label{sec:gatt}
Suppose that $v_u$ is the query node and $\mathcal{H}_t$ represents a set of historical events up to time $t$, in order to compute the node representation $\mathbf{h}_u^{(t)}$, it is necessary to know if node $v_u$ is still connected with node $v_i\in\mathcal{N}_u$ at a future time $t$. To account for this, we use $\mathrm{N}_i(\mathrm{d}t\big| \mathcal{H}_{t})\!\in\!\mathbb{N}_+$ to denote the number of edges $(v_i,v_u)$ appearing within a short time window $[t,t\!+\!\mathrm{d}t)$. For example in social networks, $\mathrm{N}_i(\mathrm{d}t\big| \mathcal{H}_{t})\!=\!2$ indicates that user $v_u$ revisits user $v_i$ twice during $[t,t\!+\!\mathrm{d}t)$. With this, we can reformulate the attention in Eq.~\eqref{eq:att} as:
\begin{flalign}
\mathrm{Att}(\mathbf{h}^{(\mathrm{d}t)}_u|\mathcal{N}_u, \mathbf{v})
    = \mathbf{E}_{\mathbf{v}_i\sim{Q}}
    \mathrm{N}_i(\mathrm{d}t\big| \mathcal{H}_{t})
    \mathbf{v}_i, \label{eq:tcats0}
\end{flalign}where $\mathbf{h}^{(\mathrm{d}t)}_u$ denotes node representation during the interval $[t,t\!+\!\mathrm{d}t)$ and the distribution $Q$ corresponds to the normalized attention coefficient in Eq. \eqref{eq:att}. However, it is intractable to train above attention network in Eq. \eqref{eq:tcats0} on large graphs, because the model parameters $\mathrm{N}_i(\mathrm{d}t\big|\mathcal{H}_{t})$ are restricted to be integers. To relax the integer constraint, one of feasible solutions is to replace the discrete counting number by its continuous relaxation $\mathbf{E}[\mathrm{N}_i(\mathrm{d}t\big|\mathcal{H}_{t})]$. To that end, we derive the generalized attention as follows:
\begin{flalign}
\mathrm{Att}(\mathbf{h}^{(\mathrm{d}t)}_u|\mathcal{N}_u, \mathbf{v})
&= \mathbf{E}_{\mathbf{v}_i\sim{Q}}
    \mathbf{E} \Big[\mathrm{N}_i(\mathrm{d}t\big|\mathcal{H}_{t})\Big]
    \mathbf{v}_i \nonumber\\
&= \mathbf{E}_{\mathbf{v}_i\sim{Q}}
    \lambda_i^\ast(t)
    \mathbf{v}_i\,\mathrm{d}t, \label{eq:gcats1}
\end{flalign}where Eq.~\eqref{eq:gcats1} holds due to Eq.~\eqref{eq:intensity} where $\lambda_i^\ast(t)\mathrm{d}t$ specifies the number of type-$i$ events, i.e., new edges on pair $(v_i,v_u)$ during $[t,t\!+\!\mathrm{d}t)$.
Since we are interested in the embedding $\mathbf{h}_u^{(t)}$ at a single point $t$ in the time dimension, we have:
\begin{equation}
\mathrm{Att}(\mathbf{h}^{(t)}_u|\mathcal{N}_u, \mathbf{v})
    = \mathbf{E}_{\mathbf{v}_i\sim{Q}}
    \lambda_i^\ast(t)
    \mathbf{v}_i. \label{eq:catf_point}
\end{equation}

\subsubsection{Further Discussion}
\textbf{Connections to Existing Attentions.}
We note that given $\mathcal{N}_u$, existing attention~\cite{brody2021attentive,vaswani2017attention,velivckovic2017graph,liu2020kalman} will produce identical node embedding $\mathbf{h}^{(t)}_u$ for different past timestamps $\{t_i\}$ and future timepoint $t$, thereby making attention \textit{a static quantity}. Essentially, these approaches can be viewed as our special cases with constant $\lambda_i^\ast(t)\!\equiv\!{1}$ for $v_i\in\mathcal{N}_u$. In other words, {\mframe} modulates the attention with the TPP intensity so as to energize attention itself \textit{a temporally evolving quantity}.
	
\begin{figure}[tb!]
\centerline{\includegraphics[width=.50\textwidth]{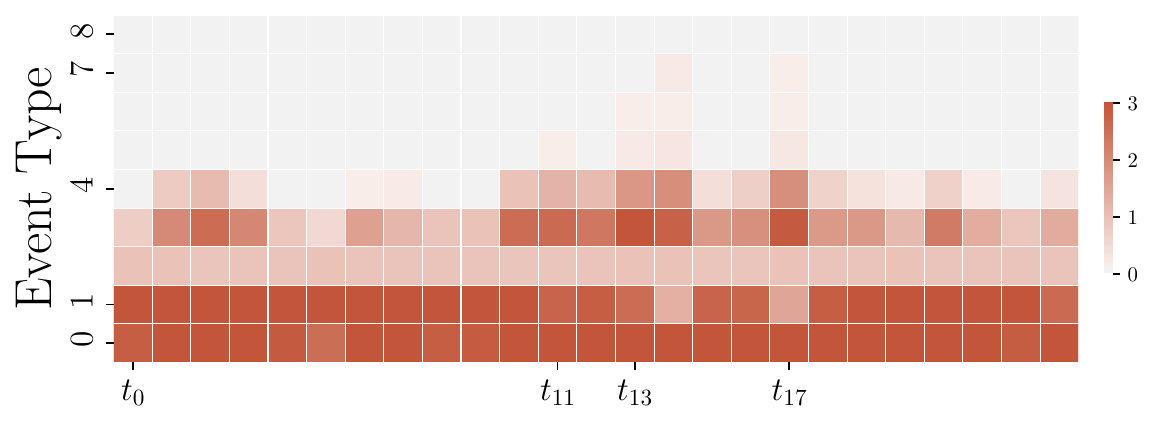}}
\vskip -0.1in
\caption{\label{fig:tf_trends}TPP-based interpreting of {\mname} for dynamic link prediction on Netflix, which shows a user's intensities for different event types (or item clusters defined in Eq.~\eqref{eq:scatf_point}) over time. One can see that this user has consistent interests in movies of type-0 and-1, while the interests in movies from type-5 to type-7 are transient, i.e., $[t_{11},t_{17}]$ in one day.}
\vskip -0.1in
\end{figure}

To clarify our advantages over static attention networks, we conclude three properties with examples from user-item behavior modelling (see Fig.~\ref{fig:tf_trends} for case studies on Netflix):
\begin{enumerate}[label=\textbf{Property \arabic*:},topsep=0in,leftmargin=0em,wide=0em]
		\item When $\lambda_i^\ast(t)\!=\!0$, the impact of neighbor $v_i$ is rejected. $\Rightarrow$ The items (e.g., baby gear) purchased far from now may not reflect user's current interests.
		\item When $\lambda_i^\ast(t)\!<\!1$, the impact of neighbor $v_i$ is attenuated. $\Rightarrow$ Buying cookies may temporarily depress the purchases of all desserts.
		\item When $\lambda_i^\ast(t)\!>\!1$, the impact of neighbor $v_i$ is amplified. $\Rightarrow$ Buying cookies may temporarily increase the probability of buying a drink.
\end{enumerate}

However, it is challenging to perform our attention on very large graphs, since the conditional intensity $\lambda_i^\ast(t)$ need to be computed $|\mathcal{N}_u|$ times for all $v_i\in\mathcal{N}_u$. One feasible solution is to cluster the nodes into groups, then we model the dynamics of the observed events at the group level:
\begin{flalign}
		\mathrm{Att}(\mathbf{h}^{(t)}_u|\mathcal{N}_u, \mathbf{v})
		&= \mathbf{E}_{\mathbf{v}_i\sim{Q}}
		\lambda_{k_i}^\ast(t)
		\mathbf{v}_i \nonumber\\
		&= \sum_{i\in\mathcal{N}_u}
		Q(\mathbf{v}_i|\mathbf{h}^{(t)}_u)
		\lambda_{k_i}^\ast(t)
		\mathbf{v}_i, \label{eq:scatf_point}
\end{flalign}where $k_i$ is the cluster identity for node $v_i$. The rationale behind this is that the nodes in one group share the dynamic patterns. This makes sense in many domains such as recommender systems where the products of the same genre have similar patterns. Also, this design shares the parameters in each cluster that mitigates data sparsity issue.

\begin{figure*}
		\centerline{\includegraphics[width=.98\textwidth]{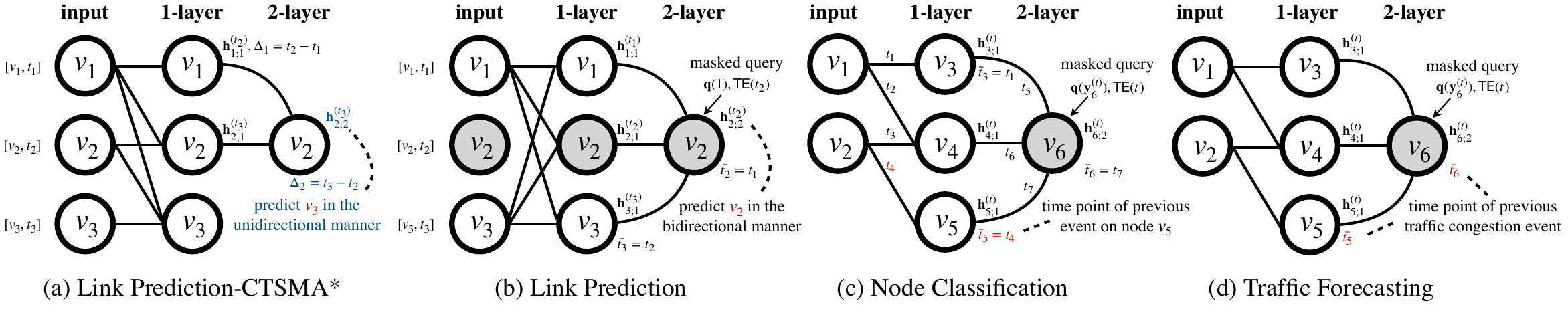}}
		\vspace{-10pt}
		\caption{\label{fig:cam}Encoding and training of {\mname} with three typical dynamic prediction tasks. The left (a) is performed in an auto-regressive manner; while the right three devised in this paper are based on auto-encoding. 
			{(a) \textbf{Our conference work~\cite{chen2021learning} (i.e. CTSMA)} tailored for dynamic link prediction;}
			(b) \textbf{Link prediction} (Sec.~\ref{sec:lpredc}): given a sequence of user behaviors $\{v_1,v_2,v_3\}$ on a user-item bipartite graph and $v_2$ is randomly selected for being masked, at the $1^\mathrm{st}$ layer the time delta as defined in Eq. \eqref{eqn:att_g} is set to $t_1-t_0$, $t_2-t_1$ and $t_3-t_2$ respectively, when computing $\mathbf{h}^{(t_1)}_{1;1},\mathbf{h}^{(t_2)}_{2;1}$ and $\mathbf{h}^{(t_3)}_{3;1}$. Note that the output $\mathbf{h}^{(t_2)}_{2;2}$ at the $2^\mathrm{nd}$ layer has direct access to information of all three time intervals with constant path length $\mathcal{O}(1)$, superior to recursive architectures~\cite{mei2017neural,trivedi2019dyrep}; 
			(c) \textbf{Node classification} (Sec.~\ref{sec:nclassf}): Let $v_6$ denote the masked query node whose input features are replaced by its true label $\mathbf{y}_6^{(t)}$, and the one-hop neighbors $v_3,v_4,v_5$ are most recently updated at time $t_1,t_3,t_4$ respectively, then the time intervals are set to $t-t_1,t-t_3,t-t_4$ when computing $\mathbf{h}_{3;1}^{(t)}, \mathbf{h}_{4:1}^{(t)},\mathbf{h}_{5;1}^{(t)}$; 
			(d) \textbf{Traffic forecasting} (Sec.~\ref{sec:tforecast}): $v_6$ denotes the masked query node whose input features are replaced by its true speed reading $\mathbf{y}_6^{(t)}$. Note that $\bar{t}_5$ ($\bar{t}_6$) represents the time point of previous traffic congestion event occurred on $v_5$ ($v_6$). }
\end{figure*}
\subsection{Instantiating Continuous Time Attention Encoder}\label{sec:catnn}
We next instantiate the encoding network in {\mname} as defined in Eq.~\eqref{eq:scatf_point}. Let $\mathbf{X}^f\!\in\!\mathbb{R}^{N\times{F}}$ and $\mathbf{X}^c\!\in\!\mathbb{R}^{N\times{K}}$ denote node feature matrix and cluster assignment matrix, where $N, F, K$ is the number of nodes, features and clusters respectively. $\mathbf{X}^c$ is obtained using weighted kernel K-means~\cite{dhillon2007weighted} based on the data $\mathbf{X}^f$, and assignments for new nodes can be decided inductively by using cluster centroids.
	
\subsubsection{Overview of the Encoding Network}
The encoding network consists of three key steps and follows an attention-intensity-attention scheme: 
\begin{enumerate}[label=\textbf{Step \arabic*:}, topsep=0in,leftmargin=0em,wide=0em]
		
    \item \textbf{Endogenous dynamic encoding (Attention).} Since the temporal and structural dynamics are entangled in domains like road networks, we propose a specific attention-intensity-attention design. 
    \reviewertwo{We first improve graph attention mechanism by taking into account non-constant interactions between event types in a point process, so as to encode the dynamic localized changes in the neighborhood (i.e., the endogenous dynamics). When an event changes spatial graph structures or node attributes, the representation outputted from the attention will be changed accordingly.} 

    \item \textbf{Conditional intensity modeling (Intensity).} Then we couple the structure-aware endogenous dynamics with the time-dependent exogenous dynamics by devising a new intensity function that specifies the occurrence of events on edges. Note that different input events and different event timestamps will result in different intensities.

    \item \textbf{Intensity-attention modulating (Attention)}. As the learned intensity naturally offers us the ability to specify the occurrence of edge-addition events at a specific future time, a new attention mechanism is proposed by using the intensity to modulate the attention in Step 1. The resulting attention can yield continuous-time node embeddings. 
\end{enumerate}

\subsubsection{Endogenous Dynamics Encoding}\label{sec:esta_att}
\reviewertwo{
To capture the multi-level interactions, we combine the node feature matrix $\mathbf{X}^f$ and the clustering assignment matrix $\mathbf{X}^c$ as an initial step, followed by linear transformation:
\begin{flalign}
    \mathbf{H}^{(t)}_{0} &=
    \big[\,\mathbf{X}^f \parallel \mathbf{X}^c\,\big] \label{eq:gcoding}\\
    \mathbf{V}^{(t)}_{l} &= \mathbf{W}^V\mathbf{H}^{(t)}_{l-1} \label{eq:vtransf}
\end{flalign}where $\parallel$ is the concatenation operation and $\mathbf{W}^V$ is a weight matrix that might be different at different layer $l$; $\mathbf{H}_{l-1}^{(t)}$ and $\mathbf{V}_l^{(t)}$ signifies the node representations outputted from layer $l\!-\!1$ and the value representations at layer $l$, respectively.

We next consider time-dependent interactions between event types, where we denote by $v_u$ query node, vector $\mathbf{d}^{(t)}_u$ the information of query time (e.g., month, day of the week) and $\mathbf{z}_u$ ($\mathbf{z}_i$) the embedding of event type on node $v_u$ ($v_i$):
\begin{flalign}\label{eq:tt_alpha}
    \mathbf{W}^{(t)}_u = \mathrm{MLP}(\mathbf{d}^{(t)}_u),\quad
    \alpha^{(t)}_{u,i}=\tanh(\mathbf{z}_u^\top\mathbf{W}^{(t)}_u\mathbf{z}_i).
\end{flalign}It can clearly be seen that $\alpha^{(t)}_{u,i}$ is different at different time $t$.

To encode fine-grained dynamics in graph structures and node attributes with neighborhood $\mathcal{N}_u$ (namely, the endogenous dynamics), we improve graph attention mechanism as:
\begin{flalign}
\mathbf{s}_{u;l}^{(t)} =
    \sum_{i\in\mathcal{N}_u} 
    \left( 
        \frac{\exp(e^{(t)}_{u,i} + \beta\alpha^{(t)}_{u,i})}
		{\sum_{i'\in\mathcal{N}_u}\exp(e^{(t)}_{u,i'} + \beta\alpha^{(t)}_{u,i'})}
    \right)\mathbf{v}^{(t)}_{i;l}, \label{eq:pp_att}
\end{flalign}where $\beta$ is a hyperparamter, $\mathbf{v}^{(t)}_{i;l}$ is the $i^\mathrm{th}$ column of $\mathbf{V}^{(t)}_{l}$, and $e^{(t)}_{u,i}$ is attention coefficient which can be implemented in various ways (e.g., SA~\cite{vaswani2017attention}, GAT~\cite{velivckovic2017graph} and GATv2~\cite{vaswani2017attention}).
}

\subsubsection{Conditional Intensity Modeling}\label{sec:esta_ints}
Recall that $\mathcal{H}_t\!=\!\{(t_i,k_i)|t_i<t\}$ denotes the history, where pair $(t_i,k_i)$ signifies the occurrence of type-$k_i$ event at time $t_i$ --- a new edge connecting $v_u$ and one node in group $k_i$ appears in graph at time $t_i$. To model the dynamics of events,	we combine the structure-aware endogenous dynamics $\mathbf{s}^{(l)}_{u;l}$ with the time-dependent exogenous dynamics:
\begin{equation}\label{eqn:att_g}
\begin{aligned}
\mathbf{g}^{(t)}_{k;l}
= \sigma \Big(
	\underbrace{ %
		\mathbf{W}^G_{k}\mathbf{s}^{(t)}_{u;l}
	}_\text{Endogenous}
+ \underbrace{ %
    \vphantom{\mathbf{W}^G_{k}\mathbf{s}^{(l)}_{u;l}}
	\mathbf{b}^G_{k}(t-\bar{t}_u)
	}_\text{Exogenous}
    \Big), 
\end{aligned}\end{equation}where $\mathbf{W}_k^G\!,\mathbf{b}^G_k$ are the model parameters for cluster $k$, and $\bar{t}_u$ signifies the time point of previous event for $v_u$. Fig.~\ref{fig:cam} provides examples of $\bar{t}_u$ in versatile applications.
	
The intensity $\lambda^\ast_k(t)$ for type-$k$ event takes the form:
\begin{flalign}\label{eqn:att_l}
\lambda^\ast_k(t)
= f_k \left(\mathbf{w}^\top_{k}\mathbf{g}^{(t)}_{k;l} + \mu_{k} \right),
\end{flalign}where $\mathbf{w}_{k}$ is the parameter for $k^\mathrm{th}$ cluster, and $\mu_{k}$ signifies the base occurrence rate. It has been well studied in~\cite{trivedi2019dyrep,zuo2020transformer} that the choice of activation function $f_k$ should consider two critical criteria: i) the intensity should be non-negative and ii) the dynamics for different event types evolve at different scales. To account for this, we adopt the softplus function:
\begin{flalign}\label{eq:splus}
    f_{k}(x)=\phi_{k}\log({1} + \exp(x/\phi_{k})),
\end{flalign}where the parameter $\phi_k$ captures the timescale difference.
	
One may raise concerns regarding Eq.~\eqref{eqn:att_g} that it might lose information by looking only at time point of previous event $\bar{t}_u$. We argue that as shown in Fig.~\ref{fig:cam}, the maximum path length between any two recent time points is $\mathcal{O}(1)$ in multi-layered architectures, superior to $\mathcal{O}(|\mathcal{N}_u|)$ for typical neural TPP models~\cite{mei2017neural,trivedi2019dyrep}.

\subsubsection{Intensity-Attention Modulating}\label{sec:esta_repr}
We use the TPP intensity to modulate the attention defined in Eq.~\eqref{eq:pp_att} by following Eq.~\eqref{eq:scatf_point}:
\begin{gather}\label{eq:iamolding}
\mathbf{h}_{u;l}^\mathrm{(t)} = 
    \sum_{i\in\mathcal{N}_u}
    \left( 
	\frac{\exp(e^{(t)}_{u,i} + \beta\alpha_{u,i}^{(t)})}
	{\sum_{i'\in\mathcal{N}_u}\exp(e^{(t)}_{u,i'}+ \beta\alpha_{u,i'}^{(t)})}
    \right)
\lambda^\ast_{k_i}(t)
    \mathbf{v}_{i;l}^{(t)}, 
\end{gather}where $\mathbf{h}_{u;l}^\mathrm{(t)}$ is in fact the $u^\mathrm{th}$ column of $\mathbf{H}^{(t)}_l$ that will be used as input of next layer to derive $\mathbf{H}^{(t)}_{l+1}$, and $k_i$ is the cluster index of neighbor $v_i$. In practise, using multi-head attention is beneficial to increase model flexibility. To facilitate this, we compute $\mathbf{h}_{u;l}^\mathrm{(t)}$ with $H$-independent heads:
\begin{gather}\label{eq:cat_mhcomb}
\mathbf{h}_{u;l}^\mathrm{(t)} = 
    \bigg\Arrowvert_{h=1}^{H}
    \mathbf{h}_{u;l;h}^\mathrm{(t)}, 
\end{gather}where $\mathbf{h}_{u;l;h}^\mathrm{(t)}$ is computed using different set of parameters.

\section{Training: Task-agnostic Posterior Maximization and Task-aware Masked Learning}\label{sec:learning}
We consider a principled learning scheme which consists of both task-agnostic loss based on posterior maximization that accounts for the whole event history on dynamic graph, as well as task-aware loss under the masking-based supervised learning. We note that the task-agnostic loss can be viewed as a regularizer term in addition to the discriminative supervised term for applications, ranging from link prediction, node classification and traffic forecasting.

\subsection{TPP Posterior Maximization of Events on Dynamic Graphs for Task-aware Learning}\label{ssec:treg}
Among existing solutions, temporal smoothness~\cite{zhu2016scalable,zhou2018dynamic} penalizes the distance between consecutive embeddings for every node to prevent sharp changes. However, this might not be desired in some applications such as recommender systems where a user's preference will change substantially from one time to another. To address the issue, we propose to maximize the TPP likelihood of events $\mathcal{H}_t$ ({\mreg}) that encourages the embeddings to learn evolving dynamics.
	
Let $\mathbf{y}$ denote the ground-truth label, $\hat{\mathbf{y}}_\Theta$ as the predicted result where $\Theta$ denotes model parameters, and function $\ell(\mathbf{y}, \hat{\mathbf{y}}_\Theta)$ measures the difference between label and prediction. Then, we optimize $\Theta$ by minimizing the risk:
\begin{equation}\label{eqn:ctr_loss}
\min_\Theta 
    \mathbf{E}_\mathbf{y} \ell(\mathbf{y}, \hat{\mathbf{y}}_\Theta) 
	- \gamma 
    \mathbf{E}_{u}\,R(\Theta; u),
\end{equation}where $\gamma$ is a hyper-parameter which controls the strength of the regularization; \reviewertwo{$R(\Theta;u)$ signifies our {\mreg} term which maximizes the log-likelihood of the whole observed events $\{(k_1,t_1),\dots,(k_n,t_n)\}$ for query node $v_u$ by using Eq.\eqref{eq:tpp_mle}:}
\begin{gather}\label{eq:ctreg}
{R}(\Theta;u) 
=\sum_{i=1}^{n} \log \lambda^\ast_{k_{i}}(t_i)
    - \int_{t_0}^{t_{n}}\lambda^\ast(s) \,\mathrm{d}s 
\end{gather}where recall that $\lambda^\ast(s)=\sum_k \lambda_k^\ast(s)$ is the conditional intensity for the entire history sequence. As stated in Sec.~\ref{sec:bk_tpp}, we recall that the sum term in the right hand side represents the log-likelihood of all the observations, while the integral term represents the log-survival probability that no event happens during the time window $[t_0,t_{n})$.
	
However, it is rather challenging to compute the integral $\Lambda\!\!=\!\!\int_{t_0}^{t_n}\!\lambda^\ast(s)\mathrm{d}s$. Due to the softplus used in Eq.~(\ref{eqn:att_l}), there is no closed-form solution for this integral. Subsequently, we consider the following two techniques to approximate $\Lambda$:
\begin{enumerate}[label=\textbf{\arabic*)},wide=0em,parsep=0em,itemindent=2em]

    \item Monte Carlo integration~\cite{metropolis1949monte}. 
    It randomly draws one set of samples $\{s_i\}_{i=1}^L$ in each interval $(t_{i-1}, t_i)$ to offer an unbiased estimation of $\Lambda$ (i.e., $\mathbf{E}\,[\widetilde{\Lambda}_\mathrm{MC}] = \Lambda$),
    \begin{gather}\label{eqn:mc}
        \widetilde{\Lambda}_\mathrm{MC} 
	= \sum_{i=1}^n
	   (t_i - t_{i-1})\left(
		\frac{1}{L} \sum_{j=1}^{L} \lambda^\ast(s_j)
		\right);
    \end{gather}
    
    \item Numerical integration~\cite{stoer2013introduction}.
    It is usually biased but fast due to the elimination of sampling. For example, trapezoidal rule approximates $\Lambda$ using the following functions,
    \begin{gather}\label{eqn:nu}
    \widetilde{\Lambda}_\mathrm{NU}
	=\sum_{i=1}^n
	   \frac{t_i-t_{i-1}}{2}
		\left(
		\lambda^\ast(t_i)
		+ \lambda^\ast(t_{i-1})
		\right).
    \end{gather}
\end{enumerate}
	
In experiments, we have found that the approximation defined in Eq.~(\ref{eqn:nu}) performs comparable to that in Eq.~(\ref{eqn:mc}) while consuming less computational time. Hence we choose numerical integration in our main approach. 
	
\begin{figure}[t!]
\centerline{\includegraphics[width=.45\textwidth]{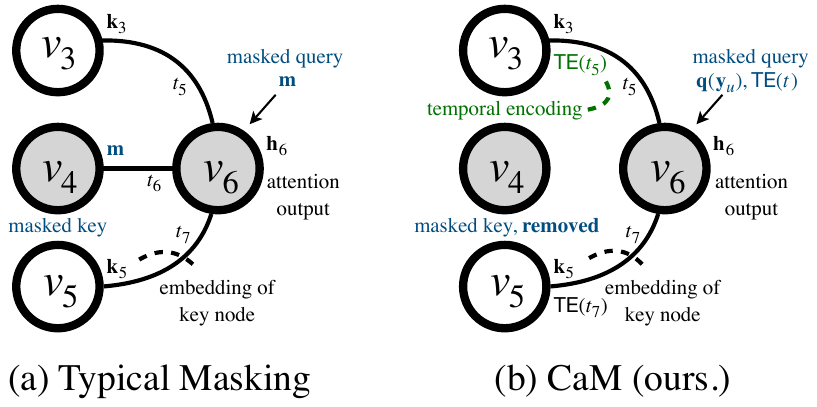}}
\vspace{-5pt}
\caption{\label{fig:cmpCaM}Comparison between typical masking used in~\cite{tailor2021degree,hou2022graphmae,thakoor2022largescale} and our proposed correlation-adjusted masking (\mssl) on graphs, where $v_3,v_4,v_5$ connected with $v_6$ at time $t_5,t_6,t_7$ respectively and $t$ signifies the future time of interest. We add time embedding TE($t_3$), TE($t_5$) to key nodes $v_3,v_5$ and TE($t$) to query node $v_6$.}
\vskip -0.1in
\end{figure}
\subsection{Correlation-adjusted Masking on Dynamic Graphs for Task-agnostic Learning}
\subsubsection{The Proposed Masking-based Learning Paradigm}
Before delving into our paradigm, we review recent masked graph models~\cite{tailor2021degree,hou2022graphmae,thakoor2022largescale} that replace a random part of nodes with special tokens. Let $\mathbf{m}$ denote the embedding of masked query node $v_u$, then the attention output is:
\begin{flalign}\label{eq:spmask}
\mathbf{h}_u
    \!=\!\sum_{i\in\mathcal{N}_u}\!\!
	\left( 
		\frac{\exp\left(
			e_{u,i}(\mathbf{m}, \mathbf{k}_i)
			\right)}
    {\sum_{i'}\exp\left(
	e_{u,i'}(\mathbf{m}, \mathbf{k}_i')
		\right)}
    \right)\mathbf{v}_i,
\end{flalign}where $\mathbf{k}_i$ is the embedding of the neighbor $i\in\mathcal{N}$. We note that different queries that are masked will share embedding $\mathbf{m}$, which indicates that the importance of different keys (indicated by the learned attention scores) are shared across these queries. This encourages the attention to have a global ranking of influential nodes, which makes sense in many link prediction applications such as recommender systems, as it exploits the items with high click-through rates. However, in some applications (e.g., traffic forecasting in Table~\ref{tab:motiv}) localized ranking of influential nodes might be desired.
	
To address the above issue, we design a new correlation-adjusted masking method, called {\mssl}. Specifically, let $\mathcal{M}_u$ denote the set of masked keys which will be removed and not used during training, and $\mathbf{q}(\mathbf{y}_u)$ represents the label-aware embedding of masked query (e.g., $\mathbf{q}(\mathbf{y}_u)=\mathbf{W}\mathbf{y}_u$). We then add sinusoidal temporal encodings to all keys and queries; without this, masked nodes will have no information about their location on dynamic graph:
\begin{flalign}\label{eq:te}
\mathrm{TE}(t, i) = \left\{
    \begin{array}{ll}
        \sin(t/10000^{j/d}), 
        \quad\quad\,\mathrm{if}~i~ \mathrm{is~even},\\
        \cos(t/10000^{j-1/d}), 
        \quad\mathrm{if}~i~ \mathrm{is~odd}.
\end{array}\right.
\end{flalign}where for key $v_i\in\mathcal{N}_u$, $t$ is time point of previous event on the pair ($v_u,v_i$), whereas for query $v_u$ it is the future time in request. We note that it is challenging to perform positional encoding technique (e.g., \cite{vaswani2017attention,li2020distance}) on dynamic graphs due to the evolving graph structures that require to recompute Laplacian eigenmap multiple times. To that end, we compute the attention output $\mathbf{h}_u$ for masked query by:
\begin{flalign}\label{eq:camask}
\mathbf{h}_u
    \!=\!\!\!\sum_{i\in\mathcal{N}_u-\mathcal{M}_u}\!\!
	\left( 
	\frac{\exp\Big(
		e_{u,i}(\mathbf{q}(\mathbf{y}_u), \mathbf{k}_i)
		\Big)}
	{\sum_{i'}\exp\Big(
		e_{u,i'}(\mathbf{q}(\mathbf{y}_u), \mathbf{k}_i')
		\Big)}
\right)\mathbf{v}_i,
\end{flalign}where recall that $\mathcal{M}_u$ denotes the set of masked keys. As shown in Fig.~\ref{fig:cmpCaM}, we summarize the difference between our {\mssl} method and typical masking~\cite{tailor2021degree,hou2022graphmae,thakoor2022largescale} on graph that randomly replaces node features with special tokens:
\begin{enumerate}[label=\textbf{\arabic*)},wide=0em,parsep=0em,itemindent=2em]

\item \reviewertwo{Typical masking methods replace masked key nodes with special tokens whose embedding is shared and represented by $\mathbf{m}$, whereas we remove these masked key nodes.} Even though both these two methods can create diversified training examples in multi-epoch training, our method operates only on a subset of graph nodes, resulting in a large reduction in computation and memory usage.
		
\item We add temporal embeddings to all keys and queries, which helps the model to focus on recently formed edges (recall that in our setting, $t$ in Eq.~\eqref{eq:te} signifies time point of previous event for key nodes) and certain past timestamps.
		
\item We adopt label-aware embedding $\mathbf{q}(\mathbf{y}_u)$ to represent the masked query node rather than the shared embedding $\mathbf{m}$. Equation~\eqref{eq:camask} shows that the key whose embedding $\mathbf{k}_i$ is closer to $\mathbf{q}(\mathbf{y}_u)$ will contribute more to the output $\mathbf{h}_u$. This essentially encourages the attention to focus on the keys that are highly correlated with the query label. We note that typical masking is a special case of {\mssl} when $\mathbf{y}_u$ consists only of ones but no zeros (also known as positive-unlabeled learning~\cite{du2014analysis}). This setting is very common in recommender systems where only ``likes'' (i.e., ones) are observed.
\end{enumerate}

\subsubsection{Task I: Learning for Dynamic Link Prediction}\label{sec:lpredc}
\textbf{Task Formulation.} The problem of link prediction is motivated in domains such as recommender systems. As defined in Definition~\ref{def:dyLink}, our goal in this case is to predict which item a user will purchase \textit{at a given future timestamp}, based on the observation of a past time window.
	
\textbf{Event.}
This work is focused on the edge-addition events which correspond to new user-item feedbacks in the context of recommender systems. This definition of event can address the core of the recommendation problem.
	
\textbf{Training Loss.}
Fig.~\ref{fig:cam}(b) presents the details about model encoding and training for link prediction. Given one user's history $\mathcal{N}_u$, at each time $t_i$ we have the embedding $\mathbf{h}_{i;L}^\mathrm{(t_i)}$ for node $v_i$ that is used to produce an output distribution:
\begin{flalign}
    \hat{\mathbf{y}}_{u}^{(t_i)} = \text{softmax}(\mathbf{h}_{i;L}^\mathrm{(t_i)}\mathbf{W}^O+\mathbf{b}^O)
\end{flalign}where $\mathbf{W}^O\!\in\!\mathbb{R}^{d\times{N}},\mathbf{b}^{O}\!\in\!\mathbb{R}^N$ are the model parameters and $N,d$ are separately the number of nodes and the embedding size. To optimize the overall parameters $\Theta$, we minimize the following objective function:
\begin{flalign}
\ell(\Theta) = 
    \mathbf{E}_u \frac{1}{|\mathcal{M}_u|} 
    \sum_{i,i'\in{\mathcal{M}_u}}
-\hat{\mathbf{y}}_{u,i}^{(t_i)}
    \log(\mathbf{y}_{u,i'}^{(t_i)})
    - \gamma 
\mathbf{E}_{u}\,R(\Theta; u), \nonumber
\end{flalign}where $\mathbf{y}_{u,i'}^{(t_i)}=1$ means that user $u$ bought item $i'$ at time $t_i$, such that $\mathbf{y}_{u,i'}^{(t_i)}=1$ only when $i'=i$ and $\mathbf{y}_{u,i'}^{(t_i)}=0$ otherwise; $R(\Theta; u)$ represents the {\mreg} term and $\mathcal{M}_u$ denotes the set of the masked nodes.
We note that the model could trivially predict the target node in multi-layered architectures, and hence the model is trained only on the masked nodes.
	
\textbf{Inference.}
In training, the model predicts each masked node $v_i$ with the embedding $\mathbf{q}(\mathbf{y}^{(t_i)}_{u,i})$, where $\mathbf{y}^{(t_i)}_{u,i}=1$ for all $v_i\in\mathcal{N}_u$. Hence, the embedding $\mathbf{q}(1)$ is shared across all masked nodes $v_i\in\mathcal{M}_u$ and does not disclose the identity of underlying node $v_i$. To predict the future behaviors at testing stage, we append $\mathbf{q}(1)$ to the end of the sequence $\mathcal{N}_u$, then exploit its output distribution to make predictions.

\subsubsection{Task II: Learning for  Dynamic Node Classification}\label{sec:nclassf}
\textbf{Task Formulation.} Node classification is one of the most widely adopted tasks for evaluating graph neural networks, and as defined in Definition~\ref{def:dyNode}, the algorithm is designed to label every node with a categorical class \textit{at a future time point}, given the observation of the history time window.

\textbf{Event.} Likewise, we focus on edge-addition events, for example on the Elliptic dataset, a new edge can be viewed as a flow of bitcoins from one transaction (node) to the other.
 
\textbf{Training Loss.}
Fig.~\ref{fig:cam}(c) presents details on model encoding and training for node classification. Given $\mathcal{N}_u$, we compute $\mathbf{h}^{(t_u)}_{u;L}$ for query $v_u$ by stacking $L$ layers to predict $\mathbf{y}_u^{(t_u)}$ at time $t_u$ (e.g., the time of next event on query $v_u$):
\begin{flalign}
\hat{\mathbf{y}}_{u}^{(t_u)} 
    = \operatorname{softmax}(\mathbf{h}^{(t_u)}_{u;L}\mathbf{W}^O) 
    + \mathbf{b}^O, 
\end{flalign}where $\mathbf{W}^O\!\in\!\mathbb{R}^{d\times{C}},\mathbf{b}^{O}\!\in\!\mathbb{R}^C$ are the model parameters and $C$ is the number of classes while $d$ is the embeddings size. 
To the end, we optimize the overall model parameters $\Theta$ by minimizing the following cross-entropy loss:
\begin{flalign}
\ell(\Theta) \!=\! 
    \mathbf{E}_{u\in\mathbb{V}}
    \mathrm{CrossEntropy}(\mathbf{y}_u^{(t_u)}, \hat{\mathbf{y}}_{u}^{(t_u)})
	- \gamma 
    \mathbf{E}_{u}\,R(\Theta; u), \nonumber
\end{flalign}where note that the loss is computed on all nodes $\mathbb{V}$, not limited to the masked ones. We empirically found that this stabilizes the training and increases the performance.

\textbf{Inference.}
During training, the embedding $\mathbf{q}(\mathbf{y}^{(t_u)}_u)$ for masked queries is derived from the label $\mathbf{y}^{(t_u)}_u$ which is not available in testing phase. Therefore, we disable the masking and use original node features $\mathbf{X}$, i.e., $\mathbf{H}_0=\mathbf{W}\mathbf{X}$ as model input. Warn that training the model solely on masked nodes will introduce a mismatch between the training and testing, which hurts the performance in general.

\begin{figure*}[t!]
\centerline{\includegraphics[width=.92\textwidth]{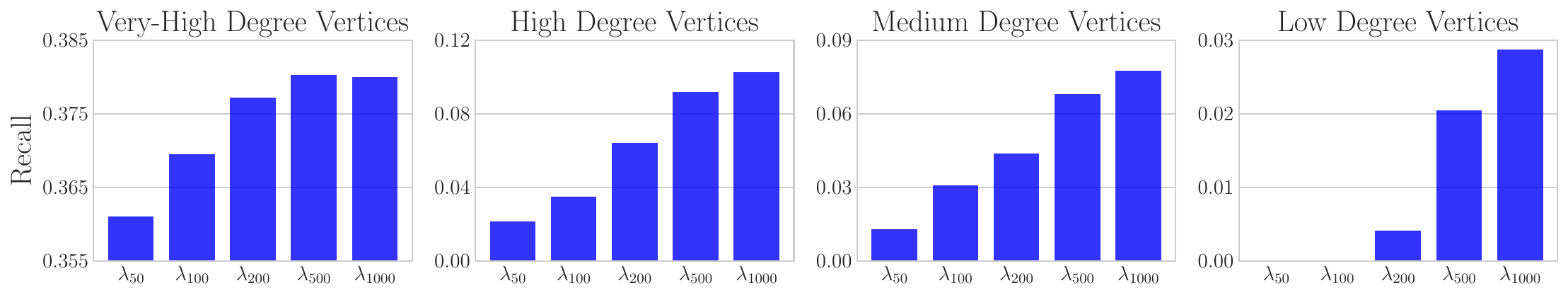}}
\caption{\label{fig:eigen_recall}Recall results of very-high degree vertices (left), high degree vertices (left middle), medium degree vertices (right middle) and low degree vertices (right) for link prediction on Netflix, where $\lambda_{50}$ corresponds to GSIMC~\cite{chen2023graph} by using representations with a frequency not greater than $\lambda_{50}$ to make predictions. The results show that \textit{low(high)-frequency representations reflect the user preferences on the popular(cold) items}.}
\end{figure*}
\subsubsection{Task III: Learning for Node Traffic Forecasting}\label{sec:tforecast}
\textbf{Task Formulation.} Traffic forecasting is the core component of intelligent transportation systems. As defined in Definition~\ref{def:traff}, the goal is to predict the future traffic speed readings on road networks given historical traffic data. 
	
	
\textbf{Event.}
In traffic networks, one road that is connected to congested roads would likely be congested. To account for this, when a road is congested, we regard it as the events of edge addition happening on its connected roads. To be more specific, we define the traffic congestion on a road when its traffic reading is significantly lower than hourly mean value at the $95\%$ confidence level.

\textbf{Training Loss.}
Fig.~\ref{fig:cam}(d) presents the details about model encoding and training for node traffic forecasting. Likewise, we predict the reading $\mathbf{y}^{(t)}_u$ at time $t$ (e.g., 15 min later) for road $v_u$ based on the time-conditioned representation $\mathbf{h}^{(t)}_{u;L}$:
\begin{flalign}
\hat{\mathbf{y}}_{u}^{(t)} =
    \left\langle 
	\mathbf{h}^{(t)}_{u;L}, \mathbf{w}^O
    \right\rangle + {b}^O, 
\end{flalign}where $\mathbf{w}^O\!\in\!\mathbb{R}^{d},{b}^{O}\!\in\!\mathbb{R}$ are the model parameters and $d$ is the embeddings size. To optimize the model parameters $\Theta$, we minimize the following least-squared loss:
\begin{flalign}
\ell(\Theta) = 
    \mathbf{E}_{u\in\mathbb{V}}
    \Big(
	\mathbf{y}_u^{(t)} - \hat{\mathbf{y}}_{u}^{(t)}
    \Big)^2
- \gamma 
    \mathbf{E}_{u}\,R(\Theta; u). \nonumber
\end{flalign}
	
\textbf{Inference.}
Analogous to the treatment for node classification, we switch off the masking and use the original node features as input in the testing stage.

\section{Interpreting: Global Analysis with Scalable and Stable Graph Spectral Analysis}\label{sec:interpret}
This section introduces a spectral perturbation technique to interpret the model output in the graph Fourier domain. To deal with very large graphs, a scalable and orthogonalized graph Laplacian decomposition algorithm is developed.

\subsection{Graph Fourier Transform}
\begin{definition}[\textbf{Graph Signal}]
Given any graph $\mathcal{G}$, the values residing on a set of nodes is referred as a graph signal. In matrix notation, graph signal can be represented by a vector $\mathbf{h}\!\in\!\mathbb{R}^N$ where $N$ is the number of nodes in graph $\mathcal{G}$.
\end{definition}
	
Let $\mathbf{H}^{(t)}$ ($\mathbf{Y}^{(t)}$) denote the embeddings (predictions) for each node at time $t$ where each $N$-size column of $\mathbf{H}^{(t)}$ ($\mathbf{Y}^{(t)}$) can be viewed as graph signal. In what follows, we consider signal's representations in the graph Fourier domain.
	
\begin{definition}[\textbf{Graph Fourier Transform}]
Let $\{\mathbf{u}_l\}$ and $\{\lambda_l\}$ denote the eigenvectors and eigenvalues of graph Laplacian $ \mathbf{L}=\mathbf{I}-\mathbf{D}^{-1/2}\mathbf{A}\mathbf{D}^{-1/2}$ where $\mathbf{D}$ is a diagonal degree matrix and $\mathbf{A}$ represents the affinity matrix, then for signal $\mathbf{f}\!\in\!\mathbb{R}^N$, we define the graph Fourier transform and its inverse by:
\begin{flalign}\label{eq:gft}
\widetilde{\mathbf{f}}(\lambda_l)=\sum_{i=0}^{N-1}  
    \mathbf{f}(i)\mathbf{u}_l(i)
	\quad\text{and}\quad
    \mathbf{f}(i) = \sum_{l=0}^{N-1}
\hat{\mathbf{f}}(\lambda_l)\mathbf{u}_l(i).
\end{flalign}where in Fourier analysis on graphs~\cite{ortega2018graph}, the eigenvalues $\{\lambda_l\}$ carry a notion of frequency: for $\lambda_l$ close to zero (i.e., low frequencies), the associated eigenvector varies litter between connected vertices. 
\end{definition}
	
For better understanding, Fig.~\ref{fig:eigen_recall} presents case studies in recommender systems on the Netflix prize data~\cite{bennett2007netflix}. Specifically, we divide the item vertices into four classes: very-high degree ($>\!5000$), high degree ($>\!2000$), medium degree ($>\!100$) and low degree vertices. Then, we report the recall results of spectral graph model GSIMC~\cite{chen2023graph} on four classes by only passing signals with a frequency not greater than $\lambda_{50},\dots,\lambda_{1000}$ to make top-$100$ recommendations. 
One can see that: 
1) the low-frequency signals with eigenvalues less than $\lambda_{100}$ contribute nothing to low degree vertices; 
2) the high-frequency signals with eigenvalues greater than $\lambda_{500}$ do not help increase the performance on very-high degree vertices. This finding reveals that low (high)-frequency signals reflect user preferences on popular (cold) items.

\subsection{Model-Level Interpreting in Graph Fourier Domain}\label{sec:gsanaly}
Most prior works on interpretable graph learning mainly often deal with the instance-level explanations in the graph vertex domain~\cite{yuan2022explainability}, e.g., in terms of the importance of node attributes indicated by the learned attention coefficients~\cite{fan2021gcn} and (sub)graph topologies based on the Shapley value~\cite{yuan2021explainability}. 
	
It has been extensively discussed in recent survey~\cite{yuan2022explainability} that the methods for model-level explanations are less studied, and existing works~\cite{yuan2020xgnn,lin2021generative} require extra training of reinforcement learning models.
	
By contrast, we aim at providing a global view of trained graph models in the graph Fourier domain. Compared with the explanations~\cite{yuan2020xgnn,lin2021generative} in vertex domain, it has a distinctive advantage since the frequency characterizes the smoothness of a graph signal (e.g., model predictions) with respect to the underlying graph structure. However, it is challenging to achieve this:
1) GFT itself is not capable of precisely quantifying the the predictive power of frequency content that a graph model learns from the data,
and 2) it is difficult to perform GFT on large graphs due to the $\mathcal{O}(N^3)$ cost of orthogonalized graph Laplacian decomposition.
	
\subsubsection{Spectral Perturbation on Graphs}
To enable the full explanatory power of GFT, we devise appropriate perturbations in the graph Fourier domain and measure their resulting change in prediction accuracy.
Let $\mathcal{S}$ denote the frequency band (e.g., greater than $\lambda_{500}$), column $\hat{\mathbf{y}}$ denotes the predictions (e.g., predicted user behaviors on all item nodes $\mathbb{V}$), then we define the intra-perturbations:
\begin{flalign}\label{eq:pertb}
\textit{perturb.}~\hat{\mathbf{y}}
    = \hat{\mathbf{y}} - \sum_{l\in\mathcal{S}} \mathbf{u}_l\mathbf{u}_l^\top\hat{\mathbf{y}}
\end{flalign}where $\{\mathbf{u}_l\}$ denotes the Laplacian eigenvectors. Essentially, Eq.~\eqref{eq:pertb} rejects the frequency content of signal $\hat{\mathbf{y}}$ in band $\mathcal{S}$.
	
\begin{proposition}
Given $\textit{perturb.}~\hat{\mathbf{y}}$ which is perturbed by Eq.~\eqref{eq:pertb}, its Fourier transform does not have support in $\mathcal{S}$, i.e.,
\begin{flalign}
\widetilde{\textit{perturb.}}~\hat{\mathbf{y}}(\lambda_l)=0
    \quad\text{for}\quad\lambda_l\in\mathcal{S}. \nonumber
\end{flalign}
\end{proposition}
\begin{proof}
Following Eq.~\eqref{eq:gft}, for $\lambda_l\in\mathcal{S}$ we have
\begin{flalign}
\widetilde{\textit{perturb.}}~\hat{\mathbf{y}}(\lambda_l) = 
	\mathbf{u}_l^\top\hat{\mathbf{y}} -
	\mathbf{u}_l^\top\sum_{l'\in\mathcal{S}} \mathbf{u}_{l'}\mathbf{u}_{l'}^\top\hat{\mathbf{y}}
	=\mathbf{u}_l^\top\hat{\mathbf{y}} - \mathbf{u}_l^\top\hat{\mathbf{y}}
	= 0, \nonumber
\end{flalign}where $\mathbf{u}_l^\top\mathbf{u}_{l'}=1$ if $l=l'$ and $\mathbf{u}_l^\top\mathbf{u}_{l'}=0$ otherwise.
\end{proof}
	
In practice, it is of importance to precisely quantify the predictive power of different frequency contents across two models. For instance, to exaimine if dynamic graph model $\hat{\mathbf{y}}_1$ (e.g., BERT4REC~\cite{sun2019bert4rec}) outperforms static graph model $\hat{\mathbf{y}}_2$ (e.g., GSIMC~\cite{chen2023graph}) due to better content at high frequencies $\mathcal{S}$ or not. To fulfill this, we propose the inter-perturbations:
\begin{flalign}\label{eq:interpertb}
\hat{\mathbf{y}}_1~\textit{perturb.}~\hat{\mathbf{y}}_2
    = \hat{\mathbf{y}}_2 
	 - \sum_{l\in\mathcal{S}} \mathbf{u}_l\mathbf{u}_l^\top\hat{\mathbf{y}}_2
    + \sum_{l\in\mathcal{S}} \mathbf{u}_l\mathbf{u}_l^\top\hat{\mathbf{y}}_1,
\end{flalign}where the frequency content of $\hat{\mathbf{y}}_2$ in band $\mathcal{S}$ is replaced by that of $\hat{\mathbf{y}}_1$.
	
\begin{proposition}
Given $\hat{\mathbf{y}}_1~\textit{perturb.}~\hat{\mathbf{y}}_2$ in Eq.~\eqref{eq:interpertb}, it satisfies 
\begin{flalign}
\hat{\mathbf{y}}_1~\widetilde{\textit{perturb.}}~\hat{\mathbf{y}}_2(\lambda_l) = \left\{
    \begin{array}{ll}
        \tilde{\mathbf{y}}_2(\lambda_l), 
        \quad\quad\,\mathrm{if}~\lambda_l\notin\mathcal{S},\\
        \tilde{\mathbf{y}}_1(\lambda_l), 
        \quad\quad\,\mathrm{if}~\lambda_l\in\mathcal{S}.
    \end{array}\right. \nonumber
\end{flalign}where $\widetilde{\mathbf{y}}_1(\lambda_l)$ and $\widetilde{\mathbf{y}}_2(\lambda_l)$ signify the information of $\hat{\mathbf{y}}_1$ and $\hat{\mathbf{y}}_2$ at frequency $\lambda_l$ in the graph Fourier domain, respectively.
\end{proposition}
\begin{proof}
As Eq.~\eqref{eq:interpertb} only changes the content in band $\mathcal{S}$, it is trivial to conclude $\hat{\mathbf{y}}_1~\widetilde{\textit{perturb.}}~\hat{\mathbf{y}}_2(\lambda_l)=\widetilde{\mathbf{y}}_2(\lambda_l)$ for $\lambda_l\notin\mathcal{S}$. 
Then for $\lambda_l\in\mathcal{S}$, applying Eq.~\eqref{eq:gft} leads to
\begin{flalign}
\hat{\mathbf{y}}_1~&\widetilde{\textit{perturb.}}~\hat{\mathbf{y}}_2(\lambda_l) \nonumber\\
	&= \mathbf{u}_l^\top\hat{\mathbf{y}}_2 
	    - \mathbf{u}_l^\top\sum_{l'\in\mathcal{S}} \mathbf{u}_{l'}\mathbf{u}_{l'}^\top\hat{\mathbf{y}}_2 
	+ \mathbf{u}_l^\top\sum_{l'\in\mathcal{S}} \mathbf{u}_{l'}\mathbf{u}_{l'}^\top\hat{\mathbf{y}}_1
	 \nonumber\\
	&= \mathbf{u}_l^\top\hat{\mathbf{y}}_2 
	   - \mathbf{u}_l^\top\hat{\mathbf{y}}_2 
	   + \mathbf{u}_l^\top\hat{\mathbf{y}}_1 \nonumber\\
	&= \mathbf{u}_l^\top\hat{\mathbf{y}}_1 = \widetilde{\mathbf{y}}_1, \nonumber
\end{flalign}where $\mathbf{u}_l^\top\mathbf{u}_{l'}=1$ if $l=l'$ and $\mathbf{u}_l^\top\mathbf{u}_{l'}=0$ otherwise.
\end{proof}
	
We highlight the qualitative results in Sec.~\ref{sec:lp_qual} that shows BERT4REC takes benefit of high-frequency signals ($>\lambda_{500}$) to achieve better accuracy than GSIMC on Netflix data.

\begin{table*}[tb!]
\caption{\label{tab:cmp_hr}Hit-Rate ({HR}) on the Koubei, Tmall and Netflix datasets for dynamic link prediction, where our model is compared against spatial graph models (i.e., GAT~\cite{velivckovic2017graph} and SAGE~\cite{hamilton2017inductive}), spectral graph models (i.e., GCN~\cite{kipf2016semi}, ChebyNet~\cite{defferrard2016convolutional}, ARMA~\cite{bianchi2021graph}, MRFCF~\cite{steck2019markov}, GSIMC and BGSIMC~\cite{chen2023graph}), sequential recommendation models (i.e., GRU4REC~\cite{hidasi2016session}, SASREC~\cite{kang2018self}, GREC~\cite{yuan2020future}, S2PNM~\cite{chen2021dyna} and BERT4REC~\cite{sun2019bert4rec}) and continuous-time graph models (i.e., DyREP~\cite{trivedi2019dyrep}, TGAT~\cite{xu2020inductive}, TiSASREC~\cite{li2020time}, TGREC~\cite{fan2021continuous}, TimelyREC~\cite{cho2021learning} and CTSMA~\cite{chen2021learning}). Bolded numbers denote the best, and the standard errors of the ranking metrics are less than 0.01 for all the results. $^*$Conference version of this work based on which we obtain the relative gain.}  
\begin{center}\resizebox{0.98\textwidth}{!}{
\begin{tabular}{l ccc c ccc c ccc} \toprule\toprule
    & 
    \multicolumn{3}{c}{\textbf{Koubei, Density=$\mathbf{0.08\%}$}}  && 
    \multicolumn{3}{c}{\textbf{Tmall, Density=$\mathbf{0.10\%}$}} &&
    \multicolumn{3}{c}{\textbf{Netflix, Density=$\mathbf{1.41\%}$}} \\\cmidrule{2-4}\cmidrule{6-8}\cmidrule{10-12}
    \textbf{Model}   &
    \textbf{HR@10} & \textbf{HR@50} & \textbf{HR@100} &&
    \textbf{HR@10} & \textbf{HR@50} & \textbf{HR@100} &&
    \textbf{HR@10} & \textbf{HR@50} & \textbf{HR@100} \\\midrule
    \textbf{GAT}~\cite{velivckovic2017graph} &  
    0.19715 & 0.26440 & 0.30125 && 
    0.20033 & 0.32710 & 0.39037 && 
    0.08712 & 0.19387 & 0.27228 \\
    \textbf{SAGE}~\cite{hamilton2017inductive}  & 
    0.20600 & 0.27225 & 0.30540 && 
    0.19393 & 0.32733 & 0.39367 && 
    0.08580 & 0.19187 & 0.26972 \\\midrule
    \textbf{GCN}~\cite{kipf2016semi}  & 
    0.20090 & 0.26230 & 0.30345 && 
    0.19213 & 0.32493 & 0.38927 && 
    0.08062 & 0.18080 & 0.26720 \\
    \textbf{ChebyNet}~\cite{defferrard2016convolutional} &  
    0.20515 & 0.28100 & 0.32385 && 
    0.18163 & 0.32017 & 0.39417 && 
    0.08735 & 0.19335 & 0.27470 \\
    \textbf{ARMA}~\cite{bianchi2021graph}  &
    0.20745 & 0.27750 & 0.31595 && 
    0.17833 & 0.31567 & 0.39140 && 
    0.08610 & 0.19128 & 0.27812 \\
    \textbf{MRFCF}~\cite{steck2019markov}  &
    0.17710 & 0.19300 & 0.19870 &&
    0.19123 & 0.28943 & 0.29260 && 
    0.08738 & 0.19488 & 0.29048 \\
    \textbf{GSIMC}~\cite{chen2023graph}  &
    0.23460 & 0.31995 & 0.35065 && 
    0.13677 & 0.31027 & 0.40760 && 
    0.09725 & 0.22733 & 0.32225  \\
    \textbf{BGSIMC}~\cite{chen2023graph}  &
    0.24390 & 0.32545 & 0.35345 && 
    0.16733 & 0.34313 & 0.43690 && 
    0.09988 & 0.23390 & 0.33063  \\\midrule  
					
    \textbf{GRU4REC}~\cite{hidasi2016session}  &
    0.26115 & 0.33495 & 0.37375 &&
    0.24747 & 0.37897 & 0.44980 && 
    0.23442 & 0.40445 & 0.49175 \\
    \textbf{SASREC}~\cite{kang2018self}  &
    0.24875 & 0.32805 & 0.36815 && 
    0.25597 & 0.39040 & 0.45803 && 
    0.22627 & 0.39222 & 0.48085 \\
    \textbf{GREC}~\cite{yuan2020future}  &
    0.23831 & 0.31533 & 0.35431 && 
    0.22470 & 0.35057 & 0.41507 && 
    0.24338 & 0.41690 & 0.50483 \\
    \textbf{S2PNM}~\cite{chen2021dyna}  &
    0.26330 & 0.33770 & 0.37730 && 
    0.25470 & 0.38820 & 0.45707 && 
    0.24762 & 0.42160 & 0.50858 \\
    \textbf{BERT4REC}~\cite{sun2019bert4rec}  &
    0.25663 & 0.33534 & 0.37483 && 
    0.26810 & 0.40640 & 0.47520 && 
    0.25337 & 0.43027 & 0.52115 \\\midrule
					
    \textbf{DyREP}~\cite{trivedi2019dyrep}  &
    0.26360 & 0.33490 & 0.37150 && 
    0.25937 & 0.38593 & 0.44903 && 
    0.22883 & 0.40997 & 0.50093 \\
    \textbf{TGAT}~\cite{xu2020inductive}  &
    0.25095 & 0.32530 & 0.36195 && 
    0.26070 & 0.39533 & 0.46503 && 
    0.22755 & 0.39623 & 0.48232 \\
    \textbf{TiSASREC}~\cite{li2020time}  &
    0.25295 & 0.33115 & 0.37110 && 
    \cellcolor{gray90}{0.26303} & 
    \cellcolor{gray90}{0.40463} & 
    \cellcolor{gray90}{0.47390} && 
    0.24355 & 0.41935 & 0.50767 \\
    \textbf{TGREC}~\cite{fan2021continuous}  &
    0.25175 & 0.31870 & 0.35655 && 
    0.24577 & 0.37960 & 0.44827 && 
    0.22733 & 0.38415 & 0.47170 \\
    \textbf{TimelyREC}~\cite{cho2021learning}  &
    0.25565 & 0.32835 & 0.37085 && 
    0.24280 & 0.36853 & 0.43240 && 
    0.23018 & 0.39772 & 0.48608 \\
    \reviewerone{\textbf{TiCoSeREC}}~\cite{dang2023uniform} &
    \reviewerone{0.24795} &
    \reviewerone{0.33485} &
    \reviewerone{0.37975} &&
    \reviewerone{0.25440} &
    \reviewerone{0.38430} &
    \reviewerone{0.45563} &&
    \reviewerone{0.22410} &
    \reviewerone{0.41090} &
    \reviewerone{0.50875} \\
    \reviewerone{\textbf{PTGCN}}~\cite{huang2023position} &
    \reviewerone{0.24075} &
    \reviewerone{0.32065} & 
    \reviewerone{0.36230} &&
    \reviewerone{0.23440} &
    \reviewerone{0.36430} &
    \reviewerone{0.43563} &&
    \reviewerone{0.19888} &
    \reviewerone{0.38927} &
    \reviewerone{0.47220} \\
    \reviewerone{\textbf{NeuFilter}}~\cite{xia2024neural} &
    \reviewerone{0.26915} & 
    \reviewerone{0.34340} & 
    \reviewerone{0.38210} &&
    \reviewerone{0.26810} & 
    \reviewerone{0.39520} & 
    \reviewerone{0.47550} &&
    \reviewerone{0.25565} &
    \reviewerone{0.42980} &
    \reviewerone{0.51522} \\
    \midrule
				    
    \rowcolor{gray90}
    \textbf{CTSMA}~\cite{chen2021learning} (conference ver.)  &
    0.27460 & 0.35250 & 0.39240 &&
    0.26760 & 0.40477 & 0.47310 &&
    0.25405 & 0.43723 & 0.52613 \\
					
    \reviewertwo{\textbf{\mframe} (ours)}  &
    \reviewertwo{\textbf{0.29740}} & 
    \reviewertwo{\textbf{0.40045}} & 
    \reviewertwo{\textbf{0.44820}} && 
    \reviewertwo{\textbf{0.27693}} &
    \reviewertwo{\textbf{0.42240}} & 
    \reviewertwo{\textbf{0.49420}} && 
    \reviewertwo{\textbf{0.29651}} & 
    \reviewertwo{\textbf{0.48301}} & 
    \reviewertwo{\textbf{0.57026}} \\
    \reviewertwo{\textbf{Rel. Gain} (over \cite{chen2021learning})} &
    \reviewertwo{8.3\%}  & 
    \reviewertwo{13.6\%} & 
    \reviewertwo{14.2\%} && 
    \reviewertwo{3.5\%}  & 
    \reviewertwo{4.4\%} &  
    \reviewertwo{4.5\%} && 
    \reviewertwo{16.7\%} & 
    \reviewertwo{10.5\%}  & 
    \reviewertwo{8.4\%} \\\bottomrule\bottomrule
    \end{tabular}
}\end{center} 
\end{table*}
\subsubsection{Scalable Graph Laplacian Decomposition for Graph Fourier Transforms with Orthogonality}\label{sec:interpret_decomp}
To address the scalability issue, we propose a scalable and orthogonalized graph Laplacian decomposition algorithm, which differs to~\cite{kumar2009sampling,halko2011finding} as the approximate eigenvectors are not orthogonal in these works. Note that orthogonality makes signal’s spectral representation computing fast and avoid the frequency components of signal’s spectrum interacting with each other in terms of frequency which thus could provide a more clean interpretability.
	
Specifically, we uniformly sample an $s$-column submatrix $\mathbf{C}$ from graph Laplacian matrix $\mathbf{L}$ and then define the normalized matrix $\mathbf{W}$ by following \cite{fowlkes2004spectral}:
\begin{flalign}
		\mathbf{C}= \left[\begin{array}{l}
			\mathbf{A} \\ \mathbf{B}^\top
		\end{array}\right]
		\quad\mathrm{and}\quad
		\mathbf{W}= \mathbf{A}^{-1/2}\mathbf{C}^{\top}\mathbf{CA}^{-1/2}.
\end{flalign}We adopt the randomized algorithm~\cite{halko2011finding} to approximate the eigenvectors of matrix $\mathbf{W}$ as an initial solution $(\mathbf{\Sigma}_W,\mathbf{U}_W)$, then by running the {\nystrom} method~\cite{fowlkes2004spectral} to extrapolate this solution to the full matrix $\mathbf{L}$.
		
\begin{algorithm}[t!]
\caption{Scalable and Orthogonalized Graph Laplacian Decomposition for Graph Fourier Transforms}
\label{alg:eigen}
\begin{algorithmic}[1]
    \REQUIRE  matrix $\mathbf{C}\in\mathcal{R}^{N\times{s}}$ derived from $s$ columns sampled from matrix $\mathbf{L}\!\in\!\mathcal{R}^{N\times{N}}$ without replacement, matrix $\mathbf{A}\!\in\!\mathcal{R}^{s\times{s}}$ composed of the intersection of these $s$ columns, matrix $\mathbf{W}\!\in\!\mathcal{R}^{s\times{s}}$, rank $r$, the oversampling parameter $p$ and the number of power iterations $q$.
    \ENSURE eigenvalues $\widetilde{\mathbf{\Sigma}}$ and eigenvectors $\mathbf{\widetilde{U}}$.
    \STATE Generate a random Gaussian matrix $\mathbf{\Omega} \in \mathbb{R}^{s\times{(r+p)}}$, then compute the sample matrix $\mathbf{A}^q\mathbf{\Omega}$.

   \STATE Perform QR-Decomposition on $\mathbf{A}^q\mathbf{\Omega}$ to obtain an orthonormal matrix $\mathbf{Q}$ that satisfies the equation $\mathbf{A}^q\mathbf{\Omega}=\mathbf{QQ}^\top\mathbf{A}^q\mathbf{\Omega}$, then solve $\mathbf{ZQ}^\top\mathbf{\Omega}=\mathbf{Q}^\top\mathbf{W\Omega}$.

   \STATE Compute the eigenvalue decomposition on the $(r+p)$-by-$(r+p)$ matrix $\mathbf{Z}$, i.e., $\mathbf{Z}=\mathbf{U_Z\Sigma_Z{U}_Z}^\top$, to obtain $\mathbf{U}_W=\mathbf{QU}_Z[:, :r]$ and $\mathbf{\Sigma}_W=\mathbf{\Sigma}_Z[:r, :r]$.

   \STATE Return $\widetilde{\mathbf{\Sigma}}\gets\mathbf{\Sigma}_W$, $\mathbf{\widetilde{U}}\gets\mathbf{CA}^{-1/2}\mathbf{U}_W\mathbf{\Sigma}_W^{-1/2}$.
\end{algorithmic}
\end{algorithm}

The pseudo-code is provided in Algorithm~\ref{alg:eigen}. To summarize, the computational complexity is $\mathcal{O}(Nsr+r^3)$ where $N,s,r$ is the number of nodes, sampled columns and eigenvectors, respectively. Recall that typically $N\!\gg\!{s}\!\ge\!{r}$, and hence our Algorithm~\ref{alg:eigen} has a lower complexity than typical {\nystrom} method, i.e., $\mathcal{O}(Nsr+s^3)$ in \cite{kumar2009sampling}.

\subsubsection{Error Bound of Algorithm~\ref{alg:eigen}}\label{sec:err_bound}
We further analyze the error bound of Algorithm~\ref{alg:eigen} with respect to the spectral norm. 
Theorem~\ref{th:Snorm} suggests that our Algorithm~\ref{alg:eigen} is as accurate as typical {\nystrom} method (i.e., $\parallel\mathbf{L}-\widetilde{\mathbf{L}}\parallel_2 + \frac{N}{\sqrt{s}}\mathbf{L}^\ast_{i,i}$ in Corollary 2~\cite{kumar2009sampling}) when the number of power iterations $q$ is large enough, while our algorithm fulfills the orthogonality constraint. 
The proof details are provided in \textcolor{blue}{\url{https://thinklab.sjtu.edu.cn/project/EasyDGL/suppl.pdf}} (see also in the supplemental material). 
	
\begin{theorem}[\textbf{Error Bound}]\label{th:Snorm}
For $\widetilde{\mathbf{L}}=\mathbf{\widetilde{U}\widetilde{\Sigma}\widetilde{U}}^\top$ in Algorithm~\ref{alg:eigen}, and $\zeta=1 + \sqrt{\frac{r}{p-1}} + \frac{e\sqrt{r+p}}{p}\sqrt{s-r}$,
\begin{flalign}
    \mathbf{E} 
	\parallel\mathbf{L}-\widetilde{\mathbf{L}} \parallel_2
    \le \zeta^{1/q}
	\parallel\mathbf{L}-\widetilde{\mathbf{L}}_r \parallel_2
	+ \left(1 + \zeta^{1/q}\right) \frac{N}{\sqrt{s}}\mathbf{L}^\ast_{i,i}, \nonumber
\end{flalign}where $\mathbf{L}^\ast_{i,i}=\max_i \mathbf{L}_{i,i}$ and $\mathbf{L}_r$ is the best rank-$r$ approximation.
\end{theorem}
	
\section{Experiments and Discussion}\label{sec:expr}
The experiments evaluate the overall performance of our pipeline {\mname} as well as the effects of its components, including {\mreg} and {\mssl}. We further perform qualitative analysis to provide a global view of static graph models and dynamic graph models. Experiments are conducted on Linux workstations with Nvidia RTX8000 (48GB) GPU and Intel Xeon W-3175X CPU@ 3.10GHz with 128GB RAM. The source code is implemented by Deep Graph Library (DGL, \textcolor{blue}{https://www.dgl.ai/}) hence our approach name easyDGL also would like to respect this library. Note that the detailed settings for our proposed model and the baselines could be found in the supplemental material or our open-source link.

\subsection{Evaluation of Overall Performance} \label{sec:expr-topn}
\subsubsection{Dynamic Link Prediction}
\textbf{Dataset.}
We adopt three large-scale real-world datasets: 
(1) Koubei\footnote{\textcolor{blue}{https://tianchi.aliyun.com/dataset/dataDetail?dataId=53}} ($223,044$ nodes and $1,828,250$ edges);  
(2) Tmall\footnote{\textcolor{blue}{https://tianchi.aliyun.com/dataset/dataDetail?dataId=35680}} ($342,373$ nodes and $7,632,826$ edges);
(3) Netflix\footnote{\textcolor{blue}{https://kaggle.com/netflix-inc/netflix-prize-data}} ($497,959$ nodes and $100,444,166$ edges). 
For each dataset, a dynamic edge $(v_u,v_i, t_k)$ indicates if user (i.e., node) $v_u$ purchased item (i.e., node) $v_i$ at time $t_k$.
Note that these datasets are larger than Wikidata\footnote{\textcolor{blue}{https://github.com/nle-ml/mmkb/TemporalKGs/wikidata}} ($11,134$ nodes and $150,079$ edges), Reddit\footnote{\textcolor{blue}{http://snap.stanford.edu/data/soc-RedditHyperlinks.html}} ($55,863$ nodes and $858,490$ edges).

\begin{table*}[tb!]
\caption{\label{tab:cmp_ndcg}NDCG ({N}) on Koubei, Tmall and Netflix for dynamic link prediction. The standard errors of the ranking metrics are all less than 0.005.}
\begin{center}\resizebox{0.98\textwidth}{!}{
    \begin{tabular}{l ccc c ccc c ccc} \toprule\toprule
    & 
    \multicolumn{3}{c}{\textbf{Koubei, Density=$\mathbf{0.08\%}$}}  && 
    \multicolumn{3}{c}{\textbf{Tmall, Density=$\mathbf{0.10\%}$}} &&
    \multicolumn{3}{c}{\textbf{Netflix, Density=$\mathbf{1.41\%}$}} \\\cmidrule{2-4}\cmidrule{6-8}\cmidrule{10-12}
    \textbf{Model}   &
    \textbf{N@10} & \textbf{N@50} & \textbf{N@100} &&
    \textbf{N@10} & \textbf{N@50} & \textbf{N@100} &&
    \textbf{N@10} & \textbf{N@50} & \textbf{N@100} \\\midrule

	\textbf{GAT}~\cite{velivckovic2017graph} & 
	0.15447 & 0.16938 & 0.17534 && 
	0.10564 & 0.13378 & 0.14393 && 
	0.04958 & 0.07250 & 0.08518 \\
	\textbf{SAGE}~\cite{hamilton2017inductive}  & 
	0.15787 & 0.17156 & 0.17701 && 
	0.10393 & 0.13352 & 0.14417 && 
	0.04904 & 0.07155 & 0.08419 \\\midrule
	\textbf{GCN}~\cite{kipf2016semi}  & 
	0.15537 & 0.16848 & 0.17548 && 
	0.10287 & 0.13208 & 0.14260 && 
	0.04883 & 0.06965 & 0.08456  \\
	\textbf{ChebyNet}~\cite{defferrard2016convolutional} & 
	0.15784 & 0.17406 & 0.18055 && 
	0.09916 & 0.12955 & 0.14175 && 
	0.04996 & 0.07268 & 0.08582  \\
	\textbf{ARMA}~\cite{bianchi2021graph}  & 
	0.15830 & 0.17320 & 0.17954 && 
	0.09731 & 0.12628 & 0.13829 && 
	0.04940 & 0.07192 & 0.08526  \\
	\textbf{MRFCF}~\cite{steck2019markov}  &
	0.10037 & 0.10410 & 0.10502  &&
	0.08867 & 0.11223 & 0.11275 && 
	0.05235 & 0.08047 & 0.09584  \\
	\textbf{GSIMC}~\cite{chen2023graph}  &
	0.17057 & 0.18970 & 0.19468 && 
	0.07357 & 0.11115 & 0.12661 && 
	0.05504 & 0.08181 & 0.09759  \\
	\textbf{BGSIMC}~\cite{chen2023graph}  &
	0.17909 & 0.19680 & 0.20134 && 
	0.09222 & 0.13082 & 0.14551 && 
	0.05593 & 0.08400 & 0.09982  \\\midrule

    \textbf{GRU4REC}~\cite{hidasi2016session}  &
    0.20769 & 0.22390 & 0.23017 &&
    0.16758 & 0.19641 & 0.20770 && 
    0.14908 & 0.18648 & 0.20064 \\
    \textbf{SASREC}~\cite{kang2018self}  &
    0.19293 & 0.21173 & 0.21612 && 
    0.17234 & 0.20207 & 0.21292 && 
    0.14570 & 0.18209 & 0.19644 \\
    \textbf{GREC}~\cite{yuan2020future}  &
    0.19115 & 0.20769 & 0.21445 && 
    0.14969 & 0.17750 & 0.18775 && 
    0.15674 & 0.19455 & 0.20894 \\
    \textbf{S2PNM}~\cite{chen2021dyna}  &
    0.20538 & 0.22171 & 0.22812 && 
    0.17250 & 0.20210 & 0.21335 && 
    0.15902 & 0.19750 & 0.21141 \\
    \textbf{BERT4REC}~\cite{sun2019bert4rec}  &
    0.19944 & 0.21685 & 0.22325 && 
    0.18042 & 0.21096 & 0.22193 && 
    0.16274 & 0.20166 & 0.21638 \\\midrule
					
    \textbf{DyREP}~\cite{trivedi2019dyrep}  &
    0.21086 & 0.22684 & 0.23240 && 
    0.17290 & 0.20071 & 0.21092 && 
    0.14593 & 0.18567 & 0.20043 \\
    \textbf{TGAT}~\cite{xu2020inductive}  &
    0.19803 & 0.21347 & 0.21933 && 
    0.17497 & 0.20419 & 0.21533 && 
    0.14541 & 0.18219 & 0.19611 \\
    \textbf{TiSASREC}~\cite{li2020time}  &
    0.19731 & 0.21518 & 0.22153 && 
    0.17502 & 0.20609 & 0.21754 && 
    0.15455 & 0.19272 & 0.20687 \\
    \textbf{TGREC}~\cite{fan2021continuous}  &
    0.19829 & 0.21226 & 0.21807 && 
    0.16378 & 0.19334 & 0.20457  && 
    0.14877 & 0.18285 & 0.19711 \\
    \textbf{TimelyREC}~\cite{cho2021learning}  &
    0.20530 & 0.22036 & 0.22725 && 
    0.16656 & 0.19434 & 0.20472 && 
    0.14615 & 0.18301 & 0.19734 \\

    \reviewerone{\textbf{TiCoSeREC}~\cite{dang2023uniform}} &
    \reviewerone{0.18985} & 
    \reviewerone{0.20623} &
    \reviewerone{0.21298} &&
    \reviewerone{0.16215} &
    \reviewerone{0.19059} &
    \reviewerone{0.20213} &&
    \reviewerone{0.14351} &
    \reviewerone{0.18445} &
    \reviewerone{0.20033} \\
    \reviewerone{\textbf{PTGCN}~\cite{huang2023position}} &
    \reviewerone{0.18899} &
    \reviewerone{0.20744} & 
    \reviewerone{0.21440} && 
    \reviewerone{0.14427} &
    \reviewerone{0.17609} &
    \reviewerone{0.18642} &&
    \reviewerone{0.14405} &
    \reviewerone{0.18237} &
    \reviewerone{0.19660}
    \\
    \reviewerone{\textbf{NeuFilter}~\cite{xia2024neural}} &
    \reviewerone{0.21296} & 
    \reviewerone{0.22921} & 
    \reviewerone{0.23542} &&
    \reviewerone{0.18259} & 
    \reviewerone{0.21397} & 
    \reviewerone{0.22543} &&
    \reviewerone{0.16417} &
    \reviewerone{0.20154} &
    \reviewerone{0.21638} \\\midrule
					
    \rowcolor{gray90}
    \textbf{CTSMA}~\cite{chen2021learning} (conference ver.)  & 
    0.21059 & 0.22767 & 0.23396 && 
    0.18020 & 0.21039 & 0.22164 && 
    0.16137 & 0.20172 & 0.21613 \\
					
    \reviewertwo{\textbf{\mframe} (ours)} & 
    \reviewertwo{\textbf{0.23092}} & 
    \reviewertwo{\textbf{0.25350}} & 
    \reviewertwo{\textbf{0.26123}} && 
    \reviewertwo{\textbf{0.21723}} & 
    \reviewertwo{\textbf{0.23617}} & 
    \reviewertwo{\textbf{0.24052}} && 
    \reviewertwo{\textbf{0.19059}} & 
    \reviewertwo{\textbf{0.23179}} & 
    \reviewertwo{\textbf{0.24594}} \\
    \reviewertwo{\textbf{Rel. Gain} (over \cite{chen2021learning})} &
    \reviewertwo{9.7\%}  & 
    \reviewertwo{11.3\%}  & 
    \reviewertwo{11.7\%} && 
    \reviewertwo{20.5\%} & 
    \reviewertwo{12.3\%} & 
    \reviewertwo{8.5\%} && 
    \reviewertwo{18.1\%} & 
    \reviewertwo{14.9\%} & 
    \reviewertwo{13.8\%}
    \\\bottomrule\bottomrule
\end{tabular}
}\end{center}
\end{table*}	
	
\textbf{Protocol.}
We follow the inductive experimental protocol \cite{liang2018variational} together with the dynamic protocol~\cite{trivedi2019dyrep} to evaluate the learned node representations, where the goal is to predict which is the most likely node $v$ that a given node $u$ would connect to at a specified future time $t$. Note that test node $u$ is not seen in training and the time $t$ can be arbitrary in different queries. To do so, we split user nodes into training, validation and test set with ratio $8\!:\!1\!:\!1$,
where all the data from the training users are used to train the model. At test stage, we sort all links of validation/test users in chronological order, and hold out the last one for testing. 
	
The results are reported in terms of the hit rate (HR) and the normalized discounted cumulative gain (NDCG) on the test set for the model which achieves the best results on the validation set. The implementation of HR and NDCG follows \cite{liang2018variational}, and the reported results are averaged over five different random data splits.

\textbf{Baseline.}
We compare state-of-the-art baselines: (1) static graph approaches. We follow the design in \textbf{IDCF}~\cite{wu2021towards} to generate a user's representation by aggregating the information in her/his purchased items, then we report the performance by using different backbone GNNs, i.e., GraphSAGE~\cite{hamilton2017inductive}, GAT\cite{velivckovic2017graph}, GCN~\cite{kipf2016semi}, ChebyNet~\cite{defferrard2016convolutional} and ARMA~\cite{bianchi2021graph}. Also, we compare with shallow spectral graph models including \textbf{MRFCF}~\cite{bianchi2021graph}, \textbf{GSIMC} and \textbf{BGSIMC}~\cite{chen2023graph}; (2) sequential recommendation approaches. \textbf{GRU4REC}~\cite{hidasi2016session} utilizes recurrent neural networks to model the sequential information, \textbf{SASREC}~\cite{kang2018self} and \textbf{GREC}~\cite{yuan2020future} adopt attention networks and causal convolutional neural networks to learn the long-term patterns. \textbf{S2PNM}~\cite{chen2021dyna} proposes a hybrid memory networks to establish a balance between the short-term and long-term patterns. \textbf{BERT4REC}~\cite{sun2019bert4rec} extends the Cloze task to SASREC for improvements; (3) continuous-time approaches. \textbf{TGAT}\cite{xu2020inductive}, \textbf{TiSASREC}~\cite{li2020time}, \textbf{TGREC}~\cite{fan2021continuous},
\reviewerone{\textbf{NeuFilter}\cite{xia2024neural}, \textbf{PTGCN}~\cite{huang2023position} and \textbf{TiCoSeREC}~\cite{dang2023uniform}} encode the temporal information independently into embeddings. \textbf{DyREP}~\cite{trivedi2019dyrep}, \textbf{TimelyREC}~\cite{cho2021learning} and \textbf{CTSMA}~\cite{chen2021learning} model the entangled spatiotemporal information on the dynamic graph.


\textbf{Results.}
Table~\ref{tab:cmp_hr} and Table~\ref{tab:cmp_ndcg} show the HR and NDCG of our {\mname} and the baselines on the Koubei, Tmall and Netflix datasets. 
It is shown that sequential models achieve better results than static graph models, and continuous-time graph models (i.e., DyREP, TimelyREC, CSTMA, TiSASREC, TGREC, TiSASREC, TGAT, \reviewerone{TiCoSeREC, PTGCN, TiCoSeREC} and {\mname}) are among the best on all three datasets. \reviewertwo{In particular, {\mname} achieves $8.4-16.7\%$ relative hit rate gains and $13.8-18.1\%$ relative NDCG gains over the best baseline on challenging Netflix}.
	
In addition, we emphasize {\mname}'s consistent performance boosting over our preliminary conference version CTSMA~\cite{chen2021learning}. The key difference is that CTSMA is a unidirectional autoregressive model. By contrast, {\mname} is a bidirectional masked autoencoding model that incorporates contexts from both left and right sides. Perhaps more importantly, {\mname} enjoys the benefit of improved robustness and generalizability, due to the masking strategy that introduces randomness into the model and also generates diversified samples for training.
	
\begin{table*}[t!]
\caption{\label{tab:nclassify}Evaluation on the Elliptic dataset (bitcoin transactions) for dynamic node classification over time, where {\mname} performs best for all columns at different time steps between static graph models (top half) and dynamic graph models (bottom half). 
\reviewerone{Note that most of dynamic graph models, such as {ROLAND}~\cite{you2022roland}, DEFT~\cite{bastos2023learnable} and SpikeNet~\cite{li2023scaling}, require multiple (i.e. 5) historical graph snapshots (see the `input' column), whereas {\mname} only accesses the update-to-date snapshot for prediction. {\mname} outperforms the best baseline TGN in accuracy with 4X speedup for training. All the methods perform poorly starting at time step 43, because this time is when the bitcoin dark market shutdown occurred which changes the external environment of the graph yet such information cannot be readily encoded by the evaluated models neither in training nor testing. In this extreme case {\mname} performs better than baselines, especially with respect to the Micro-F1 score which is for the minority (illicit) class.}}
\vspace{-10pt}
\begin{center}\resizebox{\textwidth}{!}{
\begin{tabular}{l c ccc c ccc c ccc} \toprule\toprule
    &  &
    \multicolumn{3}{c}{\textbf{Timestamp=[40, 42]}} && 
    \multicolumn{3}{c}{\textbf{Timestamp=[43, 46]}} &&
    \multicolumn{3}{c}{\textbf{Timestamp=[47, 49]}} \\\cmidrule{3-5}\cmidrule{7-9}\cmidrule{11-13}
    \textbf{Model}    & \textbf{Input} &
    \textbf{Macro-F1} & \textbf{Micro-F1} & \textbf{Accuracy} &&
    \textbf{Macro-F1} & \textbf{Micro-F1} & \textbf{Accuracy} &&
    \textbf{Macro-F1} & \textbf{Micro-F1} & \textbf{Accuracy} \\\midrule
    \textbf{APPNP}~\cite{gasteiger2018combining}   & 1 &
    0.666$\pm$0.009 & 0.542$\pm$0.020 & 0.776$\pm$0.010 &&
    0.476$\pm$0.028 & 0.014$\pm$0.020 & 0.782$\pm$0.008 && 
    0.462$\pm$0.029 & 0.000$\pm$0.000 & 0.792$\pm$0.007 \\
    \textbf{GAT}~\cite{velivckovic2017graph} & 1 &
    0.660$\pm$0.021 & 0.517$\pm$0.085 & 0.754$\pm$0.016 && 
    0.465$\pm$0.010 & 0.018$\pm$0.018 & 0.731$\pm$0.017 && 
    0.452$\pm$0.024 & 0.013$\pm$0.028 & 0.735$\pm$0.043 \\
    \textbf{SAGE}~\cite{hamilton2017inductive} & 1 &
    0.691$\pm$0.016 & 0.514$\pm$0.067 & 0.820$\pm$0.015 && 
    0.507$\pm$0.009 & 0.011$\pm$0.024 & 0.824$\pm$0.013 && 
    0.493$\pm$0.035 & 0.024$\pm$0.047 & 0.831$\pm$0.030 \\
					
    \textbf{ARMA}~\cite{bianchi2021graph}  & 1 &
    0.707$\pm$0.031 & \cellcolor{gray90}{0.556$\pm$0.050} & 0.825$\pm$0.007 && 
    0.512$\pm$0.027 & 0.041$\pm$0.024 & 0.826$\pm$0.009 && 
    0.498$\pm$0.021 & 0.016$\pm$0.036 & 0.834$\pm$0.012 \\
    \textbf{ChebyNet}~\cite{defferrard2016convolutional} & 1 &
    0.682$\pm$0.014 & 0.517$\pm$0.031 & 0.800$\pm$0.024 && 
    0.499$\pm$0.030 & 0.000$\pm$0.000 & 0.790$\pm$0.040 && 
    0.472$\pm$0.037 & 0.000$\pm$0.000 & 0.795$\pm$0.050 \\
    \textbf{GCN}~\cite{kipf2016semi}  &  1 &
    0.540$\pm$0.047 & 0.231$\pm$0.097 & 0.780$\pm$0.012 && 
    0.459$\pm$0.014 & 0.021$\pm$0.025 & 0.799$\pm$0.005 && 
    0.520$\pm$0.015 & 0.018$\pm$0.036 & 0.884$\pm$0.005  \\\midrule

    \textbf{TGAT}~\cite{xu2020inductive}  & 1 &
    0.693$\pm$0.027 & 0.529$\pm$0.035 & 0.803$\pm$0.009  && 
    0.501$\pm$0.010 & \cellcolor{gray90}{0.046$\pm$0.036} & 0.805$\pm$0.006  && 
    0.495$\pm$0.025 & 0.047$\pm$0.028 & 0.813$\pm$0.009 \\  
    \textbf{DySAT}~\cite{sankar2020dysat}  & 5 &
    0.611$\pm$0.020 & 0.477$\pm$0.061 & 0.739$\pm$0.009 && 
    0.469$\pm$0.011 & 0.021$\pm$0.018 & 0.744$\pm$0.016 && 
    0.470$\pm$0.020 & 0.000$\pm$0.000 & 0.774$\pm$0.034 \\
    \textbf{EvolveGCN}~\cite{pareja2020evolvegcn}  & 5 &
    0.591$\pm$0.012 & 0.382$\pm$0.019 & 0.735$\pm$0.002 && 
    0.458$\pm$0.005 & 0.000$\pm$0.000 & 0.734$\pm$0.008 && 
    0.443$\pm$0.035 & 0.037$\pm$0.022 & 0.727$\pm$0.024 \\
    \textbf{JODIE}~\cite{kumar2019predicting} & 5 &
    0.696$\pm$0.021 & 0.486$\pm$0.064 & 0.814$\pm$0.017 && 
    0.525$\pm$0.011 & 0.025$\pm$0.030 & 0.836$\pm$0.009 && 
    0.544$\pm$0.008 & 0.000$\pm$0.000 & 0.873$\pm$0.011 \\
    \textbf{TGN}~\cite{rossi2020temporal} & 5 &
    \cellcolor{gray90}{0.719$\pm$0.016} & 
    0.542$\pm$0.023 & 
    \cellcolor{gray90}{0.834$\pm$0.009} && 
    \cellcolor{gray90}{0.529$\pm$0.010} & 
    0.024$\pm$0.028 & 
    0.827$\pm$0.021 && 
    0.572$\pm$0.013 & 
    \cellcolor{gray90}{0.071$\pm$0.011} & 
    0.885$\pm$0.042 \\

    \reviewerone{\textbf{ROLAND}~\cite{you2022roland}} & 
    \reviewerone{5} &
    \reviewerone{0.637$\pm$0.034} & 
    \reviewerone{0.501$\pm$0.045} & 
    \reviewerone{0.792$\pm$0.023} && 
    \reviewerone{0.453$\pm$0.023} & 
    \reviewerone{0.000$\pm$0.000} & 
    \reviewerone{0.803$\pm$0.013} && 
    \reviewerone{0.525$\pm$0.014} & 
    \reviewerone{0.000$\pm$0.000} & 
    \reviewerone{0.884$\pm$0.011} \\
    
    \reviewerone{\textbf{DEFT}~\cite{bastos2023learnable}} & 
    \reviewerone{5} &
    \reviewerone{0.684$\pm$0.011} & 
    \reviewerone{0.516$\pm$0.034} & 
    \cellcolor{gray90}{0.833$\pm$0.004} && 
    \reviewerone{0.489$\pm$0.014} & 
    \reviewerone{0.013$\pm$0.021}  & 
    \reviewerone{0.834$\pm$0.011} && 
    \reviewerone{0.519$\pm$0.014} & 
    \reviewerone{0.004$\pm$0.007} & 
    \reviewerone{0.880$\pm$0.021} \\
 
    \reviewerone{\textbf{SpikeNet}~\cite{li2023scaling}} & 
    \reviewerone{5} &
    \reviewerone{0.703$\pm$0.010} & 
    \reviewerone{0.520$\pm$0.021} & 
    \cellcolor{gray90}{0.833$\pm$0.007} && 
    \reviewerone{0.528$\pm$0.008} & 
    \reviewerone{0.005$\pm$0.007} & 
    \cellcolor{gray90}{0.856$\pm$0.005} && 
    \cellcolor{gray90}{0.590$\pm$0.010}	& 
    \reviewerone{0.036$\pm$0.033} & 
    \cellcolor{gray90}{0.902$\pm$0.002} \\
    \midrule
					
    \reviewertwo{\textbf{\mframe}}  & 1 &
    \reviewertwo{\textbf{0.748$\pm$0.016}} & 
    \reviewertwo{\textbf{0.570$\pm$0.013}} & 
    \reviewertwo{\textbf{0.847$\pm$0.009}} && 
    \reviewertwo{\textbf{0.547$\pm$0.002}} & 
    \reviewertwo{\textbf{0.081$\pm$0.023}} & 
    \reviewertwo{\textbf{0.853$\pm$0.004}} && 
    \reviewertwo{\textbf{0.598$\pm$0.008}} & 
    \reviewertwo{\textbf{0.134$\pm$0.009}} & 
    \reviewertwo{\textbf{0.905$\pm$0.008}}  \\\bottomrule\bottomrule
\end{tabular}
}\end{center}
\end{table*}
\begin{table*}[t!]
\caption{\label{tab:traffic}
Evaluation on the META-LA dataset for 15 minutes (horizon 3), 30 minutes (horizon 6) and 45 minutes (horizon 9) ahead traffic forecasting, where the baselines include static graph models (top half) and dynamic graph models (bottom half). 
\reviewerone{Note that DCRNN\cite{li2018diffusion} , DySAT\cite{sankar2020dysat}, GaAN\cite{zhang2018gaan}, STGCN\cite{yu2018spatio}, DSTAGNN\cite{lan2022dstagnn}, STID\cite{shao2022spatial}, Trafformer\cite{jin2023trafformer} and TrendGCN\cite{jiang2023enhancing} use 60 minutes as historical time window, namely the most recent 12 graph snapshots (see the `input' column) are used to forecast traffic conditions in the next 15, 30 and 45 minutes, whereas {\mname} makes predictions only based on current graph snapshot. Specifically, {\mname} performs competitively with DSTAGNN with 6X speedup for training.}}
\begin{center}\resizebox{\textwidth}{!}{
\begin{tabular}{l c ccc c ccc c ccc} \toprule\toprule
    &  &
    \multicolumn{3}{c}{\textbf{Horizon=3 (15 minutes)}} && 
    \multicolumn{3}{c}{\textbf{Horizon=6 (30 minutes)}} &&
    \multicolumn{3}{c}{\textbf{Horizon=9 (45 minutes)}} \\\cmidrule{3-5}\cmidrule{7-9}\cmidrule{11-13}
    \textbf{Model}   &  \textbf{Input}  &
    \textbf{MAE} & \textbf{RMSE} & \textbf{MAPE} &&
    \textbf{MAE} & \textbf{RMSE} & \textbf{MAPE} &&
    \textbf{MAE} & \textbf{RMSE} & \textbf{MAPE} \\\midrule
    \textbf{APPNP}~\cite{gasteiger2018combining}  & 1 &
    3.987$\pm$0.010 & 6.708$\pm$0.015 & 0.117$\pm$0.000 &&
    4.496$\pm$0.011 & 7.769$\pm$0.017 & 0.137$\pm$0.000 &&
    4.902$\pm$0.008 & 8.579$\pm$0.021 & 0.156$\pm$0.001 \\
    \textbf{GAT}~\cite{velivckovic2017graph} & 1 & 
    4.992$\pm$0.237 & 8.587$\pm$0.470 & 0.155$\pm$0.010 &&
    5.228$\pm$0.184 & 9.049$\pm$0.348 & 0.164$\pm$0.008 &&
    5.456$\pm$0.163 & 9.473$\pm$0.278 & 0.174$\pm$0.006 \\
    \textbf{SAGE}~\cite{hamilton2017inductive}  &  1 &
    3.318$\pm$0.004 & 6.097$\pm$0.007 & 0.088$\pm$0.000 &&
    3.866$\pm$0.007 & 7.291$\pm$0.009 & 0.110$\pm$0.001 &&
    4.259$\pm$0.007 & 8.134$\pm$0.030 & 0.127$\pm$0.001 \\
    \textbf{ARMA}~\cite{bianchi2021graph}  & 1 &
    3.298$\pm$0.008 & 6.058$\pm$0.017 & 0.087$\pm$0.001 &&
    3.824$\pm$0.006 & 7.221$\pm$0.011 & 0.108$\pm$0.000 &&
    4.191$\pm$0.007 & 8.017$\pm$0.018 & 0.124$\pm$0.000 \\
					
    \textbf{ChebyNet}~\cite{defferrard2016convolutional} & 1 &  
    3.457$\pm$0.016 & 6.352$\pm$0.022 & 0.091$\pm$0.002 &&
    4.025$\pm$0.002 & 7.571$\pm$0.028 & 0.112$\pm$0.000 &&
    4.466$\pm$0.023 & 8.390$\pm$0.045 & 0.129$\pm$0.000 \\
    
    \textbf{GCN}~\cite{kipf2016semi}  & 1 & 
    5.748$\pm$0.013 & 9.780$\pm$0.036 & 0.188$\pm$0.001 &&
    5.813$\pm$0.011 & 9.988$\pm$0.031 & 0.192$\pm$0.000 &&
    5.901$\pm$0.009 & 10.21$\pm$0.026 & 0.198$\pm$0.001 \\\midrule
					
    \textbf{DCRNN}~\cite{li2018diffusion}  & 12 &
    3.002$\pm$0.018 & 5.699$\pm$0.019 & 0.079$\pm$0.000 && 
    3.429$\pm$0.027 & 6.782$\pm$0.018 & 0.095$\pm$0.001 && 
    3.735$\pm$0.045 & 7.493$\pm$0.039 & 0.107$\pm$0.001 \\
    \textbf{DySAT}~\cite{sankar2020dysat}  & 12 &
    3.115$\pm$0.012 & 6.101$\pm$0.039 & 0.085$\pm$0.000 && 
    3.710$\pm$0.015 & 7.391$\pm$0.041 & 0.108$\pm$0.001 && 
    4.184$\pm$0.018 & 8.317$\pm$0.031 & 0.128$\pm$0.001 \\
    \textbf{GaAN}~\cite{zhang2018gaan}  & 12 &
    3.029$\pm$0.013 & 5.799$\pm$0.043 & 0.083$\pm$0.001 && 
    3.580$\pm$0.017 & 7.068$\pm$0.063 & 0.104$\pm$0.002 && 
    3.994$\pm$0.022 & 7.960$\pm$0.084 & 0.122$\pm$0.002 \\
    \textbf{STGCN}~\cite{yu2018spatio}  & 12 &
    3.195$\pm$0.027 & 6.240$\pm$0.050 & 0.089$\pm$0.001 && 
    3.392$\pm$0.025 & 6.806$\pm$0.049 & 0.097$\pm$0.001 && 
    3.575$\pm$0.027 & 7.251$\pm$0.036 & 0.104$\pm$0.001 \\
    \textbf{DSTAGNN}~\cite{lan2022dstagnn}  & 12 &
    2.960$\pm$0.035 & 5.611$\pm$0.052 & 0.079$\pm$0.001 && 
    3.321$\pm$0.030 & 6.683$\pm$0.056 & 0.093$\pm$0.001 && 
    3.595$\pm$0.029 & 7.249$\pm$0.058 & 0.102$\pm$0.000 \\
    
    \reviewerone{\textbf{STID}~\cite{shao2022spatial}} & 
    \reviewerone{12} &
    \reviewerone{2.783$\pm$0.004} & 
    \reviewerone{5.553$\pm$0.010} & 
    \reviewerone{0.074$\pm$0.001} &&
    \reviewerone{3.167$\pm$0.004} & 
    \reviewerone{6.636$\pm$0.007} & 
    \reviewerone{0.090$\pm$0.001} &&
    \reviewerone{3.382$\pm$0.005} & 
    \reviewerone{7.223$\pm$0.024} & 
    \reviewerone{0.099$\pm$0.001} \\
    \reviewerone{\textbf{Trafformer}~\cite{jin2023trafformer}} & 
    \reviewerone{12} &
    \reviewerone{2.835$\pm$0.044} & 
    \reviewerone{5.598$\pm$0.099}  & 
    \cellcolor{gray90}{0.070$\pm$0.001}  && 
    \reviewerone{3.183$\pm$0.041}  & 
    \reviewerone{6.530$\pm$0.091}  & 
    \cellcolor{gray90}{0.083$\pm$0.001}  && 
    \reviewerone{3.430$\pm$0.042}  & 
    \reviewerone{7.137$\pm$0.071}  & 
    \cellcolor{gray90}{0.091$\pm$0.001}  \\
    \reviewerone{\textbf{TrendGCN}~\cite{jiang2023enhancing}} & 
    \reviewerone{12} &
    \cellcolor{gray90}{2.777$\pm$0.017}  & 
    \cellcolor{gray90}{5.352$\pm$0.040}  & 
    \reviewerone{0.074$\pm$0.001}  && 
    \cellcolor{gray90}{3.150$\pm$0.024}  & 
    \cellcolor{gray90}{6.422$\pm$0.058}  & 
    \reviewerone{0.088$\pm$0.001}  && 
    \cellcolor{gray90}{3.393$\pm$0.028}  & 
    \cellcolor{gray90}{7.042$\pm$0.066}  & 
    \reviewerone{0.098$\pm$0.002}  \\

    \midrule
    \reviewertwo{\textbf{\mframe}}  & 1 &
    \reviewertwo{\textbf{2.714$\pm$0.011}} & 
    \reviewertwo{\textbf{5.218$\pm$0.007}} & 
    \reviewertwo{\textbf{0.067$\pm$0.001}} && 
    \reviewertwo{\textbf{3.104$\pm$0.014}} & 
    \reviewertwo{\textbf{6.197$\pm$0.033}} & 
    \reviewertwo{\textbf{0.080$\pm$0.001}} && 
    \reviewertwo{\textbf{3.373$\pm$0.028}} & 
    \reviewertwo{\textbf{6.791$\pm$0.051}} & 
    \reviewertwo{\textbf{0.089$\pm$0.001}}    \\\bottomrule\bottomrule
\end{tabular}
}\end{center}
 \vspace{-10pt}
\end{table*}
\subsubsection{Dynamic Node Classification}
\textbf{Dataset.}
There are few datasets for node classification in the dynamic setting. To the best of our knowledge, the Elliptic\footnote{\textcolor{blue}{https://www.kaggle.com/ellipticco/elliptic-data-set}} dataset is the largest one with $203,769$ nodes and $234,355$ edges derived from the bitcoin transaction records whose timestamp is actually discrete value from 0 to 49 which are evenly spaced in two weeks. We argue that the Reddit and Wikipedia used in \cite{xu2020inductive,jin2022neural} is more suitable for edge classification, because the label for a user (node) on Reddit indicates if the user is banned from posting under a subreddit (node), and likewise the label on Wikipedia is about whether the user (node) is banned from editing a Wikipedia page (node). In contrast, a node in Elliptic is a transaction and edges can be viewed as a flow of bitcoins between one transaction and the other. The label of each node indicates if the transaction is ``licit", ``illicit" or ``unknown".

\textbf{Protocol.}
The evaluation protocol follows~\cite{pareja2020evolvegcn}, where we use the information up to (discrete) time step $t$ and predict the label of a node which appears at the next (discrete) step $t\!+\!1$. Specifically, we split the data into training, validation and test set along the time dimension, where we use the first $38$ time steps to train the model, the next two time steps as the validation set and the rest used for testing. We note that the dark market shutdown occurred at step $43$, and thus we further group the test data in time steps $[40, 49]$ into subsets $[40,42]$, $[43,46]$ and $[47,49]$ to better study how the model behaves under external substantial changes.
	
We report the results in terms of Macro-F1, Micro-F1 and accuracy on the test set for the model that achieves the best on the validation data set. We note that nearly $80\%$ of nodes are labeled as ``unknown", such that a model which labels every node as unknown can achieve the accuracy up to $0.8$. Because illicit nodes are of more interest, we report the Micro-F1 for the minority (illicit) class. The results are the mean from five random trials.

\textbf{Baseline.}
We compare with state-of-the-art methods for dynamic node classification: (1) static graph approaches, i.e., APPNP~\cite{gasteiger2018combining}, GAT~\cite{velivckovic2017graph}, SAGE~\cite{hamilton2017inductive}, ARMA~\cite{bianchi2021graph}, GCN~\cite{kipf2016semi} and ChebyNet~\cite{defferrard2016convolutional}; 
(2) discrete-time approaches. \textbf{DySAT}\cite{sankar2020dysat} uses attention networks to model both temporal and structural information. Differently, \textbf{EvolveGCN}\cite{pareja2020evolvegcn} models the evolution of parameters in GCN over time. 
\reviewerone{\textbf{ROLAND}\cite{you2022roland} views node embeddings at different layers as hierarchical node states and recurrently update them over time. \textbf{DEFT}~\cite{bastos2023learnable} considers the history of evolving spectra in the form of learnable wavelets. \textbf{SpikeNET}~\cite{li2023scaling} models the dynamics of temporal graphs with spiking neural networks, a low-power alternative to RNNs;} 
(3) continuous-time approaches. \textbf{TGAT}~\cite{xu2020inductive} and \textbf{JODIE}~\cite{kumar2019predicting} adopt linear and sinusoidal functions of the elapsed time to encode temporal information, respectively. \textbf{TGN}~\cite{rossi2020temporal} combined the merits of temporal modules in JODIE, TGAT and DyREP to achieve the performance improvement. 
We omit the results of CTDNE~\cite{nguyen2018continuous} and DyREP~\cite{trivedi2019dyrep}, since it is known (see \cite{rossi2020temporal,jin2022neural}) that their accuracies are lower than those of TGN.
	

\textbf{Results.}
Table \ref{tab:nclassify} compares {\mname} with the baselines by Macro-F1, Micro-F1 and Accuracy on the Elliptic dataset. We see that ARMA performs the best among static graph models and further achieves comparable results to dynamic models: DySAT, JODIE and EvolveGCN. This is due to that ARMA provides a larger variety of frequency responses so that the network can better model sharp changes. Besides, TGN beats all other baselines, while {\mname} obtains better results than TGN with 4X speedup for training.
	
In particular, all models perform poorly after the time step $43$, when there was an external sudden change --- the shutdown of the dark market, which results in the universal performance degradation. 
\reviewertwo{{\mname} achieves $97\%$ and $88\%$ relative Micro-F1 gains over the best baseline during time step $[43,46]$ and $[47,49]$, respectively. This shows the merits of {\mname} in effective adaptation under sharp changes.}

\begin{figure*}[tb!]
\centerline{\includegraphics[width=.98\textwidth]{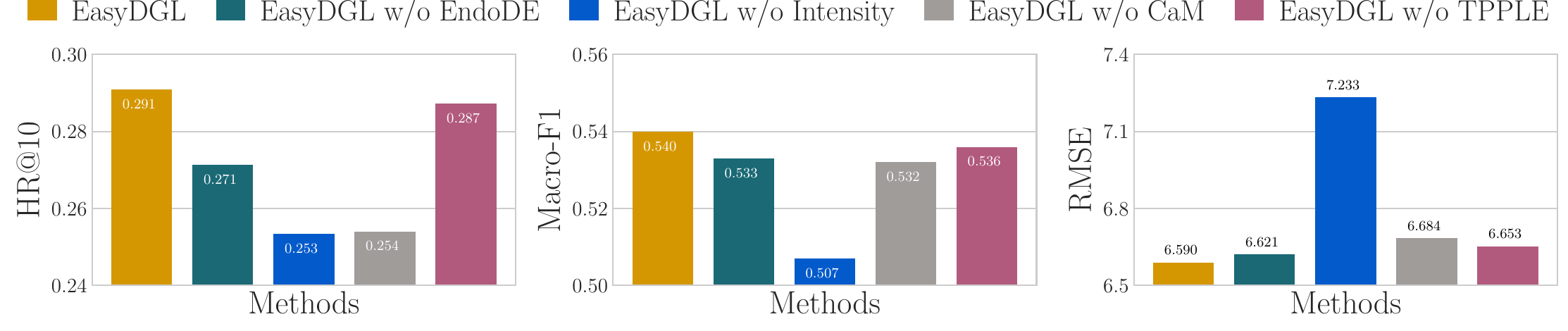}}
\caption{\label{fig:abl}Ablation study on Netflix for \textbf{dynamic link prediction (left)}, on Elliptic for \textbf{dynamic node classification (middle)} during step $[43,46]$, on META-LA for 30 minutes ahead \textbf{traffic forecasting (right)}. The higher HR@10 and Macro-F1 mean are better, while the lower RMSE is better.}
\end{figure*}
\begin{figure*}[tb!]
\centerline{\includegraphics[width=.98\textwidth]{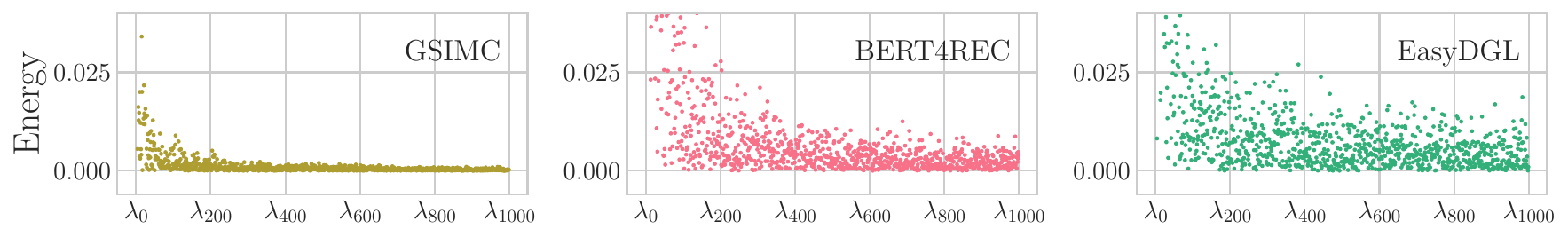}}
\caption{\label{fig:vf_analysis}Spectral interpreting of GSIMC~\cite{chen2023graph}, BERT4REC~\cite{sun2019bert4rec} and {\mname} for dynamic link prediction on Netflix. The energy of GSIMC is focused on low frequencies since high-frequency signals are penalized during minimization, whereas BERT4REC and {\mname} put more energies to high frequencies which helps increase the recall rates of low-degree item nodes. This explains the performance gap between static and dynamic models.}
\end{figure*}

\subsubsection{Node Traffic Forecasting}
\textbf{Dataset.}
We conduct experiments on real-world METR-LA dataset \cite{li2018diffusion} from the Los Angels Metropolitan Transportation Authority, which has averaged traffic speed measured by 207 sensors every $5$ minutes on the highways. It is shown in~\cite{yu2018spatio,li2018diffusion} that traffic flow forecasting on METR-LA ($207$ sensors, $34,272$ samples, Los Angels) is more challenging than that on PEMS-BAY ($325$ sensors, $52,116$ samples, Bay Area) dataset due to complicated traffic conditions.

\textbf{Protocol.}
We follow the protocol used in~\cite{li2018diffusion} to study the performance for traffic forecasting task, where we aggregate traffic speed readings at $5$-minute level and apply z-score normalization. We utilize first $70\%$ of data for model training, $20\%$ for validation while the remaining $10\%$ data for testing. Each of compared models is allowed to leverage the traffic conditions up to now for $15/30/45$-minutes ahead forecasting. We also adopt the mean absolute error (MAE), rooted mean squared error (RMSE) as well as mean absolute percentage error (MAPE) to evaluate the model accuracies. Detailed formulations of these metrics can be found in \cite{li2018diffusion}.

\textbf{Baseline.}
Existing traffic forecasting methods are mainly focused on discrete-time graphs, as traffic data is collected in a constant rate (e.g., every 5 minutes). These approaches maintain a sliding window of most recent 60-minutes data (i.e., $12$ snapshots) to make $15/30/45$-minutes ahead forecasting. \textbf{DCRNN}~\cite{li2018diffusion} leverages GCN to encode spatial information, while \textbf{GaAN}~\cite{zhang2018gaan} uses GAT to focus on most relevant neighbors in spatial domain. These methods exploit RNNs to capture the evolution of each node. To increase the scalability, \textbf{DySAT}~\cite{sankar2020dysat} utilize the attention to incorporate both structural and temporal dynamics while \textbf{STGCN}~\cite{yu2018spatio} resorts to simple CNNs. \textbf{DSTAGNN}~\cite{lan2022dstagnn} integrate attention network and gated CNN for accuracy improvement. 
\reviewerone{\textbf{STID}~\cite{shao2022spatial} identifies critical factors in spatial and temporal dimensions and designs a concise recursive model for improved efficiency. \textbf{Trafformer}~\cite{jin2023trafformer} learns the spatiotemporal affinity matrix at runtime and then performs regular attention mechanism to make predictions. \textbf{TrendGCN}~\cite{jiang2023enhancing} uses GAN structures to extend the flexibility of GCNs.}

Following \cite{li2018diffusion}, the embedding size is set to $32$ and the sliding window size for discrete-time approaches, such as DSTAGNN, DSTAGNN, STID and Trafformer, is $12$. Note that {\mname} and static graph models leverage only the most recent snapshot to make multi-step traffic forecasting, or namely the window size is one. Thus, the comparison is in favor of discrete-time approaches.

\textbf{Results.}
Table \ref{tab:traffic} presents the results on the METR-LA dataset for $15/30/45$ minutes ahead traffic forecasting. Dynamic approaches outperform static ones by a notable margin, and {\mname} achieves the best results among baselines. \reviewertwo{To be specific, {\mname} lowers down the RMSE results by $0.134$, $0.225$ and $0.251$ when Horizon=$3$, $6$, $9$ respectively. We note that the protocol favors discrete-time approaches: they accept $12$ graph snapshots as input, while {\mname} accesses to only $1$ snapshot. Also, {\mname} takes 1.1 hour training time, 6X faster than recently published DSTAGNN.}

\subsection{Ablation Study}
We evaluate the effectiveness of our proposed components: 
\textbf{{\mname} w/o EndoDE} drops endogenous dynamics encoding, making the intensity independent of localized changes.
\textbf{{\mname} w/o Intensity} switches off intensity-based modulation, thereby reducing our attention to regular attention.
\textbf{{\mname} w/o {\mssl}} disables our masking mechanism, while 
\textbf{{\mname} w/o {\mreg}} removes the TPP posterior maximization term by setting $\gamma=0$ in Eq.~\eqref{eqn:ctr_loss}. 
	
Fig.~\ref{fig:abl} reports HR@10 on Netflix for dynamic link prediction, RMSE on Elliptic for dynamic node classification, and Macro-F1 on META-LA for traffic forecasting. We omit the results in other metrics (NDCG, MAE MAPE) and on other data as their trends are similar. 
 {\mname} w/o EndoDE performs worse than {\mname} for all three tasks which suggests the necessity of modelling entangled spatiotemporal information; 
{\mname} w/o Intensity exhibits a significant drop in the performance, validating the efficacy of our TPP-based modulation design;
{\mname} w/o {\mssl} shows worse performance, which provides an evidence to the necessity of designing mask-based learning for dynamic graph models;
the performance of {\mname} w/o {\mreg} also degrades that suggests the effectiveness of our {\mreg} regularization.

\begin{table*}[tb!]
\caption{\label{tbl:perturb}Hit-Rate (HR) and NDCG (N) on the Netflix dataset for dynamic link prediction. \textbf{(top)} intra-perturbation is applied to remove the spectral contents of BERT4REC~\cite{sun2019bert4rec} with frequencies greater than $\lambda_{500}$. The performance degradation of the resulting \textit{perturb.} BERT4REC model suggests the necessity of high-frequency information; 
\textbf{(middle)} inter-perturbation is applied to replace the high-frequency ($>\lambda_{500}$) contents of  GSIMC~\cite{chen2023graph} by that of BERT4REC. One can see that the BERT4REC \textit{perturb.} GSIMC model performs comparable to BERT4REC. This implies that dynamic graph model BERT4REC outperforms static graph model SGMC due to the ability to learn high-frequency information; 
\textbf{(bottom)} inter-perturbation is applied to replace both low-frequency ($\le\lambda_{50}$) and high-frequency ($\lambda_{900}-\lambda_{1000}$) contents of TGAT with that of {\mname}. The performance improvement demonstrates the efficacy of the recommendation strategies and our superiority of modelling complex temporal dynamics over TGAT.}
\centering
\resizebox{.68\textwidth}{!}{
\begin{tabular}{c cccccc}\toprule\toprule
    \textbf{Netflix}
    & \textbf{HR@10} & \textbf{HR@50} & \textbf{HR@100}
    & \textbf{N@10} & \textbf{N@50} & \textbf{N@100} \\\midrule
    \textbf{BERT4REC} &
    0.25337 & 0.43027 & 0.52115 &
    0.16274 & 0.20166 & 0.21638\\
    \textbf{\textit{perturb.} BERT4REC} &
    0.14454 & 0.21172 & 0.24198 &
    0.09477 & 0.10994 & 0.11484 \\\midrule
	
    \textbf{GSIMC} &
    0.09725 & 0.22733 & 0.32225 & 
    0.05504 & 0.08181 & 0.09759 \\
    \textbf{BERT4REC \textit{perturb.} GSIMC} &
    0.25287 & 0.42514 & 0.50937 &
    0.16243 & 0.20033 & 0.21398 \\\midrule
	
    \textbf{TGAT} & 
    0.22755 & 0.39623 & 0.48232 &
    0.14541 & 0.18219 & 0.19611 \\

    \reviewertwo{\textbf{{\mframe} \textit{perturb.} TGAT}} & 
    \reviewertwo{0.25310} & \reviewertwo{0.43285} & \reviewertwo{0.51202} & 
    \reviewertwo{0.16003} & \reviewertwo{0.19914} & \reviewertwo{0.21197}
    \\\bottomrule\bottomrule
\end{tabular}}
\end{table*}
\begin{figure}[tb!]
    \centerline{\includegraphics[width=.50\textwidth]{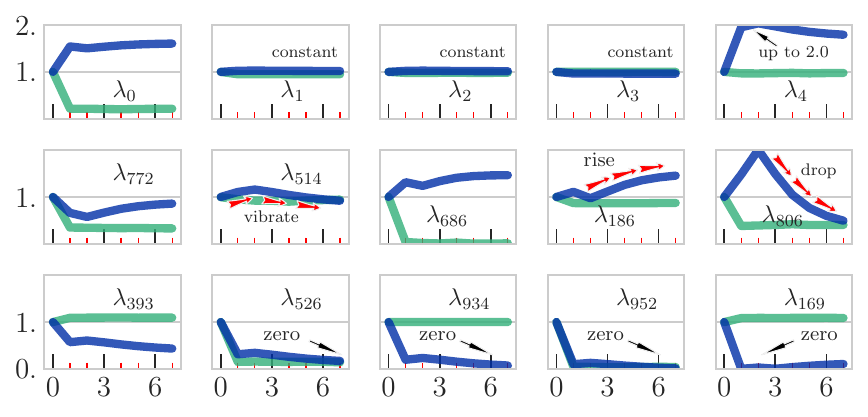}}
    \caption{\label{fig:vf_trends} The energy scale factor of 15 frequencies in 15 plots by spectral interpreting of TGAT\cite{xu2020inductive} (green) and {\mname} (blue) for the dynamic link prediction models trained on Netflix, which studies how TGAT and {\mname} behave in the graph Fourier domain over time (x-axis, in 7 days). It shows that {\mname} amplifies low-frequency signals (e.g., $\lambda_0,\lambda_{4}$) as time elapsed, signal at frequency $\lambda_{806}$ increases until the time interval reaches $2$ days long then drops dramatically, and high-frequency signals (e.g., $\lambda_{934},\lambda_{952}$) are attenuated as time goes on.}
    \vskip -0.1in
\end{figure}
\subsection{Qualitative Study}\label{sec:lp_qual}
We conduct experiments on Netflix for qualitative study.
	
\textbf{R1:} dynamic graph approach has a distinct advantage on the ability to manipulate high-frequency signals. In Fig.~\ref{fig:vf_trends}, we analyze static GSIMC model, dynamic BERT4REC and {\mname} models in the graph Fourier domain, where we apply $\ell_2$ normalization to each user's spectral signals in test set and visualize the averaged values. It shows that GSIMC only passes low frequencies (a.k.a., low-pass filtering), while dynamic models are able to learn high-frequency signals. 
	
In Table~\ref{tbl:perturb}, we apply the intra-perturbation to BERT4REC in order to remove spectral content at frequencies greater than $\lambda_{500}$. We see that \textit{perturb}. BERT4REC exhibits a significant drop in performance, which demonstrates the necessity of high-frequency information. Meanwhile, we use BERT4REC to perturb the predictions of GSIMC so as to measure the importance of high frequencies for GSIMC. The fact that BERT4REC \textit{perturb.} GSIMC performs much better than GSIMC and comparable to BERT4REC, indicating that BERT4REC beats GSIMC due to the ability of learning high-frequency information on dynamic graph.
	
\textbf{R2:} {\mname} tends to amplify low-frequency signal and attenuate high-frequency signal as time elapses. We retreat spectral signal for each user as a function of elapsed time (in days), calculate the average of signals over users then draw the evolution of averaged signals over time ($7$-days long) in Fig.~\ref{fig:vf_trends}. It shows for {\mname}, the energy in low frequencies increases in the dimension of elapsed time, while the closely related TGAT rejects low frequency $\lambda_0$ and remains $[\lambda_1,\lambda_{4}]$ unchanged. In the middle, we observe {\mname} is capable of modeling complex evolving dynamics for example, signal at $\lambda_{514}$ vibrates around initial value at the beginning; signal at $\lambda_{806}$ increases until time interval reaches $2$ days long, then drop dramatically. Fig.~\ref{fig:vf_trends} (bottom) also shows that {\mname} attenuates high frequencies $\lambda_{934},\lambda_{952}$. These results suggest that during a near future, {\mname} will recommend niche items to meet customer’s unique and transient needs in an aggressive manner. As time passes, {\mname} is inclined to take a conservative strategy which recommends the items with high click-through rate and high certainty.
	
\textbf{R3:} {\mname} benefits from the above recommendation strategies to achieve better prediction accuracy than TGAT. In Table~\ref{tbl:perturb}, we leverage {\mname} to perturb the predictions of TGAT at both low-frequency ($\le\lambda_{50}$) and high-frequency ($\lambda_{900}-\lambda_{1000}$) bands. We can see that the resulting {\mname} \textit{perturb.} TGAT model achieves the significant improvement, validating the effectiveness of the learned recommendation strategies and our superiority of modelling the complicated temporal dynamics over TGAT.

\section{Conclusion}\label{sec:conclude}
We have presented a pipeline for dynamic graph representation learning, where the TPP-modulated attention networks encode the continuous-time graph with its history, the principled family of loss functions accounts for both prediction accuracy and history fitting, and the quantitative analysis based on scalable spectral perturbations provides a global understanding of the learned models on dynamic graphs. Extensive empirical results have shown the universal applicability and superior performance of our pipeline, which is implemented by the DGL library. 
	

\section*{Acknowledgement}
This research work was partly supported by National Natural Science Foundation of China (62222607) and Shanghai Committee of Science and Technology (22511105100). The authors are thankful to Minjie Wang and David Wipf with AWS Shanghai AI Lab, for their valuable suggestions on the use of DGL toolkit for the implementation of our models.

\bibliographystyle{abbrv}
\bibliography{ref}

\begin{thebibliography}{10}

\bibitem{aalen2008survival}
O.~Aalen, O.~Borgan, and H.~Gjessing.
\newblock {\em Survival and event history analysis: a process point of view}.
\newblock Springer Science, 2008.

\bibitem{bastos2023learnable}
A.~Bastos, A.~Nadgeri, K.~Singh, T.~Suzumura, and M.~Singh.
\newblock Learnable spectral wavelets on dynamic graphs to capture global
  interactions.
\newblock In {\em AAAI}, volume~37, pages 6779--6787, 2023.

\bibitem{bennett2007netflix}
J.~Bennett, S.~Lanning, et~al.
\newblock The netflix prize.
\newblock In {\em KDD Cup}, volume 2007, page~35, 2007.

\bibitem{bianchi2021graph}
F.~M. Bianchi, D.~Grattarola, L.~Livi, and C.~Alippi.
\newblock Graph neural networks with convolutional arma filters.
\newblock {\em IEEE Transactions on Pattern Analysis and Machine Intelligence},
  44(7):3496--3507, 2022.

\bibitem{brody2021attentive}
S.~Brody, U.~Alon, and E.~Yahav.
\newblock How attentive are graph attention networks?
\newblock In {\em ICLR}, 2022.

\bibitem{chen2023graph}
C.~Chen, H.~Geng, Z.~Gang, Z.~Han, H.~Chai, X.~Yang, and J.~Yan.
\newblock Graph signal sampling for inductive one-bit matrix completion: a
  closed-form solution.
\newblock In {\em ICLR}, 2023.

\bibitem{chen2021learning}
C.~Chen, H.~Geng, N.~Yang, J.~Yan, D.~Xue, J.~Yu, and X.~Yang.
\newblock Learning self-modulating attention in continuous time space with
  applications to sequential recommendation.
\newblock In {\em ICML}, pages 1606--1616, 2021.

\bibitem{chen2021dyna}
C.~Chen, D.~Li, J.~Yan, and X.~Yang.
\newblock Modeling dynamic user preference via dictionary learning for
  sequential recommendation.
\newblock {\em IEEE Transactions on Knowledge and Data Engineering},
  34(11):5446--5458, 2022.

\bibitem{chien2022learning}
J.-T. Chien and Y.-H. Chen.
\newblock Learning continuous-time dynamics with attention.
\newblock {\em IEEE Transactions on Pattern Analysis and Machine Intelligence},
  2022.

\bibitem{cho2021learning}
J.~Cho, D.~Hyun, S.~Kang, and H.~Yu.
\newblock Learning heterogeneous temporal patterns of user preference for
  timely recommendation.
\newblock In {\em WWW}, pages 1274--1283, 2021.

\bibitem{daley2003introduction}
D.~J. Daley and D.~Vere-Jones.
\newblock {\em An introduction to the theory of point processes: volume I:
  elementary theory and methods}.
\newblock Springer, 2003.

\bibitem{dang2023uniform}
Y.~Dang, E.~Yang, G.~Guo, L.~Jiang, X.~Wang, X.~Xu, Q.~Sun, and H.~Liu.
\newblock Uniform sequence better: Time interval aware data augmentation for
  sequential recommendation.
\newblock In {\em AAAI}, volume~37, pages 4225--4232, 2023.

\bibitem{defferrard2016convolutional}
M.~Defferrard, X.~Bresson, and P.~Vandergheynst.
\newblock Convolutional neural networks on graphs with fast localized spectral
  filtering.
\newblock In {\em NIPS}, pages 3844--3852, 2016.

\bibitem{devlin2018bert}
J.~Devlin, M.-W. Chang, K.~Lee, and K.~Toutanova.
\newblock Bert: Pre-training of deep bidirectional transformers for language
  understanding.
\newblock {\em arXiv preprint arXiv:1810.04805}, 2018.

\bibitem{dhillon2007weighted}
I.~S. Dhillon, Y.~Guan, and B.~Kulis.
\newblock Weighted graph cuts without eigenvectors a multilevel approach.
\newblock {\em IEEE Transactions on Pattern Analysis and Machine Intelligence},
  29(11):1944--1957, 2007.

\bibitem{drineas2005nystrom}
P.~Drineas, M.~W. Mahoney, and N.~Cristianini.
\newblock On the nystr{\"o}m method for approximating a gram matrix for
  improved kernel-based learning.
\newblock {\em Journal of Machine Learning Research}, 6(12), 2005.

\bibitem{du2014analysis}
M.~C. Du~Plessis, G.~Niu, and M.~Sugiyama.
\newblock Analysis of learning from positive and unlabeled data.
\newblock In {\em NIPS)}, pages 703--711, 2014.

\bibitem{fan2021gcn}
Y.~Fan, Y.~Yao, and C.~Joe-Wong.
\newblock Gcn-se: Attention as explainability for node classification in
  dynamic graphs.
\newblock In {\em ICDM}, pages 1060--1065, 2021.

\bibitem{fan2021continuous}
Z.~Fan, Z.~Liu, J.~Zhang, Y.~Xiong, L.~Zheng, and P.~S. Yu.
\newblock Continuous-time sequential recommendation with temporal graph
  collaborative transformer.
\newblock In {\em CIKM}, pages 433--442, 2021.

\bibitem{fang2022bayesian}
S.~Fang, A.~Narayan, R.~Kirby, and S.~Zhe.
\newblock Bayesian continuous-time tucker decomposition.
\newblock In {\em ICML}, pages 6235--6245, 2022.

\bibitem{fowlkes2004spectral}
C.~Fowlkes, S.~Belongie, F.~Chung, and J.~Malik.
\newblock Spectral grouping using the nystrom method.
\newblock {\em IEEE Transactions on Pattern Analysis and Machine Intelligence},
  26(2):214--225, 2004.

\bibitem{garcia2018learning}
A.~Garc{\'\i}a-Dur{\'a}n, S.~Duman{\v{c}}i{\'c}, and M.~Niepert.
\newblock Learning sequence encoders for temporal knowledge graph completion.
\newblock In {\em EMNLP}, pages 4816--4821, 2018.

\bibitem{gasteiger2018combining}
J.~Gasteiger, A.~Bojchevski, and S.~Günnemann.
\newblock Predict then propagate: Graph neural networks meet personalized
  pagerank.
\newblock In {\em ICLR}, 2019.

\bibitem{goel2020diachronic}
R.~Goel, S.~M. Kazemi, M.~Brubaker, and P.~Poupart.
\newblock Diachronic embedding for temporal knowledge graph completion.
\newblock In {\em AAAI}, pages 3988--3995, 2020.

\bibitem{halko2011finding}
N.~Halko, P.-G. Martinsson, and J.~A. Tropp.
\newblock Finding structure with randomness: Probabilistic algorithms for
  constructing approximate matrix decompositions.
\newblock {\em SIAM review}, 53(2):217--288, 2011.

\bibitem{hamilton2017inductive}
W.~L. Hamilton, R.~Ying, and J.~Leskovec.
\newblock Inductive representation learning on large graphs.
\newblock In {\em NIPS}, pages 1025--1035, 2017.

\bibitem{he2022masked}
K.~He, X.~Chen, S.~Xie, Y.~Li, P.~Doll{\'a}r, and R.~Girshick.
\newblock Masked autoencoders are scalable vision learners.
\newblock In {\em CVPR}, pages 16000--16009, 2022.

\bibitem{hidasi2016session}
B.~Hidasi, A.~Karatzoglou, L.~Baltrunas, and D.~Tikk.
\newblock Session-based recommendations with recurrent neural networks.
\newblock In {\em ICLR}, 2016.

\bibitem{hou2022graphmae}
Z.~Hou, X.~Liu, Y.~Cen, Y.~Dong, H.~Yang, C.~Wang, and J.~Tang.
\newblock Graphmae: Self-supervised masked graph autoencoders.
\newblock In {\em KDD}, page 594–604. ACM, 2022.

\bibitem{huang2023position}
L.~Huang, Y.~Ma, Y.~Liu, B.~Danny~Du, S.~Wang, and D.~Li.
\newblock Position-enhanced and time-aware graph convolutional network for
  sequential recommendations.
\newblock {\em ACM Transactions on Information Systems}, 41(1):1--32, 2023.

\bibitem{jiang2023enhancing}
J.~Jiang, B.~Wu, L.~Chen, K.~Zhang, and S.~Kim.
\newblock Enhancing the robustness via adversarial learning and joint
  spatial-temporal embeddings in traffic forecasting.
\newblock In {\em CIKM}, pages 987--996, 2023.

\bibitem{jin2023trafformer}
D.~Jin, J.~Shi, R.~Wang, Y.~Li, Y.~Huang, and Y.-B. Yang.
\newblock Trafformer: unify time and space in traffic prediction.
\newblock In {\em AAAI}, volume~37, pages 8114--8122, 2023.

\bibitem{jin2022neural}
M.~Jin, Y.-F. Li, and S.~Pan.
\newblock Neural temporal walks: Motif-aware representation learning on
  continuous-time dynamic graphs.
\newblock In {\em NeurIPS}, 2022.

\bibitem{kang2018self}
W.-C. Kang and J.~McAuley.
\newblock Self-attentive sequential recommendation.
\newblock In {\em ICDM}, pages 197--206, 2018.

\bibitem{kazemi2020representation}
S.~M. Kazemi, R.~Goel, K.~Jain, I.~Kobyzev, A.~Sethi, P.~Forsyth, and
  P.~Poupart.
\newblock Representation learning for dynamic graphs: A survey.
\newblock {\em Journal of Machine Learning Research}, 21(70):1--73, 2020.

\bibitem{kingma2015adam}
D.~P. Kingma and J.~Ba.
\newblock Adam: A method for stochastic optimization.
\newblock In {\em ICLR}, 2015.

\bibitem{kipf2016semi}
T.~N. Kipf and M.~Welling.
\newblock Semi-supervised classification with graph convolutional networks.
\newblock In {\em ICLR}, 2016.

\bibitem{kumar2009sampling}
S.~Kumar, M.~Mohri, and A.~Talwalkar.
\newblock Sampling techniques for the nystrom method.
\newblock In {\em AISTATS}, pages 304--311, 2009.

\bibitem{kumar2012sampling}
S.~Kumar, M.~Mohri, and A.~Talwalkar.
\newblock Sampling methods for the nystr{\"o}m method.
\newblock {\em Journal of Machine Learning Research}, 13(1):981--1006, 2012.

\bibitem{kumar2019predicting}
S.~Kumar, X.~Zhang, and J.~Leskovec.
\newblock Predicting dynamic embedding trajectory in temporal interaction
  networks.
\newblock In {\em KDD}, 2019.

\bibitem{lan2022dstagnn}
S.~Lan, Y.~Ma, W.~Huang, W.~Wang, H.~Yang, and P.~Li.
\newblock Dstagnn: Dynamic spatial-temporal aware graph neural network for
  traffic flow forecasting.
\newblock In {\em ICML}, pages 11906--11917, 2022.

\bibitem{li2020time}
J.~Li, Y.~Wang, and J.~McAuley.
\newblock Time interval aware self-attention for sequential recommendation.
\newblock In {\em WSDM}, pages 322--330, 2020.

\bibitem{li2023scaling}
J.~Li, Z.~Yu, Z.~Zhu, L.~Chen, Q.~Yu, Z.~Zheng, S.~Tian, R.~Wu, and C.~Meng.
\newblock Scaling up dynamic graph representation learning via spiking neural
  networks.
\newblock In {\em AAAI}, pages 8588--8596, 2023.

\bibitem{li2010making}
M.~Li, J.~T.-Y. Kwok, and B.~L{\"u}.
\newblock Making large-scale nystr{\"o}m approximation possible.
\newblock In {\em ICML}, page 631, 2010.

\bibitem{li2020distance}
P.~Li, Y.~Wang, H.~Wang, and J.~Leskovec.
\newblock Distance encoding: Design provably more powerful neural networks for
  graph representation learning.
\newblock {\em NeurIPS}, 33:4465--4478, 2020.

\bibitem{li2018diffusion}
Y.~Li, R.~Yu, C.~Shahabi, and Y.~Liu.
\newblock Diffusion convolutional recurrent neural network: Data-driven traffic
  forecasting.
\newblock In {\em ICLR}, 2018.

\bibitem{liang2018variational}
D.~Liang, R.~G. Krishnan, M.~D. Hoffman, and T.~Jebara.
\newblock Variational autoencoders for collaborative filtering.
\newblock In {\em WWW}, pages 689--698, 2018.

\bibitem{lin2021generative}
W.~Lin, H.~Lan, and B.~Li.
\newblock Generative causal explanations for graph neural networks.
\newblock In {\em ICML}, pages 6666--6679, 2021.

\bibitem{liu2020kalman}
H.~Liu, J.~Lu, X.~Zhao, S.~Xu, H.~Peng, Y.~Liu, Z.~Zhang, J.~Li, J.~Jin,
  Y.~Bao, et~al.
\newblock Kalman filtering attention for user behavior modeling in ctr
  prediction.
\newblock In {\em NeurIPS}, 2020.

\bibitem{liu2022graph}
Y.~Liu, M.~Jin, S.~Pan, C.~Zhou, Y.~Zheng, F.~Xia, and P.~Yu.
\newblock Graph self-supervised learning: A survey.
\newblock {\em IEEE Transactions on Knowledge and Data Engineering}, 2022.

\bibitem{mao2018spatio}
X.~Mao, K.~Qiu, T.~Li, and Y.~Gu.
\newblock Spatio-temporal signal recovery based on low rank and differential
  smoothness.
\newblock {\em IEEE Transactions on Signal Processing}, 66(23):6281--6296,
  2018.

\bibitem{mcneil2021temporal}
M.~J. McNeil, L.~Zhang, and P.~Bogdanov.
\newblock Temporal graph signal decomposition.
\newblock In {\em KDD}, pages 1191--1201, 2021.

\bibitem{mei2017neural}
H.~Mei and J.~Eisner.
\newblock The neural hawkes process: a neurally self-modulating multivariate
  point process.
\newblock In {\em NIPS}, pages 6757--6767, 2017.

\bibitem{metropolis1949monte}
N.~Metropolis and S.~Ulam.
\newblock The monte carlo method.
\newblock {\em Journal of the American Statistical Association},
  44(247):335--341, 1949.

\bibitem{nguyen2018continuous}
G.~H. Nguyen, J.~B. Lee, R.~A. Rossi, N.~K. Ahmed, E.~Koh, and S.~Kim.
\newblock Continuous-time dynamic network embeddings.
\newblock In {\em WWW}, pages 969--976, 2018.

\bibitem{ortega2018graph}
A.~Ortega, P.~Frossard, J.~Kova{\v{c}}evi{\'c}, J.~M. Moura, and
  P.~Vandergheynst.
\newblock Graph signal processing: Overview, challenges, and applications.
\newblock {\em Proceedings of the IEEE}, 106(5):808--828, 2018.

\bibitem{pareja2020evolvegcn}
A.~Pareja, G.~Domeniconi, J.~Chen, T.~Ma, T.~Suzumura, H.~Kanezashi, T.~Kaler,
  T.~Schardl, and C.~Leiserson.
\newblock Evolvegcn: Evolving graph convolutional networks for dynamic graphs.
\newblock In {\em AAAI}, pages 5363--5370, 2020.

\bibitem{pope2019explainability}
P.~E. Pope, S.~Kolouri, M.~Rostami, C.~E. Martin, and H.~Hoffmann.
\newblock Explainability methods for graph convolutional neural networks.
\newblock In {\em CVPR}, pages 10772--10781, 2019.

\bibitem{qu2020continuous}
L.~Qu, H.~Zhu, Q.~Duan, and Y.~Shi.
\newblock Continuous-time link prediction via temporal dependent graph neural
  network.
\newblock In {\em WWW}, pages 3026--3032, 2020.

\bibitem{rossi2020temporal}
E.~Rossi, B.~Chamberlain, F.~Frasca, D.~Eynard, F.~Monti, and M.~Bronstein.
\newblock Temporal graph networks for deep learning on dynamic graphs.
\newblock {\em arXiv preprint arXiv:2006.10637}, 2020.

\bibitem{sankar2020dysat}
A.~Sankar, Y.~Wu, L.~Gou, W.~Zhang, and H.~Yang.
\newblock Dysat: Deep neural representation learning on dynamic graphs via
  self-attention networks.
\newblock In {\em WSDM}, pages 519--527, 2020.

\bibitem{shao2022spatial}
Z.~Shao, Z.~Zhang, F.~Wang, W.~Wei, and Y.~Xu.
\newblock Spatial-temporal identity: A simple yet effective baseline for
  multivariate time series forecasting.
\newblock In {\em CIKM}, pages 4454--4458, 2022.

\bibitem{steck2019markov}
H.~Steck.
\newblock Markov random fields for collaborative filtering.
\newblock In {\em NeurIPS}, pages 5474--5485, 2019.

\bibitem{stewart1990matrix}
G.~W. Stewart.
\newblock {\em Matrix perturbation theory}.
\newblock Boston: Academic Press, 1990.

\bibitem{stoer2013introduction}
J.~Stoer and R.~Bulirsch.
\newblock {\em Introduction to Numerical Analysis}, volume~12.
\newblock Springer Science \& Business Media, 2013.

\bibitem{sun2019bert4rec}
F.~Sun, J.~Liu, J.~Wu, C.~Pei, X.~Lin, W.~Ou, and P.~Jiang.
\newblock Bert4rec: Sequential recommendation with bidirectional encoder
  representations from transformer.
\newblock In {\em CIKM}, pages 1441--1450, 2019.

\bibitem{tailor2021degree}
S.~A. Tailor, J.~Fernandez-Marques, and N.~D. Lane.
\newblock Degree-quant: Quantization-aware training for graph neural networks.
\newblock In {\em ICLR}, 2021.

\bibitem{thakoor2022largescale}
S.~Thakoor, C.~Tallec, M.~G. Azar, M.~Azabou, E.~L. Dyer, R.~Munos,
  P.~Veli{\v{c}}kovi{\'c}, and M.~Valko.
\newblock Large-scale representation learning on graphs via bootstrapping.
\newblock In {\em ICLR}, 2022.

\bibitem{trivedi2019dyrep}
R.~Trivedi, M.~Farajtabar, P.~Biswal, and H.~Zha.
\newblock Dyrep: Learning representations over dynamic graphs.
\newblock In {\em ICLR}, 2019.

\bibitem{vaswani2017attention}
A.~Vaswani, N.~Shazeer, N.~Parmar, J.~Uszkoreit, L.~Jones, A.~N. Gomez,
  {\L}.~Kaiser, and I.~Polosukhin.
\newblock Attention is all you need.
\newblock In {\em NIPS}, pages 6000--6010, 2017.

\bibitem{velivckovic2017graph}
P.~Veli{\v{c}}kovi{\'c}, G.~Cucurull, A.~Casanova, A.~Romero, P.~Lio, and
  Y.~Bengio.
\newblock Graph attention networks.
\newblock In {\em ICLR}, 2017.

\bibitem{wu2021towards}
Q.~Wu, H.~Zhang, X.~Gao, J.~Yan, and H.~Zha.
\newblock Towards open-world recommendation: An inductive model-based
  collaborative filtering approach.
\newblock In {\em ICML}, pages 11329--11339, 2021.

\bibitem{xia2024neural}
J.~Xia, D.~Li, H.~Gu, T.~Lu, P.~Zhang, L.~Shang, and N.~Gu.
\newblock Neural kalman filtering for robust temporal recommendation.
\newblock In {\em WSDM}, pages 836--845, 2024.

\bibitem{xu2020inductive}
D.~Xu, C.~Ruan, E.~Korpeoglu, S.~Kumar, and K.~Achan.
\newblock Inductive representation learning on temporal graphs.
\newblock In {\em ICLR}, 2020.

\bibitem{you2022roland}
J.~You, T.~Du, and J.~Leskovec.
\newblock Roland: graph learning framework for dynamic graphs.
\newblock In {\em KDD}, pages 2358--2366, 2022.

\bibitem{you2020does}
Y.~You, T.~Chen, Z.~Wang, and Y.~Shen.
\newblock When does self-supervision help graph convolutional networks?
\newblock In {\em ICML}, pages 10871--10880, 2020.

\bibitem{yu2018spatio}
B.~Yu, H.~Yin, and Z.~Zhu.
\newblock Spatio-temporal graph convolutional networks: a deep learning
  framework for traffic forecasting.
\newblock In {\em IJCAI}, pages 3634--3640, 2018.

\bibitem{yuan2020future}
F.~Yuan, X.~He, H.~Jiang, G.~Guo, J.~Xiong, Z.~Xu, and Y.~Xiong.
\newblock Future data helps training: Modeling future contexts for
  session-based recommendation.
\newblock In {\em WWW}, pages 303--313, 2020.

\bibitem{yuan2020xgnn}
H.~Yuan, J.~Tang, X.~Hu, and S.~Ji.
\newblock Xgnn: Towards model-level explanations of graph neural networks.
\newblock In {\em KDD}, pages 430--438, 2020.

\bibitem{yuan2022explainability}
H.~Yuan, H.~Yu, S.~Gui, and S.~Ji.
\newblock Explainability in graph neural networks: A taxonomic survey.
\newblock {\em IEEE Transactions on Pattern Analysis and Machine Intelligence},
  2022.

\bibitem{yuan2021explainability}
H.~Yuan, H.~Yu, J.~Wang, K.~Li, and S.~Ji.
\newblock On explainability of graph neural networks via subgraph explorations.
\newblock In {\em ICML}, pages 12241--12252, 2021.

\bibitem{zhang2018gaan}
J.~Zhang, X.~Shi, J.~Xie, H.~Ma, I.~King, and D.~Y. Yeung.
\newblock Gaan: Gated attention networks for learning on large and
  spatiotemporal graphs.
\newblock In {\em UAI}, 2018.

\bibitem{zhou2018dynamic}
L.~Zhou, Y.~Yang, X.~Ren, F.~Wu, and Y.~Zhuang.
\newblock Dynamic network embedding by modeling triadic closure process.
\newblock In {\em AAAI}, volume~32, pages 571--578, 2018.

\bibitem{zhu2016scalable}
L.~Zhu, D.~Guo, J.~Yin, G.~Ver~Steeg, and A.~Galstyan.
\newblock Scalable temporal latent space inference for link prediction in
  dynamic social networks.
\newblock {\em IEEE Transactions on Knowledge and Data Engineering},
  28(10):2765--2777, 2016.

\bibitem{zuo2020transformer}
S.~Zuo, H.~Jiang, Z.~Li, T.~Zhao, and H.~Zha.
\newblock Transformer hawkes process.
\newblock In {\em ICML}, pages 11692--11702, 2020.

\bibitem{zuo2018embedding}
Y.~Zuo, G.~Liu, H.~Lin, J.~Guo, X.~Hu, and J.~Wu.
\newblock Embedding temporal network via neighborhood formation.
\newblock In {\em KDD}, pages 2857--2866, 2018.

\end{thebibliography}

\appendices

\section{Proof of Theorem 1}
We provide the proof details of our main result, presented in Theorem~\ref{th:Snorm}. The main procedure follows the seminar work\cite{halko2011finding} and we note that the challenge lies in Lemma.~\ref{lemma:proj_upper} which we believe is new and different from results in\cite{li2010making,kumar2012sampling}.

\subsection{Preliminaries and Existing Results}
For the sake of completeness, we present the definitions and existing results used in the proof.

\begin{definition}
(\textbf{Orthogonal Projector}, \cite{drineas2005nystrom}). 
A matrix $\mathbf{P}$ is called an orthogonal projector if $\mathbf{P}=\mathbf{P}^\top=\mathbf{P}^2$.
\end{definition}

\begin{lemma}\label{lemma:stewart}
As shown in \cite{stewart1990matrix}, the matrices $\mathbf{G}\in\mathbb{R}^{n\times{n}}$ and $\mathbf{F}\in\mathbb{R}^{n\times{n}}$ satisfy, 
\begin{eqnarray}
\max_{1\le{i}\le{n}} 
    \left| \sigma_i(\mathbf{G}) - \sigma_i(\mathbf{F}) \right|
\le& \parallel \mathbf{G} - \mathbf{F} \parallel_2 \\
\sum_{i=1}^n 
    \Big(
        \sigma_i(\mathbf{G}) - \sigma_i(\mathbf{F})
        \Big)^2
\le& \parallel \mathbf{G} - \mathbf{F} \parallel^2_F. 
\end{eqnarray}
\end{lemma}

\begin{theorem}\label{th:mli}
(\textbf{Proposition 1, \cite{halko2011finding}}).
Given a real $l$-by-$l$ matrix $\mathbf{A}$ with eigenvalues $\sigma_{1}\geq\ldots\geq\sigma_{l}$, choose a target rank $k$ and an oversampling parameter $p \geq 2,$ where $k+p\leq{l}$. Draw an $l$-by-$(k+p)$ standard Gaussian matrix $\Omega$, then construct the sample matrix $\mathbf{A}^q\mathbf{\Omega}$ where $q\ge{1}$. The orthonormal basis $\mathbf{Q}$ of matrix $\mathbf{A}^q\mathbf{\Omega}$ (i.e., $\mathbf{A}^q\mathbf{\Omega}=\mathbf{QQ}^\top\mathbf{A}^q\mathbf{\Omega}$) satisfies
\begin{flalign}
\mathbb{E}\parallel(\mathbf{I}-\mathbf{QQ}^\top)\mathbf{A}\parallel_{2} \leq \zeta^{1/q}\sigma_{k+1}(\mathbf{A}),
\end{flalign}where $\zeta=1+\sqrt{\frac{k}{p-1}} + \frac{e\sqrt{k+p}}{p}\sqrt{l-k}$.
\end{theorem}

\begin{theorem}\label{th:halko}
(\textbf{Theorem 10.5, \cite{halko2011finding}}).
Suppose that $\mathbf{A}$ is a real $l$-by-$l$ matrix with eigenvalues $\sigma_{1}\geq\ldots\geq\sigma_{l}$. Choose a target rank $k$ and an oversampling parameter $p \geq 2,$ where $k+p\leq{l}$. Draw an $l\times(k+p)$ standard Gaussian matrix $\Omega$, and construct the sample matrix $\mathbf{A\Omega}$. Then,
\begin{flalign}
\mathbb{E}\parallel(\mathbf{I}-\mathbf{QQ}^\top)\mathbf{A}\parallel_{F}
\leq\left(1+\frac{k}{p-1}\right)^{1/2}\left(\sum_{i>k} \sigma_{i}^{2}\right)^{1/2},
\end{flalign}where $\mathbf{Q}$ is the orthonormal basis of the range of matrix $\mathbf{A\Omega}$ such that $\mathbf{A\Omega}=\mathbf{QQ}^\top\mathbf{A\Omega}$.
\end{theorem}

\begin{theorem}\label{th:kumar}
(\textbf{Corollary 2, \cite{kumar2012sampling}}).
Suppose that $\mathbf{X}$ is a real $d$-by-$n$ matrix. Choose a set $\mathcal{S}$ of size $l$ at random without replacement from $\{1, 2, \dots, n\}$, and let $\mathbf{H}$ equals the columns of $\mathbf{X}$ corresponding to indices in $\mathcal{S}$. Let $\mathbf{HH}^\top$ be an approximation to $\mathbf{XX}^\top$, then
\begin{flalign}
\mathbb{E}
\parallel \mathbf{XX}^\top - \kappa\mathbf{HH}^\top \parallel_F
\leq \frac{n}{\sqrt{l}}  
        \max_i \parallel\mathbf{X}_{*,i}\parallel^2,
\end{flalign}where $\kappa = \frac{n}{l}$ is a non-zero scaling parameter, and $\parallel\mathbf{X}_{*,i}\!\parallel$ is the Euclidean norm of the $i^\mathrm{th}$ column of matrix $\mathbf{X}$.
\end{theorem}

\subsection{Proof of Theorem~\ref{th:Snorm}}
\begin{proof}
Since the graph Laplacian matrix $\mathbf{L}$ is a symmetric positive semidefinite matrix, we can write it as:
\begin{flalign}
    \mathbf{L}=\mathbf{X}^\top\mathbf{X}, \nonumber
\end{flalign}where $\mathbf{X}\in\mathbb{R}^{d\times{n}}$ and $d$ is the rank of matrix $\mathbf{L}$.

Let $\mathbf{S}=\{0, 1\}^{n\times{s}}$ be a column sampling matrix where $\mathbf{S}_{i,j}$ equals to $1$ if the $i^\mathrm{th}$ column of $\mathbf{L}$ is chosen in the $j^\mathrm{th}$ random trial and equals to $0$ otherwise. Then, $\mathbf{C}=\mathbf{X}^\top\mathbf{H}$ and $\mathbf{A}=\mathbf{H}^\top\mathbf{H}$ where $\mathbf{H}=\mathbf{XS}$.

We take $\mathbf{R}=\mathbf{H}\mathbf{A}^{-1/2}\mathbf{Q}$, then 
\begin{flalign}
\widetilde{\mathbf{L}}
&=\mathbf{CA}^{-1/2}\widetilde{\mathbf{V}}\widetilde{\mathbf{V}}^\top\mathbf{A}^{-1/2}\mathbf{X}^\top \nonumber\\
&= \mathbf{X}^\top\mathbf{H}\mathbf{A}^{-1/2}\mathbf{Q}{\mathbf{Q}}^\top\mathbf{A}^{-1/2}\mathbf{H}^\top\mathbf{X}\nonumber\\
&= \mathbf{X}^\top\mathbf{P_RX} = \mathbf{X}^\top\mathbf{U_RU_R}^\top\mathbf{X}, \nonumber
\end{flalign}where $\mathbf{P_R}\!=\!\mathbf{H}\mathbf{A}^{-1/2}\mathbf{Q}{\mathbf{Q}}^\top\mathbf{A}^{-1/2}\mathbf{H}^\top$ is an orthogonal projector and $\mathbf{U_R}$ is the orthonormal basis of matrix $\mathbf{R}$.

To bound the approximate error, we have
\begin{flalign}
&\parallel \mathbf{L} - \widetilde{\mathbf{L}} \parallel_2\nonumber\\
= &\parallel 
    \mathbf{X}^\top\mathbf{X} -
    \mathbf{X}^\top\mathbf{U_RU_R}^\top\mathbf{X} 
\parallel_2\nonumber\\
\stackrel{(a)}{=} &\parallel 
    \mathbf{X}^\top\mathbf{X} -
    (\mathbf{P_RX})^\top\mathbf{P_R}\mathbf{X} 
\parallel_2\nonumber\\
\stackrel{(b)}{=} &\parallel 
    \mathbf{X} -
    \mathbf{U_RU_R}^\top\mathbf{X} 
\parallel^2_2 \nonumber\\
=&\parallel 
    \mathbf{X} -
    \mathbf{X}\mathbf{U_RU_R}^\top
\parallel^2_2\nonumber\\
=& \parallel 
    \mathbf{X}\mathbf{X}^\top -
    \mathbf{X}\mathbf{U_RU_R}^\top\mathbf{X}^\top
\parallel_2,\nonumber
\end{flalign}where (a) holds due to the orthogonal project $\mathbf{P_R}$ satisfying $\mathbf{P_R}=\mathbf{P_R}^\top=\mathbf{P_R}^2$; 
(b) holds due to $\parallel\mathbf{AB}\parallel_2=\parallel\mathbf{BA}\parallel_2$ for any $\mathbf{A}\in\mathbb{R}^{m\times{n}}$ and $\mathbf{B}\in\mathbb{R}^{n\times{m}}$; 

Using Lemma~\ref{lemma:proj_upper} in Sec.~\ref{sec:lemma2},
\begin{flalign}
&\parallel 
    \mathbf{X}\mathbf{X}^\top -
    \mathbf{X}\mathbf{U_RU_R}^\top\mathbf{X}^\top
\parallel^2_2\nonumber\\
\le& \parallel 
    \mathbf{X}\mathbf{X}^\top -
    \kappa\mathbf{R}(\mathbf{HQ})^\top\mathbf{HQ}\mathbf{R}^\top
\parallel_2\nonumber\\
\le& \parallel
    \mathbf{X}\mathbf{X}^\top - \kappa\mathbf{H}^\top\mathbf{H}
\parallel_2 \nonumber\\
&\quad\quad\quad+ \kappa\parallel 
    \mathbf{H}\mathbf{H}^\top -
    \mathbf{R}(\mathbf{HQ})^\top\mathbf{HQ}\mathbf{R}^\top
\parallel_2, \label{eqn:spc_err}
\end{flalign}where the last step holds due to the triangle inequality.

Since $\mathbf{A}^{-1/2}\mathbf{H}^\top\mathbf{H}\mathbf{A}^{-1/2}=\mathbf{I}_s$, it gives
\begin{flalign}
&\parallel\mathbf{HH}^\top - 
    \mathbf{R}(\mathbf{HQ})^\top\mathbf{HQ}\mathbf{R}^\top
    \parallel_2 \nonumber\\
=& \parallel \mathbf{HH}^\top - 
    \mathbf{HA}^{-1/2}\mathbf{Q}(\mathbf{HQ})^\top\mathbf{HQQ}^\top\mathbf{A}^{-1/2}\mathbf{H}^\top
    \parallel_2 \nonumber\\
=& \parallel \mathbf{HA^{-1/2}}
        (\mathbf{A^{1/2}H}^\top - 
    \mathbf{QQ}^\top\mathbf{H}^\top\mathbf{H}\mathbf{A}^{-1/2}\mathbf{H}^\top)
    \parallel_2 \nonumber\\
=& \parallel \mathbf{H}^\top\mathbf{H} -
    \mathbf{QQ}^\top\mathbf{H}^\top\mathbf{H}
    \parallel_2 
= \parallel (\mathbf{I} - \mathbf{QQ}^\top)\mathbf{A} \parallel_2.\nonumber
\end{flalign}

Using Theorem~\ref{th:mli}, we can bound the expected error,
\begin{flalign}
&\kappa\mathbb{E} 
\parallel \mathbf{HH}^\top\!\! - 
    \mathbf{R}(\mathbf{HQ})^\top\mathbf{HQ}\mathbf{R}^\top\!\!
    \parallel_2 \nonumber\\
=& \kappa\mathbb{E}\parallel (\mathbf{I} - \mathbf{QQ}^\top)\mathbf{A} \parallel_2 \nonumber\\ 
\le& \zeta^{1/q}\sigma_{r+1}(\kappa\mathbf{A}) 
= \zeta^{1/q}\sigma_{r+1}(\kappa\mathbf{HH}^\top) \nonumber\\
\le& \zeta^{1/q}\sigma_{r+1}(\mathbf{XX}^\top) + \zeta^{1/q}\parallel
    \mathbf{XX}^\top - \kappa\mathbf{HH}^\top\!\!
\parallel_2, \label{enq:spc_A_err}
\end{flalign}where the last inequality holds because of $\sigma_{r+1}(\kappa\mathbf{HH}^\top) - \sigma_{r+1}(\mathbf{XX}^\top) \le \max_i | \sigma_{i}(\mathbf{XX}^\top) - \sigma_{i}(\kappa\mathbf{HH}^\top) |$ and Lemma~\ref{lemma:stewart}.

Combining Eq. (\ref{eqn:spc_err}) and (\ref{enq:spc_A_err}), we conclude our result
\begin{flalign}
&\mathbb{E} \parallel \mathbf{L} - \widetilde{\mathbf{L}} \parallel_2 \nonumber\\
&\le \zeta^{1/q}\sigma_{r+1}(\mathbf{XX}^\top) 
    + (1 + \zeta^{1/q})\parallel
    \mathbf{XX}^\top - \kappa\mathbf{HH}^\top\!\!
\parallel_2 \nonumber\\
&\le \zeta^{1/q}\parallel 
    \mathbf{L} - \mathbf{L}_r \parallel_2
+ (1 + \zeta^{1/q}) 
    \frac{n}{\sqrt{s}}\mathbf{L}^\ast_{i,i}, \nonumber
\end{flalign}where the last step is due to $\parallel\mathbf{XX}^\top - \kappa\mathbf{HH}^\top\!\!\parallel_2 \le \parallel\mathbf{XX}^\top - \kappa\mathbf{HH}^\top\!\!\parallel_F$ and Theorem~\ref{th:kumar}.
\end{proof}

\subsection{Proof of Lemma~\ref{lemma:proj_upper}}
\label{sec:lemma2}
\begin{lemma}\label{lemma:proj_upper}
Given $\mathbf{X}\in\mathbb{R}^{d\times{n}}$, let $\mathbf{U_R}$ be the orthonormal basis of the range of matrix $\mathbf{R}\in\mathbb{R}^{d\times{s}}$. Then for any $\mathbf{HQ}\in\mathbb{R}^{s\times{s}}$,
\begin{flalign}
\parallel \mathbf{XX}^\top\!\!-& \mathbf{XU_R}(\mathbf{XU_R})^\top\!\!\parallel_2 \nonumber\\
&\le \parallel 
    \mathbf{XX}^\top\!\!- \kappa\mathbf{R}(\mathbf{HQ})^\top\mathbf{HQ}\mathbf{R}^\top \!\!\parallel_2. \nonumber
\end{flalign}where $\kappa = \frac{n}{s}$ is a non-zero scaling parameter.
\end{lemma}
\begin{proof}
Let $\mathbf{P_R}=\mathbf{U_RU_R^\top}$. On using the property of orthogonal projector (i.e., $\mathbf{P_R}=\mathbf{P_R^\top}=\mathbf{P_R}^2$), we have
\begin{flalign}
&\parallel 
    \mathbf{XX}^\top\!\! -
    \mathbf{XU_R}(\mathbf{XU_R})^\top \!\! 
    \parallel_2 \nonumber\\
=& \parallel \mathbf{XX}^\top - \mathbf{XP_R}(\mathbf{XP_R})^\top \parallel_2 \label{eqn:x_norm}\\
=& \parallel \mathbf{X} - \mathbf{P_RX} \parallel^2_2  
= \max_{\parallel\mathbf{v}\parallel=1}
    \parallel
        \mathbf{v}^\top( \mathbf{X} - \mathbf{P_RX} )
    \parallel^2. \nonumber
\end{flalign}

We then decompose the vector $\mathbf{v}$ as $\mathbf{v}=\alpha\mathbf{y} + \beta\mathbf{z}$, where $\mathbf{y}\in\mathrm{ran}(\mathbf{R}),\mathbf{z}\in\mathrm{ran}^{\perp}(\mathbf{R})$ and $\alpha^2 + \beta^2 = 1$. It is clear to see that $\mathbf{y}^\top\mathbf{P_R}=\mathbf{y}^\top$, and $\mathbf{z}^\top\mathbf{P_R}=0$. Thereby,
\begin{flalign}
&\parallel \mathbf{X} 
- \mathbf{P_RX} \parallel_2 \nonumber\\
\le& \max_{\mathbf{y}\in\mathrm{ran}(\mathbf{R}),\parallel\mathbf{y}\parallel=1}
    \parallel
        \mathbf{y}^\top(\mathbf{X} - \mathbf{P_RX})
    \parallel \nonumber\\
&\quad\quad\quad + \max_{\mathbf{y}\in\mathrm{ran}^{\perp}(\mathbf{R}),\parallel\mathbf{z}\parallel=1}
    \parallel
        \mathbf{z}^\top(\mathbf{X} - \mathbf{P_RX}) 
    \parallel \nonumber\\
\le& \max_{\mathbf{z}\in\mathrm{ran}^{\perp}(\mathbf{R}),\parallel\mathbf{z}\parallel=1}
    \parallel
        \mathbf{z}^\top\mathbf{X} 
    \parallel. \label{eqn:x_reform}
\end{flalign}

For $\mathbf{z}\in\mathrm{ran}^{\perp}(\mathbf{R})$, $\mathbf{z}^\top\mathbf{R}(\mathbf{HQ})^\top\mathbf{HQ}\mathbf{R}^\top\mathbf{z}=0$. Then, 
\begin{flalign}
\parallel \mathbf{z}^\top\mathbf{X} \parallel^2
&= \mathbf{z}^\top\mathbf{XX}^\top\mathbf{z} \nonumber\\
&= \mathbf{z}^\top(\mathbf{XX}^\top 
    - \kappa\mathbf{R}(\mathbf{HQ})^\top\mathbf{HQ}\mathbf{R}^\top)\mathbf{z} \nonumber\\
&\le \max_{\parallel\mathbf{z}\parallel=1}
    \mathbf{z}^\top(\mathbf{XX}^\top 
    - \kappa\mathbf{R}(\mathbf{HQ})^\top\mathbf{HQ}\mathbf{R}^\top)\mathbf{z} \nonumber\\
&= \parallel 
        \mathbf{XX}^\top 
        - \kappa\mathbf{R}(\mathbf{HQ})^\top\mathbf{HQ}\mathbf{R}^\top\!\!
    \parallel_2 \label{eqn:x_upper}.
\end{flalign}Combining Eq. (\ref{eqn:x_norm}-\ref{eqn:x_upper}) concludes the lemma.
\end{proof}

\section{Implementation Details}\label{sec:impdetails}
In this section, we present the details of our implementation in order for reproducibility. All experiments are conducted on the machines with Xeon 3175X CPU, 128G memory and RTX8000 GPU with 48 GB memory. The configurations and packages are listed below:
\begin{itemize}
	\item Ubuntu 16.04
	\item CUDA 10.2
	\item Python 3.7 
	\item Tensorflow 1.15.3
	\item Pytorch 1.10
	\item DGL 0.8.2
	\item NumPy 1.19.0 with MKL Intel
\end{itemize}

\subsection{EasyDGL Architectures for Three Tasks on Graph}
\subsubsection{Dynamic Link Prediction:}
\begin{itemize}
    \item Use maximum sequence length to $30$ with the masked probability $0.2$.
	\item Two-layer Attention-Intensity-Attention with two heads.
	\item Use \textit{ReLU} as the activation.
	\item Use inner product between user embedding and item embedding as ranking score. 
\end{itemize}

\subsubsection{Dynamic Node Classification}
\begin{itemize}
    \item Randomly mask graph nodes with probability $0.2$.
	\item Two-layer \textit{GATConv} and one-layer Attention-Intensity-Attention block with two heads.
	\item Use \textit{ReLU} as the activation.
	\item Use one-layer \textit{Linear} for multi-class prediction. 
\end{itemize}

\subsubsection{Traffic Forecasting}
\begin{itemize}
    \item Randomly mask graph nodes with probability $0.2$.
	\item Two-layer \textit{SAGEConv} and one-layer Attention-Intensity-Attention with eight heads.
	\item Use \textit{ReLU} as the activation.
	\item Use one-layer \textit{Linear} for prediction. 
\end{itemize}

\subsection{Baseline Architectures for Dynamic Link Prediction}
\textbf{As mentioned, we follow IDCF\cite{wu2021towards} to build typical GNN architectures. Here we introduce the details for them.}

{\bf GAT.} We use the \textit{GATConv} layer available in DGL for implementation. The detailed architecture description is as below:
\begin{itemize}
	\item A sequence of one-layer  \textit{GATConv} with four heads. 
	\item Add self-loop and use batch normalization for graph convolution in each layer.
	\item Use \textit{tanh} as the activation.
	\item Use inner product between user embedding and item embedding as ranking score. 
\end{itemize}

{\bf GraphSAGE.} We use the \textit{SAGEConv} layer available in DGL for implementation. The detailed architecture description is as below:
\begin{itemize}
	\item A sequence of two-layer \textit{SAGEConv}.
	\item Add self-loop and use batch normalization for graph convolution in each layer.
	\item Use \textit{ReLU} as the activation.
	\item Use inner product between user embedding and item embedding as ranking score. 
\end{itemize}

{\bf GCN.} We use the \textit{SGConv} layer available in DGL for implementation. The detailed architecture description is as below:
\begin{itemize}
	\item One-layer \textit{SGConv} with two hops.
	\item Add self-loop and use batch normalization for graph convolution in each layer.
	\item Use \textit{ReLU} as the activation.
	\item Use inner product between user embedding and item embedding as ranking score. 
\end{itemize}

{\bf ChebyNet.} We use the \textit{ChebConv} layer available in DGL for implementation. The detailed architecture description is as below:
\begin{itemize}
	\item One-layer \textit{ChebConv} with two hops.
	\item Add self-loop and use batch normalization for graph convolution in each layer.
	\item Use \textit{ReLU} as the activation.
	\item Use inner product between user embedding and item embedding as ranking score. 
\end{itemize}

{\bf ARMA.} We use the \textit{ARMAConv} layer available in DGL for implementation. The detailed architecture description is as below:
\begin{itemize}
	\item One-layer \textit{ARMAConv} with two hops.
	\item Add self-loop and use batch normalization for graph convolution in each layer.
	\item Use \textit{tanh} as the activation.
	\item Use inner product between user embedding and item embedding as ranking score. 
\end{itemize}

{\bf We also summarize the implementation details of the compared sequential and temporal baselines as follows.}

{\bf GRU4REC.}\footnote{\textcolor{blue}{https://github.com/hidasib/GRU4Rec}} We use the software provided by the authors for experiments. The detailed architecture description is as below:
\begin{itemize}
	\item A sequence of two GRU cells.
	\item Use maximum sequence length to $30$.
	\item Use inner product between user embedding and item embedding as ranking score. 
\end{itemize}

{\bf SASREC.}\footnote{\textcolor{blue}{https://github.com/kang205/SASRec}} We use the software provided by the authors for experiments. The detailed architecture description is as below:
\begin{itemize}
	\item A sequence of two-block Transformer with four heads on Koubei, eight heads on Tmall and eight head on Netflix.
	\item Use maximum sequence length to $30$.
	\item Use inner product between user embedding and item embedding as ranking score. 
\end{itemize}

{\bf GREC.}\footnote{\textcolor{blue}{https://github.com/fajieyuan/WWW2020-grec}}
We use the software provided by the authors for experiments. The detailed architecture description is as below:
\begin{itemize}
	\item A sequence of six-layer dilated CNN with degree $1,2,2,4,4,8$.
	\item Use maximum sequence length to $30$ with the masked probability $0.2$.
	\item Use inner product between user embedding and item embedding as ranking score. 
\end{itemize}

{\bf S2PNM.}\footnote{\textcolor{blue}{https://github.com/cchao0116/S2PNM-TKDE2022}}
We use the software provided by the authors for experiments. The detailed architecture description is as below:
\begin{itemize}
	\item A sequence of one-block GRU-Transformer.
	\item Use maximum sequence length to $30$.
	\item Use inner product between user embedding and item embedding as ranking score. 
\end{itemize}

{\bf BERT4REC.}\footnote{\textcolor{blue}{https://github.com/FeiSun/BERT4Rec}}
We use the software provided by the authors for experiments. The detailed architecture description is as below:
\begin{itemize}
	\item A sequence of three-block Transformer with eight heads.
	\item Use maximum sequence length to $30$ with the masked probability $0.2$.
	\item Use inner product between user embedding and item embedding as ranking score. 
\end{itemize}

{\bf DyREP.}\footnote{\textcolor{blue}{https://github.com/uoguelph-mlrg/LDG}}
We use the software provided by the third party for experiments. The detailed architecture description is as below:
\begin{itemize}
	\item A sequence of one Attention-RNN Layer.
	\item Use maximum sequence length to $30$.
	\item Use linear layer of user embedding and item embedding with softplus activation as ranking score. 
\end{itemize}

{\bf TGAT.}\footnote{\textcolor{blue}{https://github.com/StatsDLMathsRecomSys/Inductive-representation-learning-on-temporal-graphs}}
We use the software provided by the authors for experiments. The detailed architecture description is as below:
\begin{itemize}
	\item A sequence of three-block Transformer with one head, time sinusoidal embeddings.
	\item Use maximum sequence length to $30$.
	\item Use inner product between user embedding and item embedding as ranking score. 
\end{itemize}

{\bf TiSASREC.}\footnote{\textcolor{blue}{https://github.com/JiachengLi1995/TiSASRec}}
We use the software provided by the authors for experiments. The detailed architecture description is as below:
\begin{itemize}
	\item A sequence of two-block Transformer with eight heads with time embedding.
	\item Use maximum sequence length to $30$.
	\item Use inner product between user embedding and item embedding as ranking score. 
\end{itemize}

{\bf TGREC.}\footnote{\textcolor{blue}{https://github.com/DyGRec/TGSRec}}
We use the software provided by the authors for experiments. The detailed architecture description is as below:
\begin{itemize}
	\item A sequence of three-block Transformer with time sinusoidal embeddings.
	\item Use maximum sequence length to $30$.
	\item Use inner product between user embedding and item embedding as ranking score. 
\end{itemize}

{\bf TimelyREC.}\footnote{\textcolor{blue}{https://github.com/Junsu-Cho/TimelyRec}}
We use the software provided by the authors for experiments. The detailed architecture description is as below:
\begin{itemize}
	\item A sequence of one-block Attention-Attention.
	\item Use maximum sequence length to $30$.
	\item Use inner product between user embedding and item embedding as ranking score. 
\end{itemize}

{\bf CTSMA.}\footnote{\textcolor{blue}{https://github.com/cchao0116/CTSMA-ICML21}}
We use the software provided by the authors for experiments. The detailed architecture description is as below:
\begin{itemize}
	\item A sequence of two-block Transformer with four heads.
	\item Use maximum sequence length to $30$.
	\item Use inner product between user embedding and item embedding as ranking score. 
\end{itemize}

\reviewerone{
{\bf TiCoSeREC.}\footnote{\textcolor{blue}{https://github.com/KingGugu/TiCoSeRec}}
We use the software provided by the authors for experiments. The detailed architecture description is as below:
\begin{itemize}
	\item A sequence of two-block Transformer with two heads.
	\item Use maximum sequence length to $30$.
	\item Use inner product between user embedding and item embedding as ranking score. 
\end{itemize}}

\reviewerone{
{\bf PTGCN.}\footnote{\textcolor{blue}{https://github.com/drhuangliwei/PTGCN}}
We use the software provided by the authors for experiments. The detailed architecture description is as below:
\begin{itemize}
	\item A sequence of two-block Transformer with four heads.
	\item Use maximum sequence length to $30$.
	\item Use inner product between user embedding and item embedding as ranking score. 
\end{itemize}}

\reviewerone{
{\bf NeuFilter.}\footnote{\textcolor{blue}{https://github.com/Yaveng/NeuFilter}}
We use the software provided by the authors for experiments. The detailed architecture description is as below:
\begin{itemize}
	\item Two GRU cells separately for item and user sequence.
	\item Use maximum sequence length to $30$.
	\item Concatenate user embeddings and item embeddings as the input of an multi-layer network which outputs ranking score. 
\end{itemize}}

\subsection{Baseline Architectures for Dynamic Node Classification}
The configurations for static graph models are identical to the dynamic link prediction task except the decoder module that is replaced by one-layer \textit{Linear}. 

\textbf{In the following, we present the architectures for new dynamic graph models.}

{\bf DySAT.}\footnote{\textcolor{blue}{https://github.com/aravindsankar28/DySAT}}
We use the software provided by the authors for experiments. The detailed architecture description is as below:
\begin{itemize}
    \item Use past five graph snapshots as input.
	\item A sequence of three-layer \textit{GATConv} and one-layer TemporalAttention with two heads.
	\item Use \textit{ReLU} as the activation.
	\item Use one-layer \textit{Linear} for multi-class prediction.
\end{itemize}

{\bf EvolveGCN.}\footnote{\textcolor{blue}{https://github.com/IBM/EvolveGCN}}
We use the software provided by the authors for experiments. The detailed architecture description is as below:
\begin{itemize}
    \item Use past five graph snapshots as input.
	\item One-layer gated GRUCell to update the parameters of two-layer \textit{GCNConv}.
	\item Use \textit{ReLU} as the activation.
	\item Use one-layer \textit{Linear} for multi-class prediction.
\end{itemize}

{\bf JODIE.}\footnote{\textcolor{blue}{https://github.com/claws-lab/jodie}}
We use the software provided by the authors for experiments. The detailed architecture description is as below:
\begin{itemize}
    \item Use past five graph snapshots as input.
	\item One-layer gated GRUCell to update the hidden states read out from two-layer TemporalTransformer.
	\item Use \textit{ReLU} as the activation.
	\item Use one-layer \textit{Linear} for multi-class prediction.
\end{itemize}

{\bf TGN.}\footnote{\textcolor{blue}{https://github.com/twitter-research/tgn}}
We use the software provided by the authors for experiments. The detailed architecture description is as below:
\begin{itemize}
    \item Use past five graph snapshots as input with sinusoidal time embedding. 
	\item Three-layer TemporalTransformer with five heads.
	\item One-layer time-decayed Recurrent unit to sequentially update node embeddings over time.
	\item Use \textit{ReLU} as the activation.
	\item Use one-layer \textit{Linear} for multi-class prediction.
\end{itemize}

\reviewerone{
{\bf RoLAND.}\footnote{\textcolor{blue}{https://github.com/snap-stanford/roland}}
We use the software provided by the authors for experiments. The detailed architecture description is as below:
\begin{itemize}
	\item Two-layer node embedding updation modules. In each layer, we first adopt the vanilla \textit{message-passing} to obtain preliminary node embeddings. Then, these preliminary node embeddings and the node embeddings output from the previous layer are both input into a one-layer \textit{GRU} to obtain the final node embedding for this layer.
	\item Use \textit{PReLU} as the activation.
	\item Use a \textit{MLP} implemented by two-layer \textit{Linear} for multi-class prediction.
\end{itemize}
}

\reviewerone{
{\bf DEFT.}\footnote{\textcolor{blue}{https://github.com/ansonb/DEFT}}
We use the software provided by the authors for experiments. The detailed architecture description is as below:
\begin{itemize}
    \item Use past five graph snapshots as input. 
	\item Two-layer \textit{ChebConv} with extra spectral filtering.
	\item Use \textit{RReLU} as the activation.
	\item Use a \textit{MLP} implemented by two-layer \textit{Linear} for multi-class prediction.
\end{itemize}
}

\reviewerone{
{\bf SpikeNet.}\footnote{\textcolor{blue}{https://github.com/EdisonLeeeee/SpikeNet}} 
We use the software provided by the authors for experiments. The detailed architecture description is as below:
\begin{itemize}
    \item Use all past graph snapshots as input.
    \item Wwo-layer \textit{GraphSage} with temporal neighborhood sampler on graphs.
    \item Use \textit{LIF} (Leaky integrate-and-ﬁre) as the activation.
    \item Use one-layer \textit{Linear} for multi-class prediction. 
\end{itemize}
}

\subsection{Baseline Architectures for Traffic Forecasting}
The configurations for static graph models are identical to the dynamic link prediction task except the decoder module that is replaced by one-layer \textit{Linear}.

\textbf{In the following, we present the architectures for new dynamic graph models.}

{\bf DCRNN.}
We use the software provided in DGL library. The detailed architecture description is as below:
\begin{itemize}
    \item Use past twelve graph snapshots as input.
    \item One-layer \textit{ChebConv} with two hops.
    \item One-layer \textit{GraphRNN} to sequential update the node embeddings over time.
    \item Use \textit{ReLU} as the activation.
    \item Use one-layer \textit{Linear} for prediction.
\end{itemize}

{\bf GaAN.}
We use the software provided in DGL library. The detailed architecture description is as below:
\begin{itemize}
    \item Use past twelve graph snapshots as input.
    \item One-layer gated graph attention that considers the edge weights.
    \item One-layer \textit{GraphRNN} to sequential update the node embeddings over time.
    \item Use \textit{ReLU} as the activation.
    \item Use one-layer \textit{Linear} for prediction.
\end{itemize}

{\bf STGCN.}
We use the software provided in DGL library. The detailed architecture description is as below:
\begin{itemize}
    \item Use past twelve graph snapshots as input.
	\item A sequence of TNTSTNTST, where T, N, S represents temporal convolutional neural network, \textit{LayerNorm} and spatial graph convolutional neural network, respectively.
	\item Use \textit{ReLU} as the activation.
	\item Use one-layer \textit{Linear} for prediction.
\end{itemize}

{\bf DSTAGNN.}\footnote{\textcolor{blue}{https://github.com/SYLan2019/DSTAGNN}}
We use the software provided by the authors for experiments. The detailed architecture description is as below:
\begin{itemize}
    \item Use past twelve graph snapshots as input.
	\item Four-layer DSTAGNN with four heads each of which uses \textit{ChebConv} with three hops.
	\item Use \textit{ReLU} as the activation.
	\item Use \textit{Conv2d} and \textit{Linear} for prediction.
\end{itemize}

\reviewerone{
{\bf Trafformer.}\footnote{\textcolor{blue}{https://github.com/jindi-tju/Trafformer}}
We use the software provided by the authors for experiments. The detailed architecture description is as below:
\begin{itemize}
    \item Use past twelve graph snapshots as input.
	\item Two-layer Transformer-style encoder with two heads.
        \item One-layer Transformer-style decoder with two heads.
	\item Use \textit{GeLU} as the activation.
	\item Use one-layer \textit{Linear} for prediction.
\end{itemize}
}

\reviewerone{
{\bf TrendGCN.}\footnote{\textcolor{blue}{https://github.com/juyongjiang/TrendGCN}}
We use the software provided by the authors for experiments. The detailed architecture description is as below:
\begin{itemize}
    \item Use past twelve graph snapshots as input.
        \item Two-layer \textit{Linear} discriminator with \textit{LeakyReLU} which focuses on the trend of individual time series.
        \item Two-layer \textit{Linear} discriminator with \textit{LeakyReLU} which emphasizes the correlation of multivariate time series.
        \item Two-layer \textit{ChebConv} and two-layer \textit{GRU}.
	\item Use one-layer \textit{Linear} for prediction.
 
\end{itemize}
}

\reviewerone{
{\bf STID}\footnote{\textcolor{blue}{https://github.com/zezhishao/STID/}}
We use the software provided by the authors for experiments. The detailed architecture description is as below:
\begin{itemize}
    \item Use past twelve graph snapshots as input.
        \item proposes aa framework with a embedding layer (by \textit{Conv2d}), multiple MLP layers, and a regression layer.
        \item Use \textit{ReLU} as the activation.
        \item Use \textit{Linear} for prediction.
\end{itemize}
}

\subsection{Choice of Hyper-parameters}
Regarding the choice of optimizer, we use Adam \cite{kingma2015adam} if not specified and the number of epochs is $200$. We search by grid the embedding size ranging in $\{64,128,\dots,512\}$, learning rate $\{1e\!-\!5,1e\!-\!4,\dots,1e\!-\!1\}$, dropout rate $\{0.1,0.2,\dots,0.7\}$, batch size $\{64,128,\dots,512\}$ and $\ell_2$ regularizer $\{1e\!-\!5,1e\!-\!4,\dots,1e\!-\!1\}$. We also study the influence of the neighborhood hops ranging from $1$ to $3$, the number of graph filtering blocks from $1$ up to $4$ and the number of heads in $\{1,2,\dots,8\}$.

We warn that we set the embedding size to $32$ for fair comparisons when evaluating the performance on the Ellicit and META-LA datasets. This is because TGN, DCRNN and DSTAGNN take days to complete the training if embedding size is greater than 64.

With regard to EasyDGL, we search by grid the masking rate in $\{10\%,20\%,\dots,50\%\}$ where, in majority of cases, $20\%$ produces the best results. We also search the best parameter for the TPPLE term in $\{1e\!-\!7,1e\!-\!6,\dots,1e\!-\!3\}$.

\section{Dataset Processing}
We present more dataset details in this section.

\subsection{Time Scaling}
We use one hour, one day and one week to scale the time data on the Netflix, Tmall and Koubei datasets, respectively. The choice of time unit is determined by the averaged time between two consecutive events for each user. We warn that again the time is discrete on the Elliptic data and the traffic speed readings on META-LA are record every five minutes. For both of these two datasets, we apply no modifications to the time data.

\subsection{Random Seed}
We use five random seeds to yield different data splits, i.e., $12345$, $54321$, $56789$, $98765$ and $7401$.

\subsection{Normalization on the META-LA data}
We calculate the mean and the standard deviation of the training readings on each road (node). When making predictions, we use these quantities to scale down the input readings and scale up the output readings. We found by experiments that this treatment can significantly reduce the RMSE and MAPE errors.

\end{document}